\documentclass{article}

\usepackage{amsmath}
\usepackage{amssymb}

\usepackage{microtype}
\usepackage{graphicx}
\usepackage{subfigure}
\usepackage{booktabs} 
\usepackage{multicol,multicap}

\usepackage{textcomp}
\usepackage{hyperref}

\usepackage{xcolor}
\definecolor{darkblue}{rgb}{0.0, 0.0, 0.5}
\definecolor{darkred}{rgb}{0.6, 0.0, 0.0}

\usepackage[accepted]{icml2024}

\usepackage{amsmath}
\usepackage{amssymb}
\usepackage{mathtools}
\usepackage{amsthm}
\usepackage{multirow}

\usepackage[capitalize,noabbrev]{cleveref}

\theoremstyle{plain}
\newtheorem{theorem}{Theorem}[section]

\newtheorem{lemma}[theorem]{Lemma}

\theoremstyle{definition}

\newtheorem{assumption}[theorem]{Assumption}
\theoremstyle{remark}
\newtheorem{remark}[theorem]{Remark}

\newlength\myindent
\usepackage[textsize=tiny]{todonotes}

\usepackage{comment}
\usepackage{bm}
\usepackage{url}            
\usepackage{booktabs}       
\usepackage{amsfonts}       
\usepackage{amsmath,amssymb}
\usepackage{empheq}
\usepackage{amsthm}
\usepackage{nicefrac}       
\usepackage{microtype}      
\usepackage{enumerate,natbib,mathrsfs}
\usepackage{tablefootnote}
\usepackage{wrapfig}
\usepackage{graphicx}
\usepackage{subfigure}
\usepackage{enumitem}
\usepackage{multirow}
\usepackage{extarrows}
\usepackage[OT1]{fontenc}

\usepackage{hyperref}

\hypersetup{
  pdffitwindow=true,
  pdfstartview={FitH},
  pdfnewwindow=true,
  colorlinks,
  linktocpage=true,
  linkcolor=red,
  urlcolor=red,
  citecolor=blue
}

\usepackage{hyperref}
\hypersetup{colorlinks=true,citecolor=blue,linkcolor=red}

\newcommand{\bbeta}{{\boldsymbol \beta}}
\newcommand{\bxi}{{\boldsymbol \xi}}
\newcommand{\bphi}{{\boldsymbol \phi}}

\newcommand{\by}{{\boldsymbol y}}

\newcommand{\bW}{{\boldsymbol W}}

\newcommand{\la}{\langle}

\newcommand{\dd}{\mathcal{\dagger}}

\newcommand{\ea}{\end{array}}
\newcommand{\ee}{\end{equation}}
\newcommand{\bea}{\begin{eqnarray}}
\newcommand{\eea}{\end{eqnarray}}
\newcommand{\beaa}{\begin{eqnarray*}}
\newcommand{\eeaa}{\end{eqnarray*}}

\def\Si{\Sigma}

\def\cE{{\cal E}}

\def\cL{{\cal L}}

\def\qq{\qquad}

\def\pa{\partial}
\def\cd{\cdot}

\def\by{{\bf y}}

\def\qed{ \hfill\qedsymbol}

\newcommand{\basa}{\begin{assumption}}
\newcommand{\easa}{\end{assumption}}

\newcommand{\bas}{\begin{assum}}
\newcommand{\eas}{\end{assum}}

\def\dd{\mathrm{d}}

\def\limP2{\,\mathop{\buildrel \Pi_2\over\longrightarrow\,}}

\def\pa{\partial}

 \def\cd{\cdot}

\def\diam{\hbox{\rm diam$\,$}}

\def\1{{\bf 1}}
\def\by{{\bf y}}

\def\:{\!:\!}

\renewenvironment{proof}[1][Proof]{\begin{trivlist}
\item[\hskip \labelsep {\bfseries #1}]}{\qed\end{trivlist}}

\def\cE{{\cal E}}

\def\cE{{\cal E}}

\def\cE{{\cal E}}

\newtheorem{assump}{Assumption}

\usepackage{algorithm}
\usepackage{algorithmic}

\def\cE{{\cal E}}

\def\cL{{\cal L}}

\def\cE{{\cal E}}

\newcommand{\ba}{\begin{array}}
\newcommand{\be}{\begin{equation}}

\newcommand{\ra}{\rangle}

\icmltitlerunning{Reflected Replica Exchange Stochastic Gradient Langevin Dynamics}

\hypersetup{
  pdffitwindow=true,
  pdfstartview={FitH},
  pdfnewwindow=true,
  colorlinks,
  linktocpage=true,
  linkcolor=darkblue,
  urlcolor=darkblue,
  citecolor=darkblue
}

\begin{document}

\twocolumn[
\icmltitle{Constrained Exploration via Reflected Replica Exchange\\ Stochastic Gradient Langevin Dynamics}



\icmlsetsymbol{equal}{*}

\begin{icmlauthorlist}
\icmlauthor{Haoyang Zheng}{purdue,equal}
\icmlauthor{Hengrong Du}{vandy,equal}
\icmlauthor{Qi Feng}{fsu,equal}
\icmlauthor{Wei Deng}{morgan}
\icmlauthor{Guang Lin}{purdue}
\end{icmlauthorlist}

\icmlaffiliation{purdue}{Purdue University, West Lafayette, IN}
\icmlaffiliation{morgan}{Machine Learning Research, Morgan Stanley, New York, NY}
\icmlaffiliation{vandy}{Vanderbilt University, Nashville, TN}
\icmlaffiliation{fsu}{Florida State University, Tallahassee, FL}

\icmlcorrespondingauthor{Wei Deng}{weideng056@gmail.com}
\icmlcorrespondingauthor{Guang Lin}{guanglin@purdue.edu}

\icmlkeywords{Machine Learning, ICML}
\vskip 0.3in
]



\printAffiliationsAndNotice{\icmlEqualContribution} 

\begin{abstract}
Replica exchange stochastic gradient Langevin dynamics (reSGLD) \cite{deng2020} is an effective sampler for non-convex learning in large-scale datasets. However, the simulation may encounter stagnation issues when the high-temperature chain delves too deeply into the distribution tails. To tackle this issue, we propose reflected reSGLD (r2SGLD): an algorithm tailored for constrained non-convex exploration by utilizing reflection steps within a bounded domain. Theoretically, we observe that reducing the diameter of the domain enhances mixing rates, exhibiting a \emph{quadratic} behavior. Empirically, we test its performance through extensive experiments, including identifying dynamical systems with physical constraints, simulations of constrained multi-modal distributions, and image classification tasks. The theoretical and empirical findings highlight the crucial role of constrained exploration in improving the simulation efficiency.
\end{abstract}

\section{Introduction}
Stochastic gradient Langevin dynamics (SGLD)~\cite{Welling11} is the go-to  Markov Chain Monte Carlo (MCMC) method in big data. The sampler smoothly transitions from stochastic optimization to sampling as the step size decreases and the injected noise enables exploration for posterior sampling. However, the simulation efficiency is severely affected by the pathological curvature of the energy landscape. To address these challenges, the  Quasi-Newton Langevin dynamics ~\cite{Ahn12, Simsekli2016, Li19} proposes to adapt to the varying curvature for more efficient exploration of the state space; Hamiltonian Monte Carlo (HMC) introduces auxiliary momentum variables and simulates the Hamiltonian dynamics to explore the periodic orbit~\cite{Neal2012,campbell2021gradient,zou2021convergence,wang2022accelerating}; Higher-order numerical techniques offer more efficient simulation by maintaining the stability using a larger stepsize~\cite{Chen15, Li19}.

Although the above samplers effectively address pathological curvatures, they are still susceptible to getting trapped in local regions during non-convex sampling. To circumvent this local trap phenomenon, importance samplers ~\cite{berg1992multicanonical, wang2001efficient, CSGLD, icsgld} inspired by statistical physics encourage the particles to explore the high energy tails of the distribution; simulated tempering~\cite{ST, Holden18} proposes to escape local optima by adjusting temperatures to explore various energy levels of the distribution, which offer viable solutions to these challenges. Building upon these, replica exchange Langevin dynamics (reLD)~\cite{PhysRevLett86, parallel_tempering05} and the big data extensions via reSGLD \cite{deng2020} utilize multiple diffusion processes at diverse temperatures and incorporate swapping during training, thereby enabling high-temperature processes to act as a bridge across various local modes. Its accelerated convergence has been both theoretically quantified \citep{Paul12, jingdong3} and empirically validated \cite{deng2020}.

\begin{wrapfigure}{r}{0.24\textwidth}
   \begin{center}
   \vskip -0.2in
     \includegraphics[width=0.24\textwidth]{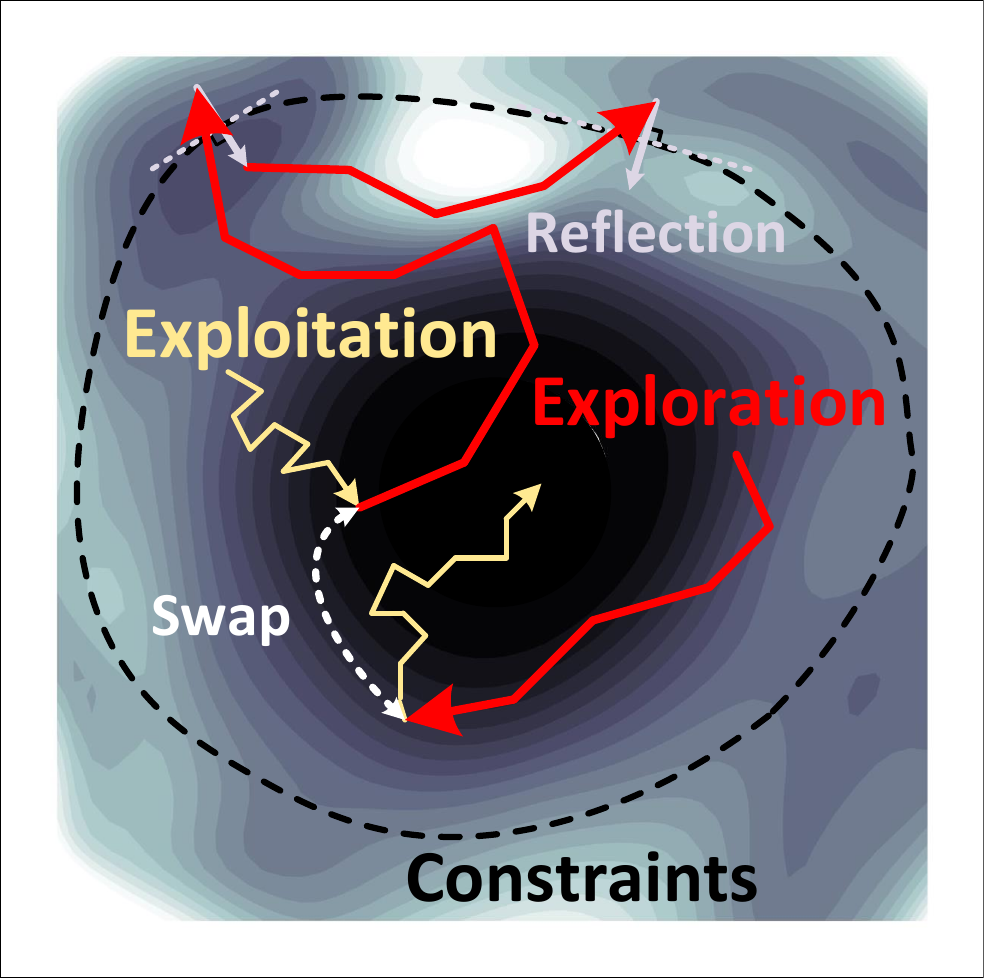}
   \end{center}
   \vskip -0.18in
   \caption{\footnotesize{Trajectory of r2SGLD.}}
   \label{path_demo}
   \vspace{-0.15 in}
\end{wrapfigure}


The na\"ive reSGLD algorithm, while effectively navigating non-convex landscapes, still faces over-exploration issues in high-temperature chains, which is often out of our interest in the optimization perspective. This is partly attributed to non-Arrhenius behavior \cite{zheng2007simulating, sindhikara2010exchange}, where increasing temperatures do not linearly translate to improved exploration efficiency. Specifically, over-exploration can result in either exploding or oscillating losses in deep learning training. This phenomenon can deteriorate the model's stability and optimization performance, and lead to poor predictions. Another interesting interpretation is from the perspective of Thompson sampling approximated by SGLD \cite{mazumdar2020approximate, zheng2024accelerating}, where excessive exploration of distribution tails often results in underestimating the optimal arm. This subsequently impedes achieving optimal regret bounds \cite{jin2021mots}.

To address this, constrained sampling techniques in MCMC play a pivotal role for different purposes in various forms, such as sampling on explicitly defined manifolds~\cite{lee2018convergence,wang2020fast}, implicitly defined manifolds~\cite{zappa2018monte,lelievre2019hybrid,kook2022sampling,zhang2022sampling,lelievre2023multiple}, and sampling with moment constraints~\cite{liu2021sampling}. 

\vspace{-0.03 in}
The proposed algorithm is further guided by constrained gradient Langevin dynamics with explicit form boundaries. A significant contribution to this field was first made by studying the convergence properties of Langevin Monte Carlo in bounded domains~\cite{sebastien_bubeck}, which uncovered a polynomial sample time for log-concave distributions. This research was further extended to non-convex settings later~\cite{Andrew_Lamperski_21_COLT}. Subsequently, Hamiltonian Monte Carlo techniques were employed to advance the field by exploring constrained sampling in the context of ill-conditioned and non-smooth distributions~\cite{kook2022sampling,noble2023unbiased}. Other notable works include proximal Langevin dynamics~\cite{brosse2017sampling} for log-concave distributions constrained to convex sets, reflected Schr\"odinger bridge for constrained generation~\cite{rSB}, and mirrored Langevin dynamics~\cite{hsieh2018mirrored, ahn2021efficient}, which focuses on convex settings. 

Based on the previous work, we propose an adapted sampler for constrained exploration---reflected replica exchange SGLD (r2SGLD)---to specifically address over-exploration in high-temperature chains and improve the mixing rates in challenging non-convex tasks. The r2SGLD algorithm employs parallel Langevin dynamics with swaps at varying temperatures, which achieves exploration through high-temperature chains and exploitation through low-temperature chains (Figure~\ref{path_demo}). Chains wandering outside the constrained domain are brought back by a reflection operation, which involves the construction of a tangent line at the nearest boundary for symmetrical reflection. By carefully constraining the exploration domain, we aim to harness the benefits of the high-temperature chain without succumbing to the detriments of over-exploration. A swapping mechanism between chains with different temperatures is subsequently used to balance between exploration and exploitation. We summarize the main contribution to this work as follows:

\vspace{-0.03 in}
(1) We contribute to the theoretical landscape by proving that r2SGLD outperforms the na\"ive reSGLD. \textcolor{black}{For the continuous-time analysis, we provide a quantitative mixing rate, marked by a refined decay rate of $\tilde O(1/\diam(\Omega)^2)$, where $\diam(\Omega)$ denotes the domain diameter.} 
\vspace{-0.055 in}

(2) This work introduces the novel use of constrained gradient Langevin dynamics in the identification of dynamical systems, which marks the first known application in this domain and broadens the methodological toolkit for dynamical system analysis. 
\vspace{-0.055 in}

(3) Extensive testing of r2SGLD against diverse baselines in both bounded 2D multi-modal distribution simulation and large-scale deep learning tasks further confirmed its superior efficacy. 

\begin{remark}
It should be noted that compared with the na\"ive reflected gradient Langevin dynamics, studying the quantitative value of the spectral gap in r2SGLD remains a fundamental challenge. For example, \citet{Bakry08} delves into a crude lower bound for the spectral gap concerning vanilla Langevin diffusion, where lower values imply faster convergence, particularly on sub-Gaussian distributions. Although these distributions are weaker than log-concavity, they still preclude the emergence of significant non-convexity.

The extension to Gaussian mixture distributions has posed a significant challenge. \citet{jingdong} investigates the considerable enhancement of the spectral gap with replica exchange Langevin diffusion compared to vanilla Langevin diffusion. We anticipate a similar acceleration would persist in bounded domains. However, given the methodological nature of our work, we recognize this topic and extensions to general non-convexity as beyond the current scope and defer it to future investigations. 
\end{remark}

\section{Methodology}\label{sec:method}

This section introduces the reflected reLD (r2LD), its infinitesimal generator, and the proposed r2SGLD algorithm. Our discussion centers on their contribution to efficient exploration in constrained non-convex sampling, which offers insights into its underlying mechanisms.


\subsection{Reflected Replica Exchange Langevin Diffusion}
The fundamental principle of the r2SGLD algorithm lies in the r2LD concept, which involves running multiple Langevin diffusions simultaneously, each at a different temperature. The system dynamics are described by the following stochastic differential equations:
\begin{equation}\label{eq:sde_2couple}
		\resizebox{0.91\hsize}{!}{$\begin{aligned}
    \dd \bbeta_t^{(1)} &= - \nabla U(\bbeta_t^{(1)}) \dd t+\sqrt{2\tau_1} \dd\bW_t^{(1)} + \nu(\bbeta_t^{(1)}) L^{(1)}(\dd t),\\
    \dd \bbeta_t^{(2)} &= - \nabla U(\bbeta_t^{(2)}) \dd t+\sqrt{2\tau_2} \dd\bW_t^{(2)} + \nu(\bbeta_t^{(2)}) L^{(2)}(\dd t),\\
		\end{aligned}$}
\end{equation}
where $\bbeta_t^{(1)}, \bbeta_t^{(2)}\in \mathbb R^d$ represent the parameter sets at respective temperatures $\tau_1$ and $\tau_2$. The term $\nabla U(\bbeta_t)$ denotes the gradient of the potential, and $\bW_t^{(1)}$ and $\bW_t^{(2)}$ are independent standard Wiener processes. The function $\nu(\bbeta_t)$ represents the inner unit normal vector at a point on the boundary $\partial\Omega$, while $L^{(1)}$ and $L^{(2)}$ denote independent local times with reference to $\partial \Omega$. 

Our focus is on restricting the process to a compact domain $\Omega\subset\mathbb{R}^d$ with a boundary $\partial\Omega$. We ensure reflected boundary conditions through the terms $\nu(\bbeta^{(1)}) L^{(1)}(\dd t)$ and $\nu(\bbeta^{(2)}) L^{(2)}(\dd t)$. We further demonstrate that the r2LD, as specified in \eqref{eq:sde_2couple}, converges to an invariant distribution $\pi(\cdot, \cdot)$:
\begin{equation}
\label{pt_density_main}
\begin{split}
  \dd\pi(x_1, x_2)=\frac{1}{Z} \underbrace{e^{-\frac{U(x_1)}{\tau_1}-\frac{U(x_2)}{\tau_2}}}_{p(x_1, x_2)}{\rm d}x_1{\rm d}x_2,
\end{split}
\end{equation}
where $Z$ is a normalizing constant: $$Z=\int_{\Omega\times\Omega}p(x_1, x_2)\dd x_1\dd x_2.$$ The density of this distribution, $p(\cdot, \cdot)$, applies only within the domain $\Omega\times\Omega$ and the density is zero outside this domain.

The swap between chains results in positional changes from $(\bbeta_t^{(1)},\bbeta_t^{(2)})$ to $(\bbeta_{t+dt}^{(2)},\bbeta_{t+dt}^{(1)})$, at a rate determined by $r (1\wedge S(\bbeta_t^{(1)},\bbeta_t^{(2)}))dt$, where $r\geq 0$ denotes the swapping intensity to regulate the frequency of swapping configurations between the two processes. The swap function $S(\cdot, \cdot)$ is given as follows:
\begin{equation}\label{S_exact}
\begin{split}
    S(\bbeta_t^{(1)}, \bbeta_t^{(2)}):=e^{ \left(\frac{1}{\tau_1}-\frac{1}{\tau_2}\right)\left(U(\bbeta_t^{(1)})-U(\bbeta_t^{(2)})\right)}.
\end{split}
\end{equation}
Similar to the Langevin diffusion in \citet{chen2018accelerating}, r2LD behaves as a reversible Markov jump process, which converges to the same invariant distribution (\ref{pt_density_main}).

\subsection{The Infinitesimal Generator}

The process $ \bbeta_t=(\bbeta_t^{(1)}, \bbeta_t^{(2)})$ follows a Markov diffusion process with infinitesimal generator $\cL$ to encompass both the r2LD \eqref{eq:sde_2couple} and swapping dynamics \eqref{S_exact}. For $(x_1, x_2)$ in the interior of $\Omega\times\Omega$, the generator $\cL$ applies to a smooth function $f$ and is structured into distinct components: $\cL^{(1)} f$ and $\cL^{(2)} f$ for the dynamics of $x_1$ and $x_2$, and $\cL^{(s)} f$ for swap dynamics:
\begin{equation}
    \begin{aligned}
    \cL f=&\underbrace{-\la\nabla_{x_1}f(x_1,x_2),\nabla U(x_1)\ra+\tau_1\Delta_{x_1}f(x_1,x_2)}_{\cL^{(1)} f}\\
&\underbrace{-\la \nabla_{x_2}f(x_1,x_2),\nabla U(x_2)\ra+\tau_2\Delta_{x_2}f(x_1,x_2)}_{\cL^{(2)} f}\\
&\underbrace{+rS(x_1, x_2)\cdot (f(x_2, x_1)-f(x_1, x_2))}_{\cL^{(s)} f}, 
\end{aligned}\label{eqn:generator}
\end{equation}
where $\nabla_{x_i}, \Delta_{x_i}$ denote the gradient and Laplacian operator with respect to $x_i$, $i=1, 2$,
and the entire formulation is subject to the Neumann boundary conditions 
\cite{Menaldi1985,Kang2014}, also \cite{wang2014analysis2} (Theorem 3.1.3):
\begin{equation}
  \left\{
  \begin{array}{ll}
    \nabla_{x_1}f(x_1, x_2)\cdot \nu(x_1) =0 & (x_1, x_2)\in \partial\Omega\times \Omega, \\
    \nabla_{x_2}f(x_1, x_2)\cdot \nu(x_2)=0 & (x_1, x_2)\in \Omega\times \partial \Omega.
  \end{array}
  \right.
  \label{eqn:bdy}
\end{equation}

For the invariant measure $\pi$ that is associated with a smooth potential function $U(\cdot)$ as in \eqref{pt_density_main}, the corresponding Dirichlet form that appears in the Poincare inequality \eqref{eqn:PoineIneq} and Logarithmic Sobolev (Log-Sobolev) inequality \eqref{eqn:LogSob} is given as follows: 
\bea\label{dirichlet form}
\cE(f)=\int \Big(\tau_1\|\nabla_{x_1}f\|^2+\tau_2\|\nabla_{x_2}f\|^2 \Big)d\pi(x_1,x_2).
\eea

Furthermore, as established in Theorem 3.3~\citep{chen2018accelerating}, the Dirichlet form $\mathcal{E}_S$ linked to the infinitesimal generator $\cL^{(s)}$ includes an additional term representing acceleration:
\begin{equation*}
  \begin{array}{ll}
        &\cE_{S}(f)=\cE(f)+\\
        &\underbrace{\frac{r}{2}\int S(x_1,x_2)\cd (f(x_2,x_1)-f(x_1,x_2))^2d\pi(x_1,x_2)}_{\text{acceleration term}},
  \end{array}
\end{equation*}
which induces positive acceleration under mild conditions, and it is vital for rapid convergence in the 2-Wasserstein ($\mathcal{W}_2$) distance. Notably, the acceleration's effectiveness depends on the swap function \eqref{S_exact}. 

\subsection{The Proposed Algorithm}
\label{tilde S}
r2SGLD is the practical interpretation of the continuous process described by the infinitesimal generator in \eqref{eq:sde_2couple}. Following Algorithm~\ref{alg}, it starts with updating the parameters $\widetilde \bbeta^{(1)}$ and $\widetilde \bbeta^{(2)}$ for the diffusion process using stochastic gradients and noises, followed by reflection operations $\mathcal{R}(\cdot)$ to reflect out-of-boundary samples back to the bounded domain $\Omega$. The gradient,
\begin{equation*}
    \nabla {\widetilde U}(\widetilde \bbeta^{(i)}_k) = \frac{N}{\left|\mathbb B_k\right|}\sum_{\mathcal D_j\in \mathbb B_k}\nabla U(\mathcal D_j|\widetilde \bbeta^{(i)}_k),\ \ \ \ i=1,2,
\end{equation*}
is estimated by a mini-batch data $\mathbb B_k$, and $\{\mathcal D_j\}_{j=1}^N$ is the dataset. Next, the algorithm determines whether to swap the chains by comparing a uniformly generated random number $u$ against the corrected swapping intensity $\widetilde S$:
 \begin{equation}\label{eq:swap_correct}
		\resizebox{0.91\hsize}{!}{$\begin{aligned}
		      \widetilde S(\widetilde \bbeta_{k}^{(1)}, \widetilde \bbeta_{k}^{(2)}) = e^{ \left(\frac{1}{\tau_1}-\frac{1}{\tau_2}\right)\left( \widetilde U(\widetilde \bbeta_{k}^{(1)})- \widetilde U(\widetilde \bbeta_{k}^{(2)})-\left(\frac{1}{\tau_1}-\frac{1}{\tau_2}\right)\frac{\widetilde \sigma^2}{\mathcal C} \right)},
		\end{aligned}$}
\end{equation}
where $\widetilde \sigma^2$ approximates the variance of $\widetilde U(\widetilde \bbeta_{k+1}^{(1)})- \widetilde U(\widetilde \bbeta_{k+1}^{(2)})$ and $\mathcal C$ acts as an adjustment to balance acceleration and bias. The algorithm proceeds iteratively until a specified number of iterations is reached. The parameter sets $\{\widetilde \bbeta^{(1)}_{k}\}_{k=1}^{K+1}$ are produced as outputs for analytical purposes.
\begin{algorithm}[tb]
   \caption{The r2SGLD Algorithm. 
   }
   \label{alg}
{\textbf{Input} Initial parameters $\widetilde \bbeta^{(1)}_{1}$, $\widetilde \bbeta^{(2)}_{1}$.}\\
{\textbf{Input} Number of iterations $K$.}\\
{\textbf{Input} Temperatures $\tau_1$, $\tau_2$.}\\
{\textbf{Input} Correction factor $\mathcal C$.}\\
{\textbf{Input} Learn rate $\eta$.}

\begin{algorithmic}
\FOR{$k=1,2,\cdots, K$}
   \STATE{\textbf{Sampling Step}}
   \vspace{-0.1 in}
   \begin{equation*}
       \begin{split}
           \widetilde \bbeta^{(1)}_{k+1} &= \mathcal{R}\bigg(\widetilde \bbeta^{(1)}_{k}- \eta \nabla {\color{black}\widetilde U}(\widetilde \bbeta^{(1)}_k)+\sqrt{2\eta\tau_1} \bxi_k^{(1)}\bigg),\\
    \widetilde \bbeta^{(2)}_{k+1} &= \mathcal{R}\bigg(\widetilde \bbeta^{(2)}_{k} - \eta\nabla {\color{black}\widetilde U}(\widetilde \bbeta^{(2)}_k)+\sqrt{2\eta\tau_2} \bxi_k^{(2)}\bigg).\\
       \end{split}
   \end{equation*}
   \STATE{\textbf{Swapping Step}}
   \STATE{\ \ \ \ \ \ \ \ Generate a uniform random number $u\in [0,1]$.}
   \STATE{\ \ \ \ \ \ \ \ Compute $\widetilde S$ follows \eqref{eq:swap_correct}.}
   \STATE{ \ \ \ \ \ \ \ \   \textbf{if} {$u<{\widetilde S}$} \textbf{then}}
   \STATE{ \ \ \ \  \ \ \ \ \ \ \ \  \ \ \ \   Swap $\widetilde \bbeta_{k+1}^{(1)}$ and $\widetilde \bbeta_{k+1}^{(2)}$.}
   \STATE{ \ \ \ \  \ \ \ \  \textbf{end if} }
\ENDFOR
    \vskip -1 in
\end{algorithmic}
{\textbf{Output} Parameters $\{\widetilde \bbeta^{(1)}_{k}\}_{k=1}^{K+1}$.}
\end{algorithm}

\section{Theoretical Analysis}
In this section, we outline the convergence analysis of continuous-time r2LD under the $\chi^2$-divergence and $\mathcal{W}_2$-distance, which extends the previous analysis \cite{chen2018accelerating,deng2020} within a constrained domain. We further highlight the acceleration benefits achieved through parameter swapping and discuss how the convergence rate is influenced by the diameter of $\Omega$. Lastly, our analysis includes an evaluation of discretization error in the 1-Wasserstein ($\mathcal{W}_1$) distance. Throughout the analysis, we adopt the following assumptions. 
\begin{assump}\label{assu:domainLip}
    $\Omega$ is a compact domain with a boundary $\partial\Omega$  whose \emph{second fundamental form} is bounded below by some constant $\kappa\le 0$ . 
\end{assump} 

\begin{remark}
For any two tangent vectors $v_1, v_2$ at $x\in \partial \Omega$, the \emph{second fundamental form} of $\partial \Omega$ is defined by $\mathbb{I}(v_1, v_2)(x)=-\langle \nabla_{v_1}\nu(x), v_2\rangle$, where $\nu(x)$ is the normal vector field, $\nabla_{v_1}\nu$ denotes the directional derivative of $\nu(x)$ along $v_1$. Note that if $\mathbb{I}(v_1, v_2)(x)\ge \kappa$ for some $\kappa\le 0$, the second fundamental form of the boundary $\partial \Omega$ is bounded below by $\kappa$. It is also worth mentioning that if $\kappa=0$, $\Omega$ is convex.
\end{remark}

\begin{assump}\label{assu:ULip} The function $U\in C^{2}(\Omega)$. Since $\Omega$ is compact, there exists an $L>0$ such that for all $x, y\in\Omega$, 
\begin{equation}
    \|\nabla U(x)-\nabla U(y)\|\le L \|x-y\|. 
\end{equation}
\end{assump}
Unless specified otherwise, the theoretical results in this work are based on Assumptions \ref{assu:domainLip} and \ref{assu:ULip}. 

\subsection{Convergence in $\chi^2$-Divergence}
Our analysis begins by identifying the invariant distribution of r2LD, which delves into the basic properties of the generator $\cL$ subject to the boundary condition \eqref{eqn:bdy}:
\begin{lemma}\label{lemma:reversibility}
    $\{\bbeta_t\}_{t\ge 0}$ is reversible and its invariant distribution $\pi$ is given by \eqref{pt_density_main}.
\end{lemma}
We define the $\chi^2$-divergence as
\begin{equation}
    \chi^2(\mu_t\|\pi)=\int_{\Omega\times\Omega}\left(\frac{\dd\mu_t}{\dd \pi}-1\right)^2 \dd\pi,
\end{equation}
where $\dd\mu_t/\dd\pi$ is the Radon--Nikodym derivative between $\mu_t$ and $\pi$. The convergence rate of $\mu_t$ in $\chi^2$-divergence is related to the Poincar\'{e} inequality:
\begin{lemma}\label{lemma:PoincareIne}
    For any probability measure $\mu\ll\pi$ where the Radon--Nikodym derivative $\dd\mu/\dd \pi$ satisfies \eqref{eqn:bdy}, the following inequality holds:
    \begin{equation}
        \chi^2(\mu\|\pi)\le C_{\rm P}\mathcal{E}\left(\frac{\dd \mu}{\dd\pi}\right),\label{eqn:PoineIneq}
    \end{equation}
    where the best constant $C_{\rm P}$ is called Poincar\'{e} constant. 
\end{lemma}
Here is our primary convergence result:
\begin{theorem} Given any initial measure $\mu_0$ for which $\dd \mu_0/\dd \pi$ satisfies \eqref{eqn:bdy}, the $\chi^2$-divergence to the invariant distribution $\pi$ decays exponentially according to:
    \begin{equation}   \chi^2(\mu_t\|\pi)\le \chi^2(\mu_0\|\pi)\exp\left(-2t(1+\eta_S)C_{\rm P}^{-1}\right),   \label{eqn:chiconvergencerate}
    \end{equation}
        where $\displaystyle\eta_S:=\inf_{t>0}\frac{\mathcal{E}_S\left( \frac{\dd \mu_t}{\dd \pi} \right)}{\mathcal{E}\left( \frac{\dd \mu_t}{\dd \pi} \right)}-1$ is the \emph{acceleration effect} in $\chi^2$-divergence.\label{thm:chiconvergence}
\end{theorem}
In addition to the observed acceleration effect, our findings highlight the significance of the Poincar\'{e} constant $C_{\rm P}$ in convergence, as shown in \eqref{eqn:chiconvergencerate}, where its lower bound directly influences the convergence rate.
In the literature, $C_{\rm P}$ is often referred to the \emph{first Neumann eigenvalue} of $\cL$ or \emph{spectral gap}, denoted by $\lambda_1(\cL)$.  For a non-convex domain that satisfies Assumption~\ref{assu:domainLip}, it has been shown in \citet{wang2014analysis2} (Corollary 3.5.2) that 
\begin{lemma}\label{lemma:cPestimate} The Poincar\'{e} constant $C_{\rm P}$, defined in Lemma \ref{lemma:PoincareIne}, satisfies the inequality:
    \begin{equation}
        C_{\rm P}=\lambda_1(\cL)\ge C_1\left(\frac{\pi^2}{\diam(\Omega)^2}+C_2\right),\label{eqn:Cpbound}
    \end{equation}
    where $C_1$ and $C_2$ are constants depending on $U$ and $\kappa$, in particular, when $\Omega$ is convex, we have $C_1=1$ and $C_2=0$, i.e.,
    \begin{equation*}
        \lambda_1(\cL)\ge \frac{\pi^2}{\diam(\Omega)^2}.
    \end{equation*}
\end{lemma}
\begin{remark}\label{rmk:nonconvex}
For a non-convex domain $\Omega$ whose second fundamental form is bounded below by $\kappa$, applying a conformal change of the Euclidean metric (as outlined in \citet{wang2007estimates}, Lemma 2.1) allows for a simplification to a convex domain case, which results in a less explicit bound specified in \eqref{eqn:Cpbound}.
However, it should be noted that $C_{\rm P}=\mathcal{O}(1/{\rm diam}(\Omega)^2)$ still holds true when $\diam(\Omega)\ll 1$.
\end{remark}
\subsection{Convergence in 2-Wasserstein Distance}
We introduce the $p$-Wasserstein distance between two Borel probability measures $\mu$ and $\pi$ on $\mathbb{R}^{2d}$
$$\mathcal{W}_p(\mu, \pi):=\inf_{\gamma\in \Gamma(\mu, \pi)}\left(\int_{\mathbb{R}^{2d}\times\mathbb{R}^{2d}}\|x-y\|^p d\gamma(x, y)\right)^{1/p},$$
where $\Gamma(\mu, \pi)$ denotes all joint coupled distributions $\gamma$ with marginal distributions $\mu$ and $\pi$. Furthermore, we define the Kullback--Leibler (KL) divergence as
\begin{equation}
    D(\mu\|\pi):=\int_{\Omega\times\Omega} \left(\frac{\dd \mu}{\dd \pi}\right)\ln\left(\frac{\dd \mu}{\dd \pi}\right)\dd \pi.\label{eqn:KLdivergence}
\end{equation}
The convergence rate under $\mathcal{W}_2$-distance is dependent on the \emph{Log-Sobolev} inequality which controls the entropy throughout the Fisher information. The proof of the following lemma is referred to \citet{wang2014analysis2} (Corollary 3.5.3). 
\begin{lemma}\label{lemma:LogSob} 
    For any probability measure $\mu\ll \pi$ and satisfying $\dd\mu/\dd\pi\ge 0$ and \eqref{eqn:bdy}, the following Logarithmic Sobolev (Log-Sobolev) inequality holds:
    \begin{equation}
        D(\mu\|\pi)\le C_{\rm LS}\mathcal{E}\left(\sqrt{\frac{\dd \mu}{\dd\pi}}\right).\label{eqn:LogSob}
    \end{equation}
     where the best constant $C_{\rm LS}$ is called Log-Sobolev constant. 
\end{lemma}
\begin{figure*}[t]
\centering
\includegraphics[width=1.8\columnwidth]{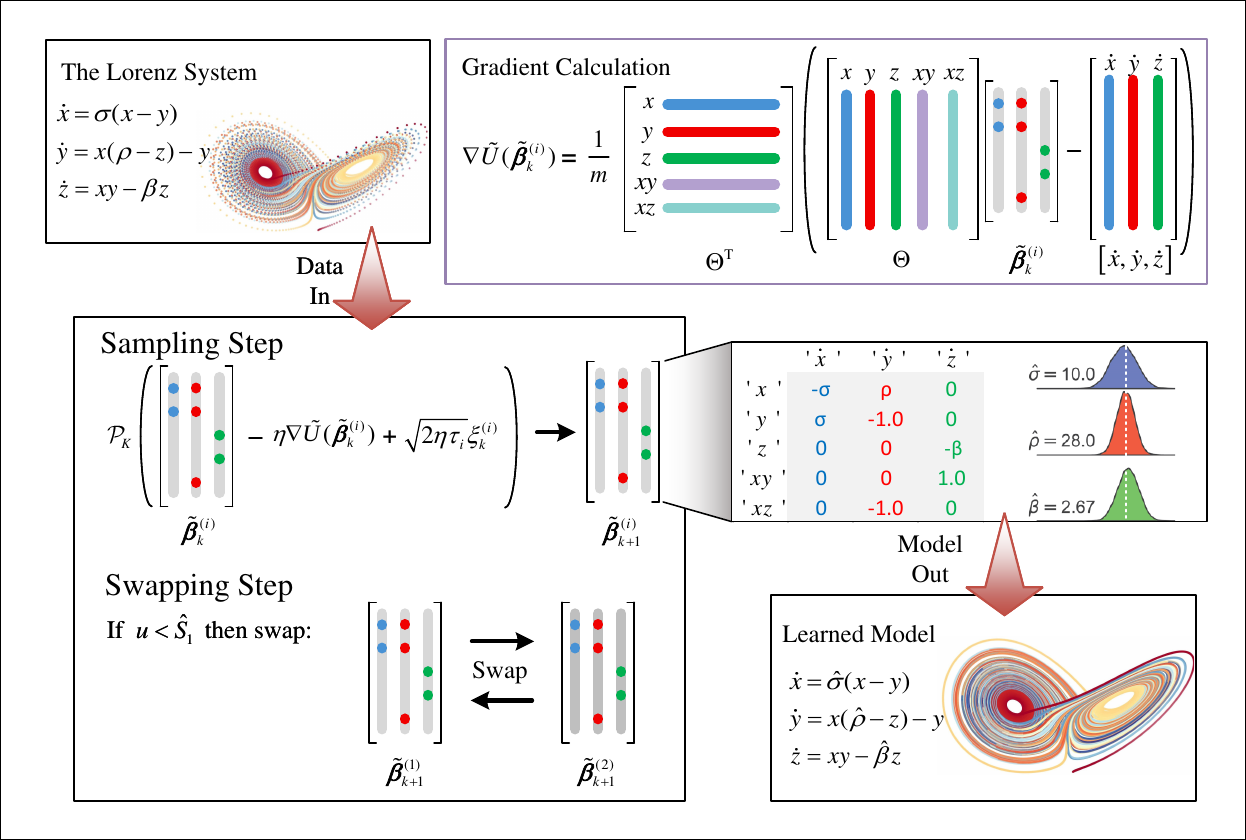}\vspace{-0.1 in}
\caption{Schematic of the r2SGLD algorithm, demonstrated on identifying the Lorenz system.}\label{fig:lorenz_framework}\vspace{-0.18 in}
\end{figure*}
\begin{theorem}\label{thm:W2conv}
 Given any initial measure $\mu_0$ for which $\dd\mu_0/\dd \pi\ge0$ and satisfying \eqref{eqn:bdy}, the 2-Wasserstein distance between $\mu_t$ and $\pi$ satisfies the following accelerated exponential decay estimate:
  \begin{equation}
    \mathcal{W}_2(\mu_t, \pi)\le \sqrt{2C_{\rm LS}D(\mu_0\|\pi)}\exp\left( -t(1+\delta_S)C_{\rm LS}^{-1} \right),
    \label{eqn:W2convergence}
  \end{equation}
  where $\displaystyle\delta_S:=\inf_{t>0}\frac{\mathcal{E}_S\left(\sqrt{ \frac{\dd \mu_t}{\dd \pi} }\right)}{\mathcal{E}\left( \sqrt{\frac{\dd \mu_t}{\dd \pi}} \right)}-1$ is the \emph{acceleration effect} in $\mathcal{W}_2$-distance.
\end{theorem} 
Now we apply the estimate for the Log-Sovolev constant for non-convex domain \cite{wang2007estimates} (Corollary 3.5.3):
\begin{lemma} The Log-Sobolev constant, defined in Lemma \ref{lemma:LogSob}, satisfies the following inequality:
\begin{equation}
    C_{\rm LS}\ge \frac{C}{\diam(\Omega)^2},\label{eqn:Clsbound}
\end{equation}
\end{lemma}
where $C$ is a constant depending on $\kappa$ and $U$. In particular, when $\Omega$ is convex, 
we have 
\begin{equation}
    C_{\rm LS}\ge \frac{\sqrt{1+4\pi^2}-1}{2\cdot {\rm diam}(\Omega)^2}.
\end{equation}
\begin{remark}
   For non-convex domain $\Omega$, as highlighted in Remark \ref{rmk:nonconvex}, the bound presented in \eqref{eqn:Clsbound} is less explicit.  
\end{remark}
For detailed derivations, readers may refer to Appendix \ref{sec:appendix_continuous}.

\vspace{-0.1 in}
\subsection{Discretization Analysis}
Due to the lack of higher-order local time estimates in the reflection term, our study on discrete-time dynamics of r2SGLD reveals its unique aspects compared to na\"ive reSGLD. Following the techniques in \citet{tanaka1979, sebastien_bubeck}, we successfully derive the upper bound for the discretization error within the $\mathcal{W}_1$-distance framework.
\begin{theorem}[Discretization error]\label{theorem:discretization}
Assume that the domain $\Omega$ is convex, and Assumptions \ref{assu:domainLip}, \ref{assu:ULip} hold true, then 
\begin{equation*}
\begin{split}
\mathcal W_1(\mu_T, \widetilde \mu_T)&\le \mathcal{\tilde O}\Big( \eta^{1/4} +  \sqrt{\max_{k}\mathbb E[\|\bphi_k\|^2]}\\
& \qq\qq + \sqrt{\eta^{-1/2}\max_{k}\sqrt{\mathbb E\left[|\psi_{k}|^2\right]}}\Big).
\end{split}
\end{equation*}
where $\widetilde\mu_T$ denotes the distribution of  $\widetilde \bbeta^{\eta}_T$, which is the continuous-time interpolation for r2SGLD, $\bphi_k:=\nabla \widetilde U-\nabla U$ is the noise in the stochastic gradient, and $\psi_k:=\widetilde S-S$ is the noise in the stochastic swapping rate. 
\end{theorem}
\vspace{-0.1 in}
Interested readers can refer to Appendix \ref{sec:appendix_discrete} for details.

\vspace{-0.1 in}
\section{Experiments}\label{section_experiments}
To validate the r2SGLD algorithm\footnote{Code is available at \href{https://github.com/haoyangzheng1996/r2SGLD}{github.com/haoyangzheng1996/r2SGLD}}, we begin by applying it to dynamical system identification (Section~\ref{section:exp_lorenz}), where the physical constraints are inherent in the dynamical system. Subsequently, its effectiveness in multi-mode distribution simulation is detailed in Section~\ref{section:exp_synthetic}. Lastly, Section~\ref{section:exp_class} illustrates the algorithm's performance in deep learning tasks.

\vspace{-0.1 in}
\subsection{Identifying the Lorenz Systems}\label{section:exp_lorenz}
The Lorenz system is a system of ordinary differential equations that underscores that chaotic systems can be completely deterministic and yet still be inherently unpredictable over long periods of time:
\begin{equation}
    \begin{aligned}
        \dot x(t) &=\sigma (y-x),\\
        \dot y(t) &=x(\rho -z)-y,\\
        \dot z(t) &=xy-\beta z,
    \end{aligned}
\end{equation}
where the system dynamics is determined by three unknown parameters: \(\sigma\) (the Prandtl number), \(\rho\) (the Rayleigh number), and \(\beta\) (the aspect ratio)\footnote{Following traditional conventions in the Lorenz system, we use $\beta$ to differentiate it from the parameters $\bbeta$ as defined in \eqref{eq:sde_2couple}.}. 
A prevalent approach to learning the Lorenz system involves the Sparse Identification of Nonlinear Dynamics \cite{brunton2016discovering, desilva2020, Kaptanoglu2022}. This method employs thresholding least squares using the position and its time derivative for $x(t)$, $y(t)$, and $z(t)$ at specific times:
\begin{equation}
    \mathbf{\dot X} = \Theta(\mathbf{X})\bbeta,
\end{equation}
where $\mathbf{\dot X}$ is the concatenation of multiple state velocities $[\dot x(t_m), \dot y(t_m), \dot z(t_m)]$ sampled at several times $t_1, t_2, \cdots, t_m, \cdots$. $\Theta(\mathbf{X})$ is the candidate basis from the sampled states, and $\bbeta$ is a sparse matrix to determine which candidate basis is active. In our work, this method is tailored by simplifying the sparse coefficient matrix, which targets the learning of the unknown parameters $\sigma$, $\rho$, and $\beta$. Another critical aspect is that our approach relies solely on gradient information of the potentials according to state velocities, rather than having access to the sampled states:
\begin{equation}
    \nabla \widetilde U(\widetilde \bbeta_{k}) = \frac{1}{m} \Theta^\intercal \left(\Theta\widetilde \bbeta_{k}-\mathbf{\dot X}\right),
\end{equation}
where the matrix $\Theta$ from $\mathcal R^{m\times 5}$ here encompasses five candidate basis functions associated with $x(t_m)$, $y(t_m)$, $z(t_m)$, $x(t_m)y(t_m)$, and $x(t_m)z(t_m)$. Additionally, the matrix $\widetilde \bbeta \in \mathcal R^{5\times 3}$, an approximation of $\bbeta$ containing parameters $\sigma, \rho, \beta$, interacts with $\mathbf{\dot X}$ from $\mathcal R^{m\times 3}$, representing $m$ sampled state velocities from $t_1$ to $t_m$. 

A general framework to identify the Lorenz system with r2SGLD algorithm is shown in Figure~\ref{fig:lorenz_framework}. The data for this work is generated using a six-stage, fifth-order Runge-Kutta method~\cite{atkinson1991introduction}. This method is applied to compute states from time 0 to 100 at intervals of 0.01, followed by calculating the velocity of $\mathbf{\dot X}$ using elementary matrix operations. Once the sampling and swapping steps are complete in each iteration, the Runge-Kutta method is used with the latest sampled model parameters, which aims to derive both the approximate states matrix $\widetilde{\mathbf{X}}$ and the candidate basis $\Theta(\widetilde{\mathbf{X}})$ for the relevant time stamp. Upon convergence of the empirical distribution to its stationary counterpart or upon satisfying certain stopping conditions, the sampling and swapping steps are halted. The outputs are the empirical posterior modes of the unknown parameters, which are used to recover the target dynamical systems.

\begin{figure*}[!ht]
\centering
\subfigure[Truth]{
\centering
\includegraphics[width=.55\columnwidth]{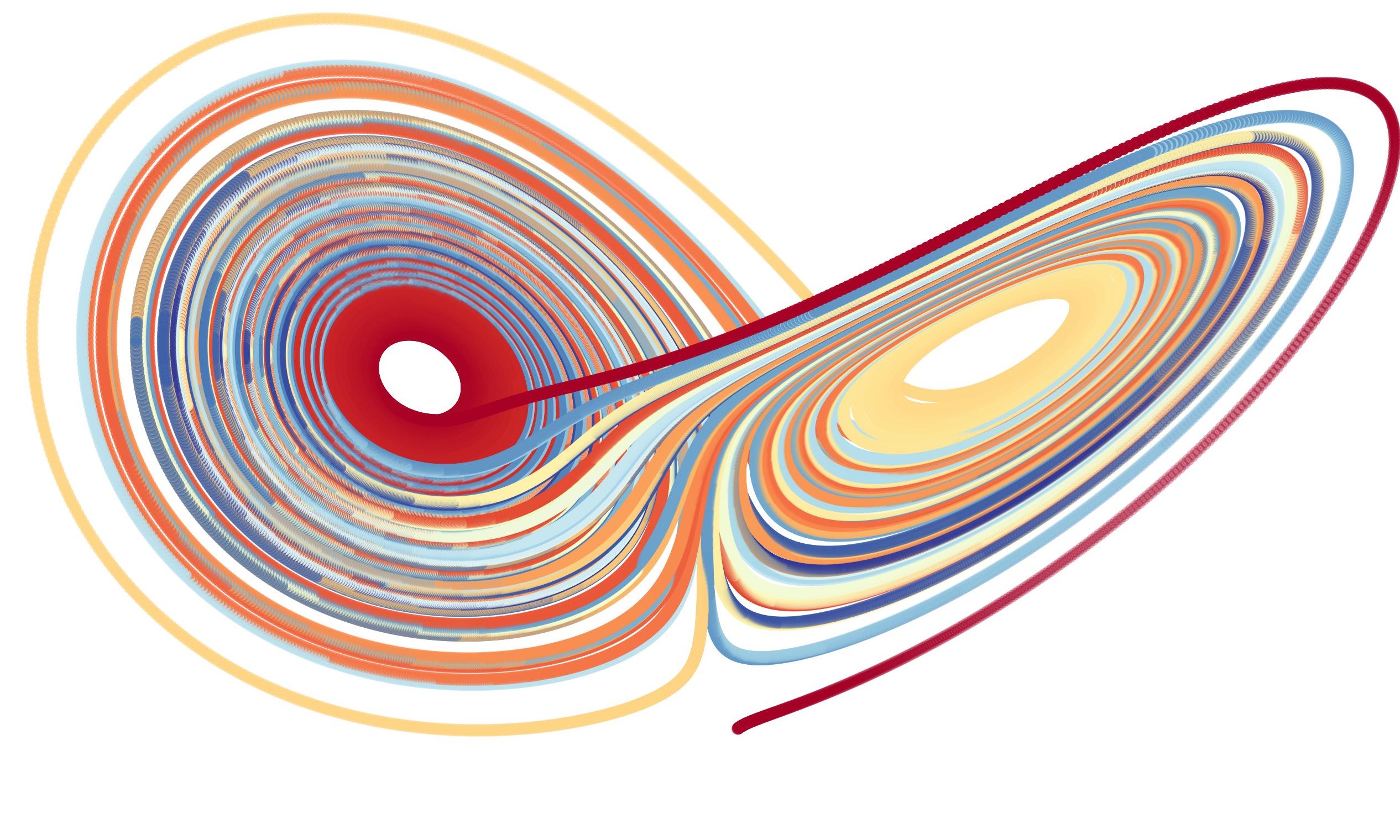}\label{fig:lorenz_sim_truth}
}
\subfigure[reSGLD]
{
\centering
\includegraphics[width=.55\columnwidth]{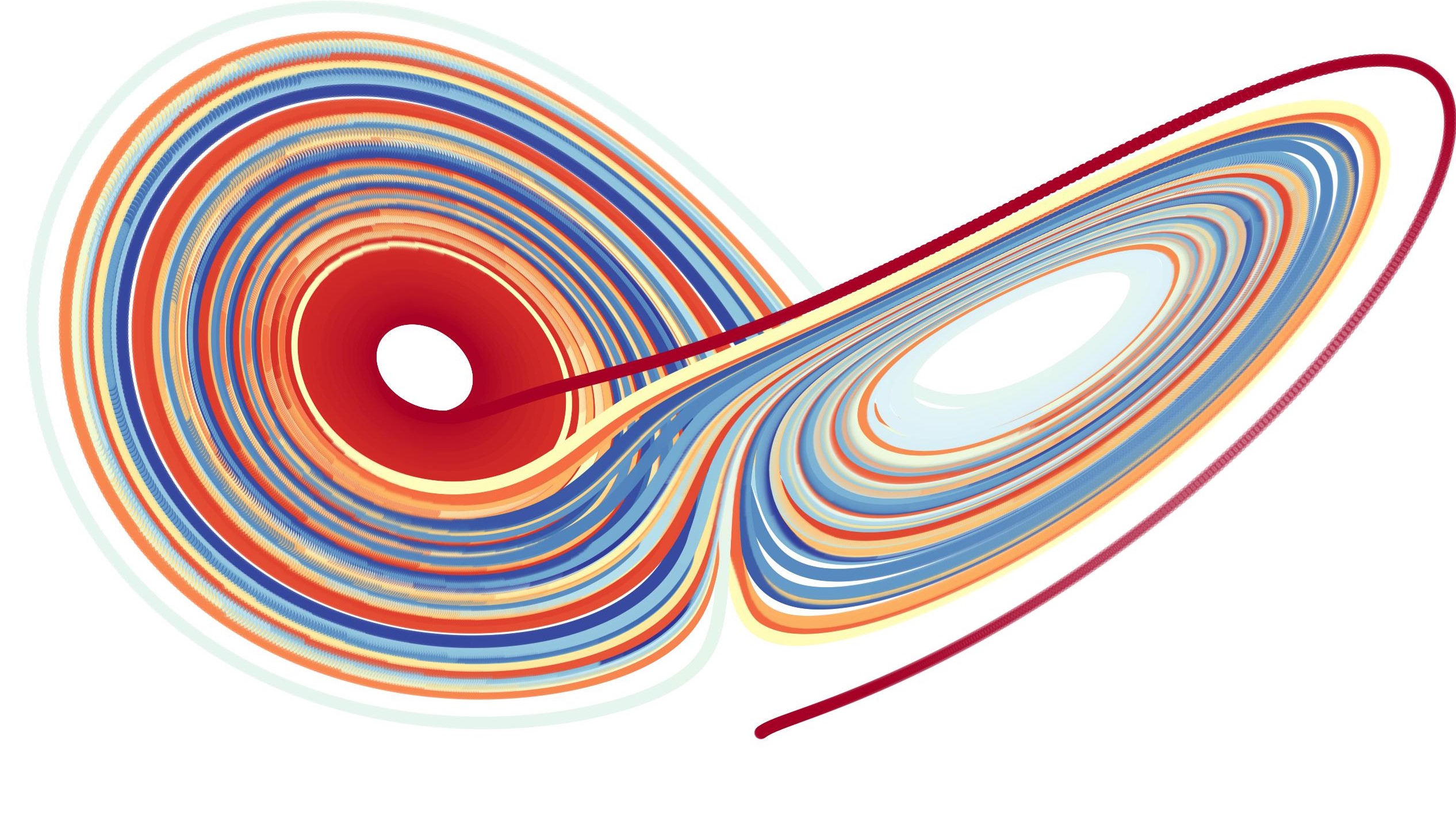}\label{fig:lorenz_sim_resgld}
}
\subfigure[r2SGLD]
{
\centering
\includegraphics[width=.55\columnwidth]{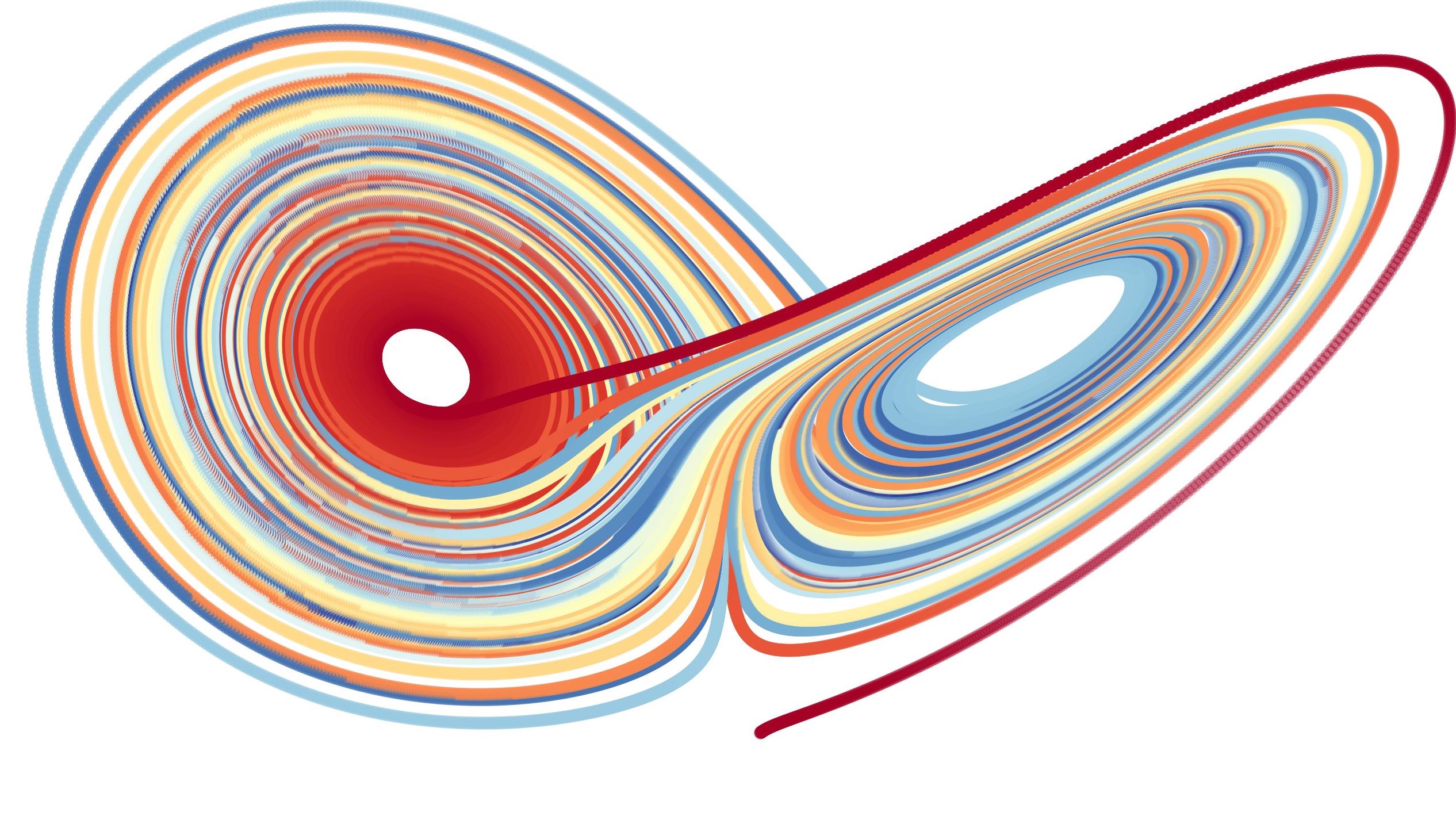}\label{fig:lorenz_sim_r2SGLD}
}
\vspace{-0.15 in}
\caption{Simulation of the Lorenz system based on the empirical posterior modes of model parameters.}
\label{fig:lorenz_sim}\vspace{-0.1 in}
\end{figure*}

We apply dual-chain reSGLD and r2SGLD for sampling within the parameter space to estimate the posterior mode of unknown parameters ($\hat\sigma$, $\hat\rho$, and $\hat\beta$). For r2SGLD, its sampling process adheres to physical constraints requiring the positivity of parameters ($\sigma, \rho, \beta>0$). To ensure global stability within the system, two additional constraints are implemented: one mandates a high rate of viscous dissipation to inhibit amplification of minor perturbations in the system's velocity field ($\sigma > 1 + \beta$), and the other requires strong thermal forcing to counterbalance the stabilizing effects of viscous dissipation ($\rho > 1$). To streamline the algorithm's performance, we enforce accurate values for model parameters associated with the less critical candidate basis, which directs the algorithm's focus primarily toward identifying the essential model parameters.

\begin{figure}[htbp]
\centering
\subfigure[reSGLD]{
\centering
\includegraphics[width=.4\columnwidth]{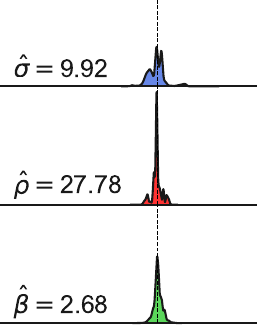}\label{fig:lorenz_posterior_no}
}\hspace{-0mm}
\subfigure[r2SGLD]{
\centering
\includegraphics[width=.4\columnwidth]{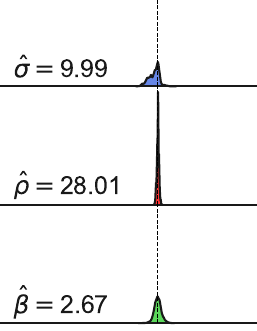}\label{fig:lorenz_posterior_with}
}\vspace{-0.1 in}
\caption{Posterior distributions of the identified model parameters.}
\label{fig:lorenz_posterior}\vspace{-0.1 in}
\end{figure}

\begin{figure*}
  \centering
  \begin{tabular}{ c c }
    \includegraphics[width=.165\linewidth]{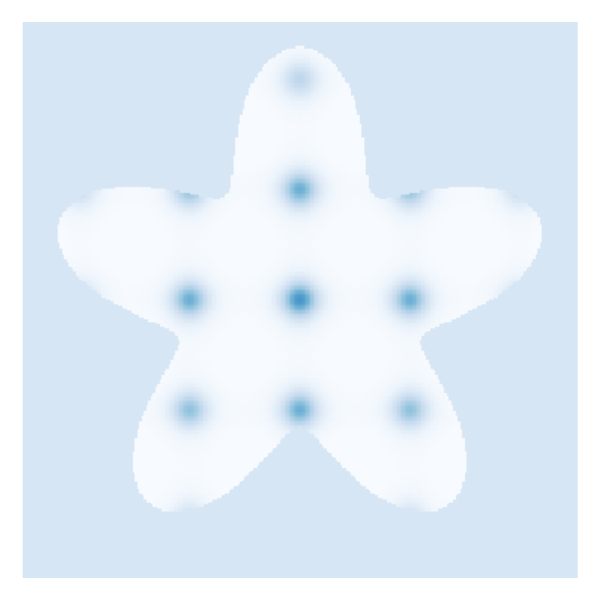} &
      \includegraphics[width=.165\linewidth]{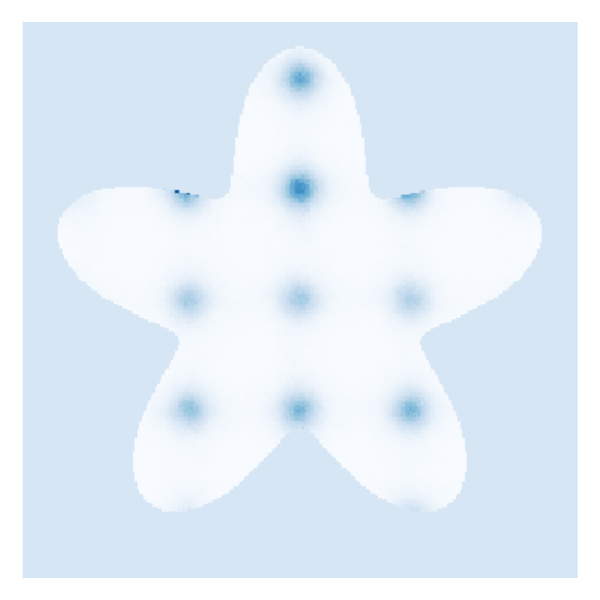} \\
      (a) Truth & (b) R-SGLD \\
    \includegraphics[width=.165\linewidth]{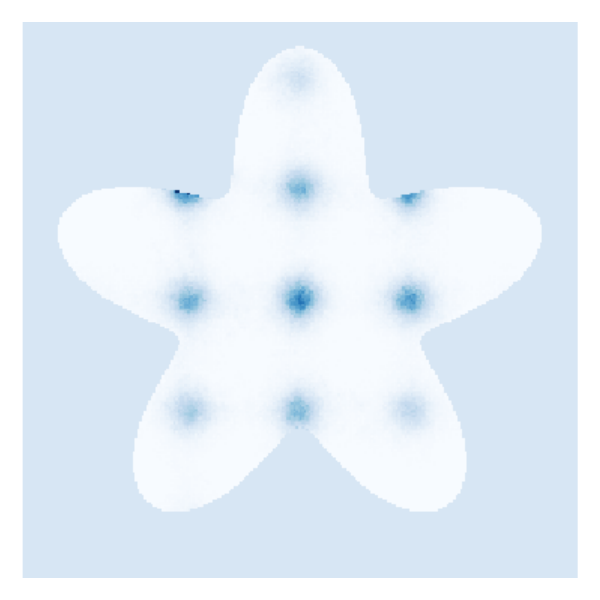} &
      \includegraphics[width=.165\linewidth]{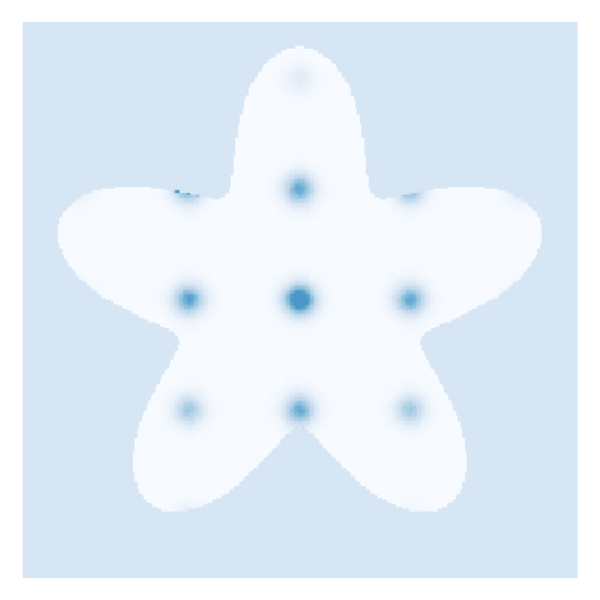} \\
      (c) R-cycSGLD & (d) r2SGLD \\
  \end{tabular}\hspace{-.1in}
  \begin{tabular}{ c }\vspace{.18 in}
    \includegraphics[width=.5144\linewidth]{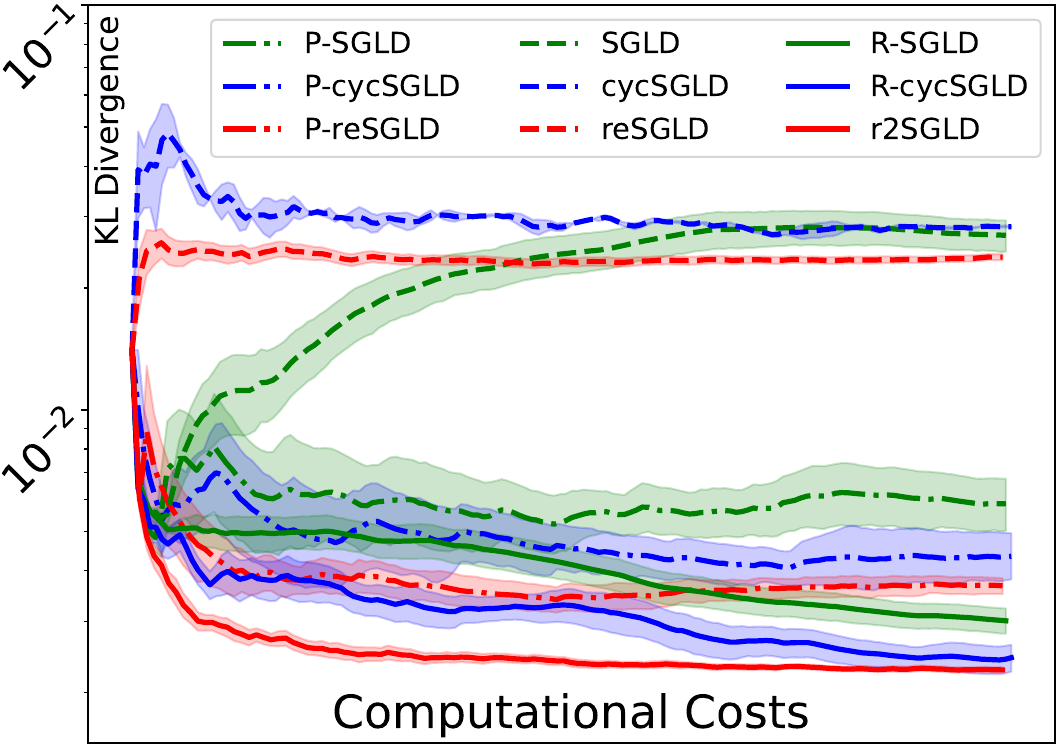} \vspace{-.18 in}\\
     (e) Kullback–Leibler Divergence \\
  \end{tabular}%
  \vspace{-.05 in}
  \caption{Empirical behavior on multi-mode distributions with flower-shaped boundaries.}\label{fig:sim_multimode}
\vspace{-.15 in}
\end{figure*}

Upon adhering to the specified conditions, we evaluate the effectiveness of reflection in reSGLD by contrasting the empirical posteriors with reflection (Figure~\ref{fig:lorenz_posterior_with}) against those without reflection (Figure~\ref{fig:lorenz_posterior_no}). The presented figures illustrate estimated values of the empirical posterior modes of $\hat\sigma$, $\hat\rho$, and $\hat\beta$, as generated by the reSGLD and r2SGLD methods. Notably, the posterior distributions displayed here are normalized to facilitate a more straightforward and clear comparison. The result reveals that the posterior modes derived from r2SGLD align more closely with the true parameters ($\sigma=10.0$, $\rho=28.0$, $\beta=\frac{8}{3}$) compared to those from reSGLD. Notably, the empirical posterior $\hat\sigma$ from reSGLD exhibits dual peaks, whereas r2SGLD's posterior demonstrates a singular peak, which indicates a more reasonable result. For simulating the dynamics of the Lorentz system, the learned posterior modes are utilized as model parameters, with the simulations presented in Figure~\ref{fig:lorenz_sim}. While the simulations from reSGLD and r2SGLD are closely matched, r2SGLD demonstrates slightly improved results over extended simulation periods. This implies that r2SGLD yields more reliable outcomes. Further details and results, including baseline comparisons and exploration of identifying the Lotka-Volterra model, are available in \ref{sec:lorenz_sup} and \ref{sec:lotka_volterra_sup}.

\subsection{Constrained Multi-modal Simulations}\label{section:exp_synthetic}

In this study, we examine the enhanced sample efficiency of r2SGLD through simulating bounded multi-modal distributions. Our proposed algorithm (with two chains), alongside standard baselines such as (reflected) SGLD (SGLD and R-SGLD), (reflected) cyclical SGLD (cycSGLD and R-cycSGLD) \cite{ruqi2020}, and reSGLD, employs gradient information to generate samples and approximate posterior distributions. This comparative analysis aims to demonstrate the acceleration effect offered by the r2SGLD.

Our study focuses on a multi-modal density characterized by a 2D mixture of 25 Gaussian distributions within a flower-shaped boundary, as depicted in Figure~\ref{fig:sim_multimode}(a). The boundary is mathematically defined as:
\begin{equation*}
    r=\sin (2 \pi p t)+m, \quad x=r \cos (2 \pi t), \quad y=r \sin (2 \pi t),
\end{equation*}
where $r$ denotes the radius, $p$ the number of petals, $m$ the extent of radial displacement, and $x$, $y$ are the horizontal and vertical coordinates, respectively. For the target distribution, we set the boundary with parameters $p=5$ and $m=3$.

The empirical posterior distributions generated by constrained sampling methods (R-SGLD, R-cycSGLD, and r2SGLD) are exhibited in Figure~\ref{fig:sim_multimode}. In this context, R-SGLD (Figure~\ref{fig:sim_multimode}(b)) exhibits the least effective performance, which notably fails to quantify the weights of each mode accurately. The performance is moderately improved with R-cycSGLD (Figure~\ref{fig:sim_multimode}(c)), but it remains improvement spaces. In contrast, the results given by r2SGLD (Figure~\ref{fig:sim_multimode}(d)) show commendable performance. 

Additionally, we consider a penalty-based approach as a baseline for simulating constrained multi-modal distributions. We incorporated an L2 regularization term in the objective function to penalize samples that exit the designated region, defined as:
\begin{equation*}
    \nabla \widetilde{U}_{\text{reg}}(\widetilde{\bbeta}_k) = \nabla \widetilde{U}(\widetilde{\bbeta}_k) + \xi \widetilde{\bbeta}_k
\end{equation*}
where $\xi$ represents the shrinkage coefficient. We alternate between the unpenalized objective $\widetilde{U}(\widetilde{\bbeta}_k)$ within the region and the penalized $\widetilde{U}_{\text{reg}}(\widetilde{\bbeta}_k)$ when out-of-bounds. This approach is evaluated using three penalized variants: penalized SGLD (P-SGLD), penalized cyclic SGLD (P-cycSGLD), and penalized reSGLD (P-reSGLD). Figure~\ref{fig:sim_multimode} (e) presents a comparative analysis of various sampling methods, including SGLD, cycSGLD, reSGLD, P-SGLD, P-cycSGLD, P-reSGLD, R-SGLD, R-cycSGLD, and r2SGLD, focusing on their KL divergence under the same computational cost. This analysis incorporates 95\% confidence intervals derived from 10 runs under different initial conditions. 

From the result, the implementation of a reflection operation markedly reduces KL divergence by improving sample efficiency through the redirection of out-of-bound samples. While the L2 regularization helps maintain samples within the boundary, it confirms that the reflection-based algorithm outperforms these penalized variants. The comparative inefficiency of the penalized approaches stems from their acceptance of some out-of-boundary samples, which compromises the sample efficiency. This further highlights the superiority of the r2SGLD algorithm in avoiding over-exploration and improving sample efficiency. Among the reflection-based sampling methods, R-SGLD reports the least favorable KL divergence, possibly due to its tendency to remain in local modes. R-cycSGLD outperforms R-SGLD owing to its adaptive learning rate, which facilitates both exploration and exploitation. The most effective performance is observed in r2SGLD, where the dual-chain structure allows for more efficient exploration. 

To empirically substantiate our theoretical assertion that reducing the domain's diameter enhances mixing rates with quadratic dependency, we conducted further experiments with r2SGLD. These experiments explored the impact of domain diameters, set between 1.5 and 3.0, on mixing rates within convex bounded domains versus a non-convex distribution with 25 Gaussian modes. For detailed experimental setups and further discussions, please see Section~\ref{sec:multimode_sup}.

\subsection{Non-convex Optimization for Image Classifications}\label{section:exp_class}

We further extend the testing to CIFAR100 benchmarks, which utilize 20 and 56-layer residual networks (ResNet20 and ResNet56, respectively) for training and testing. Our performance evaluation considers modified baseline algorithms with momentum terms, which is crucial for intricate tasks like image classification. This integration of gradient and curvature information enables a more effective exploration of the parameter space to improve model performance. Therefore, our algorithmic baseline includes stochastic gradient with momentum (SGDM), its reflected variant (R-SGDM), (reflected) stochastic gradient Hamiltonian Monte Carlo (SGHMC and R-SGHMC), (reflected) cyclical SGHMC (cycSGHMC and R-cycSGHMC) featuring a cyclical learning rate. To highlight the performance of the proposed algorithm, we include (reflected) replica exchange SGHMC (reSGHMC and r2SGHMC) for the test.

Our study on CIFAR 100 with the proposed algorithm first unveiled a notable advantage: employing reflection operation enables larger learning rates\footnote{In empirical Deep Neural Network (DNN) training using SGD, a larger learning rate not only increases discretization errors but also effectively raises the temperature of the resulting Markov chain \cite{Mandt}. We specify that ``large learning rates" are primarily employed during the initial training stages to explore the landscape, and are gradually decayed to implement simulated annealing for enhanced global optimization.}, which is fundamental for detailed model exploration. Algorithms without reflected projection often exhibit failure (test accuracy under 10\% or infinite training loss), particularly at initial learning rates over 10.0\footnote{The initial learning rate for CIFAR 100 is scaled to 2e-4 when accounting for the training data size of 50,000.}. Conversely, setting reflection operation between [-4.0, 4.0] not only stabilizes the model but also improves test accuracy to over 70\%. This highlights the critical role of reflected projection in ensuring successful learning in complex tasks. Moreover, it indicates the potential for state-of-the-art results given sufficient computational resources for model exploration.
\begin{table}[!htbp]
\begin{sc}
\vspace{-0.1in}
\caption[Table caption text]{Uncertainty estimation comparison between algorithms with reflection and without reflection. } \label{tab:cifar_record}
\vspace{-0.1in}
\small
\begin{center} 
        \begin{tabular}{c|cccc}
        \hline
        \multirow{2}{*}{\scriptsize Methods} & \multicolumn{3}{c}{Metrics (ResNet20)}                   \\ \cline{2-4} 
                                & ACC (\%) $\uparrow$ & NLL $\downarrow$     & Brier $(\mbox{\textperthousand})$ $\downarrow$     \\ \hline \hline
        \scriptsize{SGDM} & $72.13\pm0.60$      & $9667\pm108$      & $2.78\pm0.05$      \\ 
        \scriptsize{SGHMC}            & $72.47\pm0.45$      & $9543\pm157$      & $2.75\pm0.05$      \\ 
        \scriptsize{\upshape{cyc}SGHMC}  & {$73.49\pm0.17$} & {$8913\pm 76$} & {$2.65\pm0.02$} \\ 
        \scriptsize{\upshape{re}SGHMC}  & {$75.01\pm0.14$} & {$8552\pm 69$} & {$2.50\pm0.01$} \\  \hline   
        \scriptsize{r-SGDM} & $72.43\pm0.35$      & $9626\pm 94$      & $2.75\pm0.03$      \\ 
        \scriptsize{r-SGHMC}    & {$72.85\pm0.51$} & {$9501\pm167$} & {$2.73\pm0.05$} \\ 
        \scriptsize{r-\upshape{cyc}SGHMC} & {$73.77\pm0.22$} & {$8953\pm52$} & {$2.62\pm0.02$} \\ 
        \textbf{\scriptsize{\upshape{r}2SGHMC}}  & \bm{$75.38\pm0.17$} & \bm{$8489\pm66$} & \bm{$2.46\pm0.02$} \\  \hline
        \end{tabular}
\end{center} 
\begin{center} 
        \begin{tabular}{c|cccc}
        \hline
        \multirow{2}{*}{\scriptsize Methods} & \multicolumn{3}{c}{Metrics (ResNet56)}                   \\ \cline{2-4} 
                                & ACC (\%) $\uparrow$ & NLL $\downarrow$     & Brier $(\mbox{\textperthousand})$ $\downarrow$     \\ \hline \hline
        \scriptsize{SGDM} & {$74.40\pm0.71$} & {$9724\pm169$} & {$3.59\pm0.23$} \\
        \scriptsize{SGHMC} & {$74.22\pm0.66$} & {$9723\pm214$} & {$3.23\pm0.21$} \\
        \scriptsize{\upshape{cyc}SGHMC} & {$77.98\pm0.61$} & {$8303\pm161$} & {$3.19\pm0.20$} \\
        \scriptsize{\upshape{re}SGHMC} & {$78.87\pm0.44$} & {$7406\pm130$} & {$2.94\pm0.06$} \\ \hline       
        \scriptsize{r-SGDM} & {$74.70\pm0.68$} & {$9507\pm106$} & {$3.53\pm0.18$} \\
        \scriptsize{r-SGHMC }    & {$75.10\pm0.55$} & {$9232\pm158$} & {$3.36\pm0.23$} \\ 
        \scriptsize{r-\upshape{cyc}SGHMC } & {$78.41\pm0.67$} & {$7711\pm144$} & {$3.12\pm0.11$} \\
        \textbf{\scriptsize{\upshape{r}2SGHMC}} & \bm{$79.39\pm0.30$} & \bm{$7155\pm91$} & \bm{$2.89\pm0.02$} \\  \hline
        \end{tabular}
\end{center} 
\end{sc}
\vspace{-0.3in}
\end{table}

Although a large learning rate has the potential to offer better model performance, we do not directly use a learning rate of more than 10.0 due to computational resource limitations. Instead, we investigate the optimal learning rate across 1,000 epochs for our proposed algorithms and baselines (see \ref{subsec:learn_rate} for details). In our study, cycSGHMC and R-cycSGHMC are implemented with a triple-cycle cosine learning rate schedule, while reSGHMC and r2SGHMC employ four chains per model to facilitate exhaustive exploration. 
It should be noted that incorporating advanced methodologies for chain swapping is essential in this task, and we have strategically adopted more sophisticated schemes for chain swap in both the reSGLD and r2SGLD frameworks. This adaptation is crucial as we engage with multiple chains instead of the conventional two. For an in-depth exploration of these enhanced chain swap mechanisms, readers are encouraged to refer to \ref{sec:classification_sup}. 
The batch size is 2,048 across all methods. We repeat experiments for each algorithm ten times to record the mean and two standard deviations for metrics such as Bayesian model averaging (BMA), negative log-likelihoods (NLL), and Brier scores (Brier). 

Upon comparative evaluation of various algorithms shown in Table~\ref{tab:cifar_record}, we observe distinct performance patterns. Algorithms with reflection, particularly r2SGHMC, consistently outperform their non-reflective counterparts across all metrics. Specifically, r2SGHMC achieves the highest accuracy (BMA), NLL, and Brier score, both in ResNet20 and ResNet56 models. While standard SGDM and SGHMC show moderate performance, their reflected versions (R-SGDM and R-SGHMC) exhibit improvements, which indicates the effectiveness of reflection in model optimization. Notably, cycSGHMC and R-cycSGHMC demonstrate significant performance boosts, attributed to their cyclical learning rate schedules. Our proposed method, r2SGHMC, characterized by its sophisticated chain-swapping mechanism, has proven to be highly effective, which makes it well-suited for complex deep-learning tasks.

\section{Conclusion and Discussion}

By avoiding unnatural samples and enabling efficient exploration within bounded domains, the r2SGLD algorithm marks a significant advancement in the non-convex exploration of reSGLD, particularly in high-temperature chains. This improvement is substantiated through both theoretical analysis and empirical validation.

Our theoretical investigation comprises two parts: analyzing convergence in continuous scenarios and evaluating the discretization errors of our algorithm. For continuous-time diffusion, we demonstrate its convergence in both $\chi^2$-divergence and $\mathcal{W}_2$-distance under non-convex bounded domains. This analysis reveals a quantitative correlation between the constraint radius and the convergence rate, where a smaller radius facilitates accelerating the mixing rate, with rates decaying as quadratic. For discrete-time dynamics, we further provide analysis of the discretization error measured by the $\mathcal{W}_1$-distance under convex bounded domains. 

Experimentally, we establish the versatility of our algorithm in diverse applications. In the realm of dynamical systems, our method pioneers the use of constrained sampling for parameter identification. We offer robust posterior distributions to demonstrate the importance of reflection operations. In sampling from constrained multi-modal distributions, r2SGLD outperforms other methods, as evidenced both in empirical distribution representation and rapid KL divergence convergence. In non-convex optimization tasks within deep learning, r2SGLD successfully handles large learning rates, a task where na\"ive reSGLD struggles, due to the absence of a reflection mechanism. Despite computational limitations necessitating a reduced learning rate, our algorithm still outperforms others across various metrics. It is noteworthy that these tests were conducted within 1,000 epochs. We anticipate even better results with higher learning rates and with sufficient training epochs.

\section*{Acknowledgements}
We thank the anonymous reviewers for their helpful comments. Q. Feng is partially supported by the National Science Foundation (DMS-2306769). G. Lin acknowledges the support of the National Science Foundation (DMS-2053746, DMS-2134209, ECCS-2328241, and OAC-2311848), the U.S. Department of Energy (DOE) Office of Science Advanced Scientific Computing Research (DE-SC0023161), and DOE–Fusion Energy Science (DE-SC0024583).

\clearpage

\section*{Impact Statement}
This paper presents work whose goal is to advance the field of Machine Learning. There are many potential societal consequences of our work, none of which we feel must be specifically highlighted here.
\bibliography{mybib,mybib2}

\begin{thebibliography}{66}
\providecommand{\natexlab}[1]{#1}
\providecommand{\url}[1]{\texttt{#1}}
\expandafter\ifx\csname urlstyle\endcsname\relax
  \providecommand{\doi}[1]{doi: #1}\else
  \providecommand{\doi}{doi: \begingroup \urlstyle{rm}\Url}\fi

\bibitem[Ahn \& Chewi(2021)Ahn and Chewi]{ahn2021efficient}
Ahn, K. and Chewi, S.
\newblock {Efficient Constrained Sampling via the Mirror-Langevin Algorithm}.
\newblock \emph{Advances in Neural Information Processing Systems (NeurIPS)}, 34:\penalty0 28405--28418, 2021.

\bibitem[Ahn et~al.(2012)Ahn, Korattikara, and Welling]{Ahn12}
Ahn, S., Korattikara, A., and Welling, M.
\newblock {B}ayesian {P}osterior {S}ampling via {S}tochastic {G}radient {F}isher {S}coring.
\newblock In \emph{Proc. of the International Conference on Machine Learning (ICML)}, 2012.

\bibitem[Atkinson(1991)]{atkinson1991introduction}
Atkinson, K.
\newblock \emph{{An Introduction to Numerical Analysis}}.
\newblock John wiley \& sons, 1991.

\bibitem[Bakry et~al.(2008)Bakry, F.~Barthe, and Guillin]{Bakry08}
Bakry, D., F.~Barthe, P.~C., and Guillin, A.
\newblock A {S}imple {P}roof of the {P}oincaré {I}nequality for {A} {L}arge {C}lass of {P}robability {M}easures.
\newblock \emph{Electron. Comm. Probab.}, 13:\penalty0 60--66, 2008.

\bibitem[Bakry et~al.(2014)Bakry, Gentil, and Ledoux]{Bakry20142}
Bakry, D., Gentil, I., and Ledoux, M.
\newblock {A}nalysis and {G}eometry of {M}arkov {D}iffusion {O}perators.
\newblock \emph{Springer}, 2014.

\bibitem[Berg \& Neuhaus(1992)Berg and Neuhaus]{berg1992multicanonical}
Berg, B.~A. and Neuhaus, T.
\newblock {Multicanonical Ensemble: A New Approach to Simulate First-Order Phase Transitions}.
\newblock \emph{{Physical Review Letters}}, 68\penalty0 (1):\penalty0 9, 1992.

\bibitem[Brenner et~al.(2007)Brenner, Sweet, VonHandorf, and Izaguirre]{Brenner07}
Brenner, P., Sweet, C.~R., VonHandorf, D., and Izaguirre, J.~A.
\newblock Accelerating the {R}eplica {E}xchange {M}ethod through an {E}fficient {A}ll-pairs {E}xchange.
\newblock \emph{The Journal of Chemical Physics}, 126:\penalty0 074103, 2007.

\bibitem[Brosse et~al.(2017)Brosse, Durmus, Moulines, and Pereyra]{brosse2017sampling}
Brosse, N., Durmus, A., Moulines, {\'E}., and Pereyra, M.
\newblock {Sampling from a Log-Concave Distribution with Compact Support with Proximal Langevin Monte Carlo}.
\newblock In \emph{Proc. of Conference on Learning Theory (COLT)}, pp.\  319--342. PMLR, 2017.

\bibitem[Brunton et~al.(2016)Brunton, Proctor, and Kutz]{brunton2016discovering}
Brunton, S.~L., Proctor, J.~L., and Kutz, J.~N.
\newblock {Discovering Governing Equations from Data by Sparse Identification of Nonlinear Dynamical Systems}.
\newblock \emph{{Proceedings of the National Academy of Sciences}}, 113\penalty0 (15):\penalty0 3932--3937, 2016.

\bibitem[Bubeck et~al.(2018)Bubeck, Eldan, and Lehec]{sebastien_bubeck}
Bubeck, S., Eldan, R., and Lehec, J.
\newblock {Sampling from a Log-Concave Distribution with Projected Langevin Monte Carlo}.
\newblock \emph{Discrete Comput Geom}, pp.\  757–783, 2018.

\bibitem[Campbell et~al.(2021)Campbell, Chen, Stimper, Hernandez-Lobato, and Zhang]{campbell2021gradient}
Campbell, A., Chen, W., Stimper, V., Hernandez-Lobato, J.~M., and Zhang, Y.
\newblock {A Gradient Based Strategy for Hamiltonian Monte Carlo Hyperparameter Optimization}.
\newblock In \emph{Proc. of the International Conference on Machine Learning (ICML)}, pp.\  1238--1248. PMLR, 2021.

\bibitem[Cattiaux et~al.(2010)Cattiaux, Guillin, and Wu]{Cattiaux2010notes}
Cattiaux, P., Guillin, A., and Wu, L.-M.
\newblock A {N}ote on {T}alagrand’s {T}ransportation {I}nequality and {L}ogarithmic {S}obolev {I}nequality.
\newblock \emph{Probability Theory and Related Fields}, 148:\penalty0 285--334, 2010.

\bibitem[Chen et~al.(2015)Chen, Ding, and Carin]{Chen15}
Chen, C., Ding, N., and Carin, L.
\newblock On the {C}onvergence of {S}tochastic {G}radient {MCMC} {A}lgorithms with {H}igh-order {I}ntegrators.
\newblock In \emph{Advances in Neural Information Processing Systems (NeurIPS)}, pp.\  2278--2286, 2015.

\bibitem[Chen et~al.(2019)Chen, Chen, Dong, Peng, and Wang]{chen2018accelerating}
Chen, Y., Chen, J., Dong, J., Peng, J., and Wang, Z.
\newblock {A}ccelerating {N}onconvex {L}earning via {R}eplica {E}xchange {L}angevin {D}iffusion.
\newblock In \emph{Proc. of the International Conference on Learning Representation (ICLR)}, 2019.

\bibitem[\c{S}im\c{s}ekli et~al.(2016)\c{S}im\c{s}ekli, Badeau, Cemgil, and Richard]{Simsekli2016}
\c{S}im\c{s}ekli, U., Badeau, R., Cemgil, A.~T., and Richard, G.
\newblock Stochastic {Q}uasi-{N}ewton {L}angevin {M}onte {C}arlo.
\newblock In \emph{Proc. of the International Conference on Machine Learning (ICML)}, pp.\  642--651, 2016.

\bibitem[de~Silva et~al.(2020)de~Silva, Champion, Quade, Loiseau, Kutz, and Brunton]{desilva2020}
de~Silva, B., Champion, K., Quade, M., Loiseau, J.-C., Kutz, J., and Brunton, S.
\newblock {Pysindy: A Python Package for the Sparse Identification of Nonlinear Dynamical Systems from Data}.
\newblock \emph{Journal of Open Source Software}, 5\penalty0 (49):\penalty0 2104, 2020.

\bibitem[Deng et~al.(2020{\natexlab{a}})Deng, Feng, Gao, Liang, and Lin]{deng2020}
Deng, W., Feng, Q., Gao, L., Liang, F., and Lin, G.
\newblock Non-{C}onvex {L}earning via {R}eplica {E}xchange {S}tochastic {G}radient {MCMC}.
\newblock In \emph{Proc. of the International Conference on Machine Learning (ICML)}, 2020{\natexlab{a}}.

\bibitem[Deng et~al.(2020{\natexlab{b}})Deng, Lin, and Liang]{CSGLD}
Deng, W., Lin, G., and Liang, F.
\newblock A {C}ontour {S}tochastic {G}radient {L}angevin {D}ynamics {A}lgorithm for {S}imulations of {M}ulti-modal {D}istributions.
\newblock In \emph{Advances in Neural Information Processing Systems (NeurIPS)}, 2020{\natexlab{b}}.

\bibitem[Deng et~al.(2022)Deng, Liang, Hao, Lin, and Liang]{icsgld}
Deng, W., Liang, S., Hao, B., Lin, G., and Liang, F.
\newblock Interacting contour stochastic gradient langevin dynamics.
\newblock In \emph{Proc. of the International Conference on Learning Representation (ICLR)}, 2022.

\bibitem[Deng et~al.(2023)Deng, Zhang, Feng, Liang, and Lin]{deng2023non}
Deng, W., Zhang, Q., Feng, Q., Liang, F., and Lin, G.
\newblock {Non-reversible Parallel Tempering for Deep Posterior Approximation}.
\newblock In \emph{Proc. of the National Conference on Artificial Intelligence (AAAI)}, volume~37, pp.\  7332--7339, 2023.

\bibitem[Deng et~al.(2024)Deng, Chen, Yang, Du, Feng, and Chen]{rSB}
Deng, W., Chen, Y., Yang, N., Du, H., Feng, Q., and Chen, R. T.~Q.
\newblock {Reflected Schr\"odinger Bridge for Constrained Generative Modeling}.
\newblock In \emph{Proc. of the Conference on Uncertainty in Artificial Intelligence (UAI)}, 2024.

\bibitem[Dong \& Tong(2020)Dong and Tong]{jingdong}
Dong, J. and Tong, X.~T.
\newblock Spectral {G}ap of {R}eplica {E}xchange {L}angevin {D}iffusion on {M}ixture {D}istributions.
\newblock \emph{ArXiv 2006.16193v2}, July 2020.

\bibitem[Dong \& Tong(2022)Dong and Tong]{jingdong3}
Dong, J. and Tong, X.~T.
\newblock Spectral {G}ap of {R}eplica {E}xchange {L}angevin {D}iffusion on {M}ixture {D}istributions.
\newblock \emph{Stochastic Processes and their Applications}, 151:\penalty0 451--489, 2022.

\bibitem[Dupuis et~al.(2012)Dupuis, Liu, Plattner, and Doll]{Paul12}
Dupuis, P., Liu, Y., Plattner, N., and Doll, J.~D.
\newblock On the {I}nfinite {S}wapping {L}imit for {P}arallel {T}empering.
\newblock \emph{SIAM J. Multiscale Modeling \& Simulation}, 10, 2012.

\bibitem[Earl \& Deem(2005)Earl and Deem]{parallel_tempering05}
Earl, D.~J. and Deem, M.~W.
\newblock Parallel {T}empering: {T}heory, {A}pplications, and {N}ew {P}erspectives.
\newblock \emph{Phys. Chem. Chem. Phys.}, 7:\penalty0 3910--3916, 2005.

\bibitem[Hindmarsh(1983)]{hindmarsh1983odepack}
Hindmarsh, A.
\newblock {Odepack, A Systemized Collection of Ode Solvers}.
\newblock \emph{Scientific Computing}, 1983.

\bibitem[Hsieh et~al.(2018)Hsieh, Kavis, Rolland, and Cevher]{hsieh2018mirrored}
Hsieh, Y.-P., Kavis, A., Rolland, P., and Cevher, V.
\newblock {Mirrored Langevin Dynamics}.
\newblock \emph{Advances in Neural Information Processing Systems (NeurIPS)}, 31, 2018.

\bibitem[Jin et~al.(2021)Jin, Xu, Shi, Xiao, and Gu]{jin2021mots}
Jin, T., Xu, P., Shi, J., Xiao, X., and Gu, Q.
\newblock {Mots: Minimax Optimal Thompson Sampling}.
\newblock In \emph{International Conference on Machine Learning}, pp.\  5074--5083. PMLR, 2021.

\bibitem[Kang \& Ramanan(2014)Kang and Ramanan]{Kang2014}
Kang, W. and Ramanan, K.
\newblock {Characterization of Stationary Distributions of Reflected Diffusions}.
\newblock \emph{The Annals of Applied Probability}, 24\penalty0 (4):\penalty0 1329 -- 1374, 2014.

\bibitem[Kaptanoglu et~al.(2022)Kaptanoglu, de~Silva, Fasel, Kaheman, Goldschmidt, Callaham, Delahunt, Nicolaou, Champion, Loiseau, Kutz, and Brunton]{Kaptanoglu2022}
Kaptanoglu, A.~A., de~Silva, B.~M., Fasel, U., Kaheman, K., Goldschmidt, A.~J., Callaham, J., Delahunt, C.~B., Nicolaou, Z.~G., Champion, K., Loiseau, J.-C., Kutz, J.~N., and Brunton, S.~L.
\newblock {Pysindy: A Comprehensive Python Package for Robust Sparse System Identification}.
\newblock \emph{Journal of Open Source Software}, 7\penalty0 (69):\penalty0 3994, 2022.

\bibitem[Kook et~al.(2022)Kook, Lee, Shen, and Vempala]{kook2022sampling}
Kook, Y., Lee, Y.-T., Shen, R., and Vempala, S.
\newblock {Sampling with Riemannian Hamiltonian Monte Carlo in a Constrained Space}.
\newblock \emph{Advances in Neural Information Processing Systems (NeurIPS)}, 35:\penalty0 31684--31696, 2022.

\bibitem[Kufner(1985)]{Kufner1985weightedSobolev}
Kufner, A.
\newblock \emph{Weighted Sobolev spaces}.
\newblock Wiley, licensed ed., 1985.

\bibitem[Lamperski(2021)]{Andrew_Lamperski_21_COLT}
Lamperski, A.
\newblock {Projected Stochastic Gradient Langevin Algorithms for Constrained Sampling and Non-Convex Learning}.
\newblock In \emph{Proc. of Conference on Learning Theory (COLT)}, 2021.

\bibitem[Lee et~al.(2018)Lee, Risteski, and Ge]{Holden18}
Lee, H., Risteski, A., and Ge, R.
\newblock Beyond {L}og-concavity: {P}rovable {G}uarantees for {S}ampling {M}ulti-modal {D}istributions using {S}imulated {T}empering {L}angevin {M}onte {C}arlo.
\newblock In \emph{Advances in Neural Information Processing Systems (NeurIPS)}, 2018.

\bibitem[Lee \& Vempala(2018)Lee and Vempala]{lee2018convergence}
Lee, Y.~T. and Vempala, S.~S.
\newblock {Convergence Rate of Riemannian Hamiltonian Monte Carlo and Faster Polytope Volume Computation}.
\newblock In \emph{Proceedings of the 50th Annual ACM SIGACT Symposium on Theory of Computing}, pp.\  1115--1121, 2018.

\bibitem[Lelievre et~al.(2019)Lelievre, Rousset, and Stoltz]{lelievre2019hybrid}
Lelievre, T., Rousset, M., and Stoltz, G.
\newblock {Hybrid Monte Carlo Methods for Sampling Probability Measures on Submanifolds}.
\newblock \emph{Numerische Mathematik}, 143:\penalty0 379--421, 2019.

\bibitem[Leli{\`e}vre et~al.(2023)Leli{\`e}vre, Stoltz, and Zhang]{lelievre2023multiple}
Leli{\`e}vre, T., Stoltz, G., and Zhang, W.
\newblock {Multiple Projection Markov Chain Monte Carlo Algorithms on Submanifolds}.
\newblock \emph{IMA Journal of Numerical Analysis}, 43\penalty0 (2):\penalty0 737--788, 2023.

\bibitem[Li et~al.(2019)Li, Wu, Mackey, and Erdogdu]{Li19}
Li, X., Wu, D., Mackey, L., and Erdogdu, M.~A.
\newblock Stochastic {R}unge-{K}utta {A}ccelerates {L}angevin {M}onte {C}arlo and {B}eyond.
\newblock In \emph{Advances in Neural Information Processing Systems (NeurIPS)}, pp.\  7746--7758, 2019.

\bibitem[Lieberman(1985)]{lieberman1985regularized}
Lieberman, G.
\newblock Regularized distance and its applications.
\newblock \emph{Pacific journal of Mathematics}, 117\penalty0 (2):\penalty0 329--352, 1985.

\bibitem[Lingenheil et~al.(2009)Lingenheil, Denschlag, Mathias, and Tavan]{Martin09}
Lingenheil, M., Denschlag, R., Mathias, G., and Tavan, P.
\newblock Efficiency of {E}xchange {S}chemes in {R}eplica {E}xchange.
\newblock \emph{Chemical Physics Letters}, 478:\penalty0 80--84, 2009.

\bibitem[Liu et~al.(2021)Liu, Tong, and Liu]{liu2021sampling}
Liu, X., Tong, X., and Liu, Q.
\newblock {Sampling with Trusthworthy Constraints: A Variational Gradient Framework}.
\newblock \emph{Advances in Neural Information Processing Systems (NeurIPS)}, 34:\penalty0 23557--23568, 2021.

\bibitem[Mandt et~al.(2017)Mandt, Hoffman, and Blei]{Mandt}
Mandt, S., Hoffman, M.~D., and Blei, D.~M.
\newblock Stochastic {G}radient {D}escent as {A}pproximate {B}ayesian {I}nference.
\newblock \emph{Journal of Machine Learning Research}, 18:\penalty0 1--35, 2017.

\bibitem[Marinari \& Parisi(1992)Marinari and Parisi]{ST}
Marinari, E. and Parisi, G.
\newblock Simulated {T}empering: {A} {N}ew {M}onte {C}arlo {S}cheme.
\newblock \emph{Europhysics Letters ({EPL})}, 19\penalty0 (6):\penalty0 451--458, 1992.

\bibitem[Mazumdar et~al.(2020)Mazumdar, Pacchiano, Ma, Jordan, and Bartlett]{mazumdar2020approximate}
Mazumdar, E., Pacchiano, A., Ma, Y., Jordan, M., and Bartlett, P.
\newblock {On Approximate Thompson Sampling with Langevin Algorithms}.
\newblock In \emph{International Conference on Machine Learning}, pp.\  6797--6807. PMLR, 2020.

\bibitem[Menaldi \& Robin(1985)Menaldi and Robin]{Menaldi1985}
Menaldi, J.-L. and Robin, M.
\newblock {Reflected Diffusion Processes with Jumps}.
\newblock \emph{The Annals of Probability}, pp.\  319--341, 1985.

\bibitem[Neal(2012)]{Neal2012}
Neal, R.~M.
\newblock {MCMC} using {Hamiltonian} {D}ynamics.
\newblock In \emph{Handbook of Markov Chain {M}onte {C}arlo}, volume~54, pp.\  113--162. Chapman and Hall/CRC, 2012.

\bibitem[Noble et~al.(2023)Noble, De~Bortoli, and Durmus]{noble2023unbiased}
Noble, M., De~Bortoli, V., and Durmus, A.
\newblock {Unbiased Constrained Sampling with Self-Concordant Barrier Hamiltonian Monte Carlo}.
\newblock In \emph{Advances in Neural Information Processing Systems (NeurIPS)}, 2023.

\bibitem[Okabe et~al.(2001)Okabe, Kawata, Okamoto, and Mikami]{DEO}
Okabe, T., Kawata, M., Okamoto, Y., and Mikami, M.
\newblock Replica {E}xchange {M}onte {C}arlo {M}ethod for the {I}sobaric–isothermal {E}nsemble.
\newblock \emph{Chemical Physics Letters}, 335:\penalty0 435–439, 2001.

\bibitem[Petzold(1983)]{linda_automatic}
Petzold, L.
\newblock {Automatic Selection of Methods for Solving Stiff and Nonstiff Systems of Ordinary Differential Equations}.
\newblock \emph{SIAM Journal on Scientific and Statistical Computing}, 4\penalty0 (1):\penalty0 136--148, 1983.
\newblock \doi{10.1137/0904010}.

\bibitem[Qi et~al.(2018)Qi, Wei, Ma, and Nussinov]{qi2018replica}
Qi, R., Wei, G., Ma, B., and Nussinov, R.
\newblock {Replica Exchange Molecular Dynamics: A Practical Application Protocol with Solutions to Common Problems and a Peptide Aggregation and Self-Assembly Example}.
\newblock \emph{{Peptide Self-Assembly: Methods and Protocols}}, pp.\  101--119, 2018.

\bibitem[Sindhikara et~al.(2010)Sindhikara, Emerson, and Roitberg]{sindhikara2010exchange}
Sindhikara, D.~J., Emerson, D.~J., and Roitberg, A.~E.
\newblock {Exchange Often and Properly in Replica Exchange Molecular Dynamics}.
\newblock \emph{Journal of Chemical Theory and Computation}, 6\penalty0 (9):\penalty0 2804--2808, 2010.

\bibitem[Swendsen \& Wang(1986)Swendsen and Wang]{PhysRevLett86}
Swendsen, R.~H. and Wang, J.-S.
\newblock Replica {M}onte {C}arlo {S}imulation of {S}pin-{G}lasses.
\newblock \emph{Physical Review Letters}, 57:\penalty0 2607--2609, 1986.

\bibitem[Syed et~al.(2022)Syed, Bouchard-C{\^o}t{\'e}, Deligiannidis, and Doucet]{syed2022non}
Syed, S., Bouchard-C{\^o}t{\'e}, A., Deligiannidis, G., and Doucet, A.
\newblock {Non-reversible Parallel Tempering: A Scalable Highly Parallel Mcmc Scheme}.
\newblock \emph{Journal of the Royal Statistical Society Series B: Statistical Methodology}, 84\penalty0 (2):\penalty0 321--350, 2022.

\bibitem[Tanaka(1979)]{tanaka1979}
Tanaka, H.
\newblock {Stochastic Differential Equations with Reflecting}.
\newblock \emph{Stochastic Processes: Selected Papers of Hiroshi Tanaka}, 9:\penalty0 157, 1979.

\bibitem[Wang \& Landau(2001)Wang and Landau]{wang2001efficient}
Wang, F. and Landau, D.~P.
\newblock {Efficient, Multiple-Range Random Walk Algorithm to Calculate the Density of States}.
\newblock \emph{{Physical Review Letters}}, 86\penalty0 (10):\penalty0 2050, 2001.

\bibitem[Wang(2007)]{wang2007estimates}
Wang, F.-Y.
\newblock {Estimates of the First Neumann Eigenvalue and the Log-Sobolev Constant on Non-convex Manifolds}.
\newblock \emph{{Mathematische Nachrichten}}, 280\penalty0 (12):\penalty0 1431--1439, 2007.

\bibitem[Wang(2014)]{wang2014analysis2}
Wang, F.-Y.
\newblock \emph{{Analysis for Diffusion Processes on Riemannian Manifolds}}, volume~18.
\newblock World Scientific, 2014.

\bibitem[Wang \& Wibisono(2022)Wang and Wibisono]{wang2022accelerating}
Wang, J.-K. and Wibisono, A.
\newblock {Accelerating Hamiltonian Monte Carlo via Chebyshev Integration Time}.
\newblock In \emph{Proc. of the International Conference on Learning Representation (ICLR)}, 2022.

\bibitem[Wang et~al.(2020)Wang, Lei, and Panageas]{wang2020fast}
Wang, X., Lei, Q., and Panageas, I.
\newblock {Fast Convergence of Langevin Dynamics on Manifold: Geodesics Meet Log-Sobolev}.
\newblock \emph{Advances in Neural Information Processing Systems (NeurIPS)}, 33:\penalty0 18894--18904, 2020.

\bibitem[Welling \& Teh(2011)Welling and Teh]{Welling11}
Welling, M. and Teh, Y.~W.
\newblock {B}ayesian {L}earning via {S}tochastic {G}radient {L}angevin {D}ynamics.
\newblock In \emph{Proc. of the International Conference on Machine Learning (ICML)}, pp.\  681--688, 2011.

\bibitem[Zappa et~al.(2018)Zappa, Holmes-Cerfon, and Goodman]{zappa2018monte}
Zappa, E., Holmes-Cerfon, M., and Goodman, J.
\newblock {Monte Carlo on Manifolds: Sampling Densities and Integrating Functions}.
\newblock \emph{{Communications on Pure and Applied Mathematics}}, 71\penalty0 (12):\penalty0 2609--2647, 2018.

\bibitem[Zhang et~al.(2020)Zhang, Li, Zhang, Chen, and Wilson]{ruqi2020}
Zhang, R., Li, C., Zhang, J., Chen, C., and Wilson, A.~G.
\newblock Cyclical {S}tochastic {G}radient {MCMC} for {B}ayesian {D}eep {L}earning.
\newblock In \emph{Proc. of the International Conference on Learning Representation (ICLR)}, 2020.

\bibitem[Zhang et~al.(2022)Zhang, Liu, and Tong]{zhang2022sampling}
Zhang, R., Liu, Q., and Tong, X.
\newblock {Sampling in Constrained Domains with Orthogonal-Space Variational Gradient Descent}.
\newblock \emph{Advances in Neural Information Processing Systems (NeurIPS)}, 35:\penalty0 37108--37120, 2022.

\bibitem[Zheng et~al.(2024)Zheng, Deng, Moya, and Lin]{zheng2024accelerating}
Zheng, H., Deng, W., Moya, C., and Lin, G.
\newblock {Accelerating Approximate Thompson Sampling with Underdamped Langevin Monte Carlo}.
\newblock In \emph{Proceedings of the International Workshop on Artificial Intelligence and Statistics}, 2024.

\bibitem[Zheng et~al.(2007)Zheng, Andrec, Gallicchio, and Levy]{zheng2007simulating}
Zheng, W., Andrec, M., Gallicchio, E., and Levy, R.~M.
\newblock {Simulating Replica Exchange Simulations of Protein Folding with a Kinetic Network Model}.
\newblock \emph{{Proceedings of the National Academy of Sciences}}, 104\penalty0 (39):\penalty0 15340--15345, 2007.

\bibitem[Zou \& Gu(2021)Zou and Gu]{zou2021convergence}
Zou, D. and Gu, Q.
\newblock {On the Convergence of Hamiltonian Monte Carlo with Stochastic Gradients}.
\newblock In \emph{Proc. of the International Conference on Machine Learning (ICML)}, pp.\  13012--13022. PMLR, 2021.

\end{thebibliography}
\bibliographystyle{icml2024}

\newpage
\appendix
\onecolumn

\icmltitle{Supplementary Materials}

\section{Experimental Details}

This section supplements the main text by detailing additional experimental procedures. Our initial focus is on the Lorenz system identification task, including a comparative analysis with four baseline methods. Subsequently, we explore the Lotka-Volterra system identification through constrained sampling, which compares both reflection-based methods and baselines without reflection operations. Furthermore, the implementation specifics of multi-mode sampling within bounded domains are discussed. This includes an examination of hyper-parameter selections and the application of these methods across various domain settings to assess the robustness of our proposed algorithm. An additional experiment is included to explore the relationship between the diameters of constrained domains and the mixing rates. Lastly, we outline hyper-parameters essential for executing large-scale image classification tasks.

\subsection{The Lorenz System}\label{sec:lorenz_sup}

For the experiments identifying the Lorenz system, we incorporate baselines such as SGLD, R-SGLD, cycSGLD, and r-cycSGLD. Across all tasks, including the main text (Section~\ref{section:exp_lorenz}) and baselines, we conduct 60,000 iterations with the first 20,000 iterations as burn-in samples. For SGLD and R-SGLD, the learning rate commences at 5e-6, decaying at a rate of 0.9999 per iteration after the first 10,000 iterations. For cycSGLD and R-cycSGLD, the initial learning rate is set at 2e-5, utilizing a 10-cycle cosine learning rate schedule. The learning rates for reSGLD and r2SGLD, set at 2e-5 and 2e-6 for the two chains, decaying post 10,000 iterations at a rate of 0.9999. We adjust correction terms to ensure a swap rate of 5\%-20\%.

\begin{figure*}[htbp]
\centering
\subfigure[SGLD]{
\centering
\includegraphics[width=.22\columnwidth]{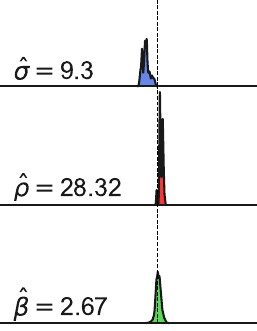}\label{fig:lorenz_posterior_sgld}
}\hspace{-0mm}
\subfigure[R-SGLD]{
\centering
\includegraphics[width=.22\columnwidth]{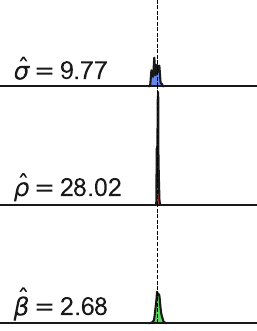}\label{fig:lorenz_posterior_rsgld}
}\hspace{-0mm}
\subfigure[CycSGLD]{
\centering
\includegraphics[width=.22\columnwidth]{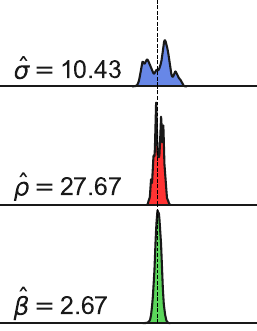}\label{fig:lorenz_posterior_cycsgld}
}
\subfigure[R-cycSGLD]{
\centering
\includegraphics[width=.22\columnwidth]{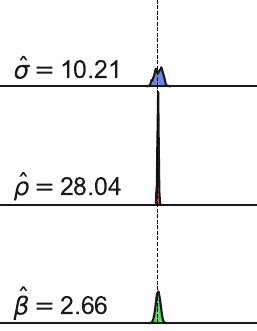}\label{fig:lorenz_posterior_rcycsgld}
}
\caption{Posterior distributions of the identified model parameters.}
\label{fig:lorenz_posterior_appendix}
\end{figure*}

\begin{figure*}[htbp]
\centering

\subfigure[SGLD]{
\centering
\includegraphics[width=.22\columnwidth]{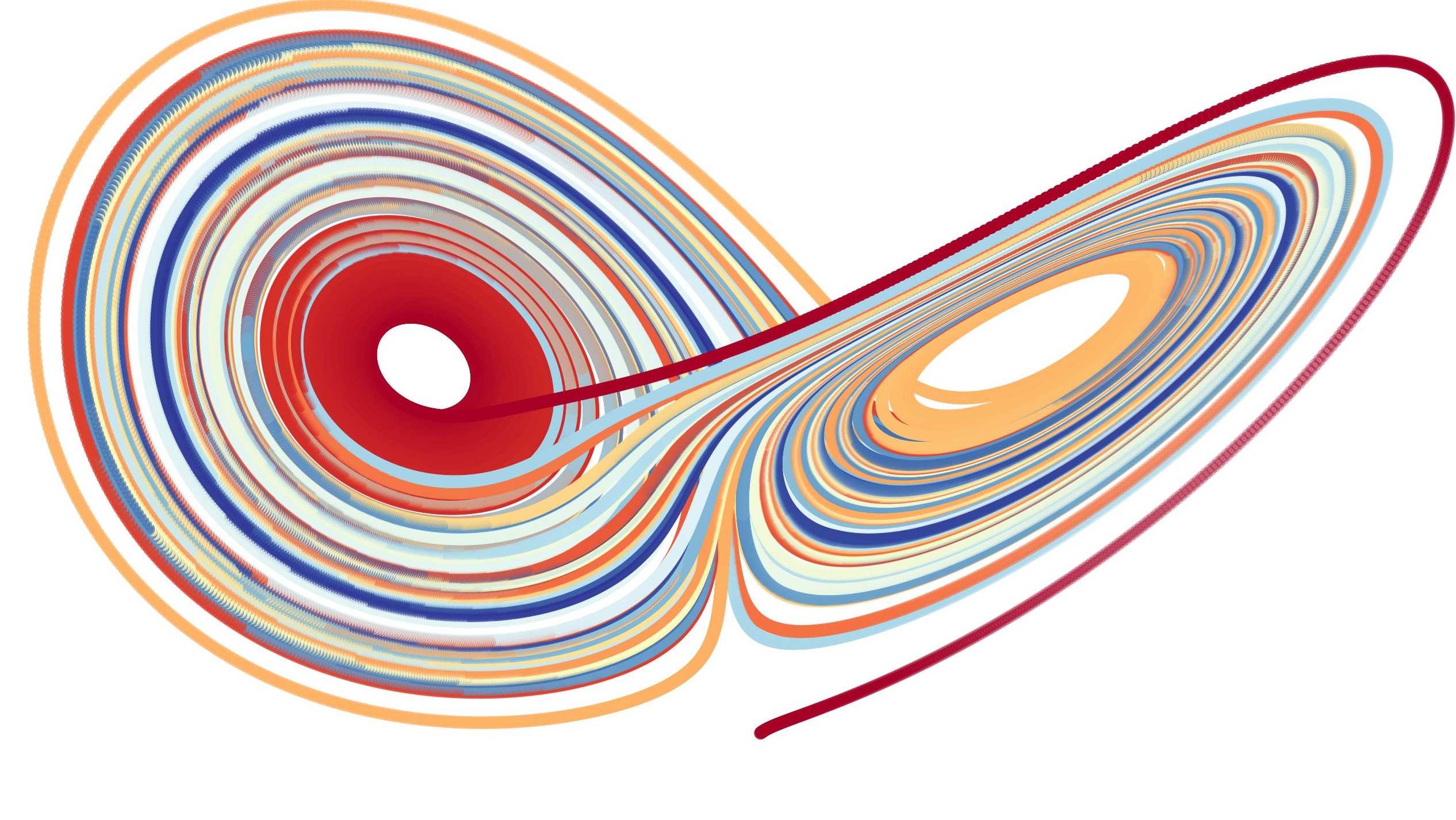}\label{fig:lorenz_sim_sgld}
}\hspace{-0mm}
\subfigure[R-SGLD]{
\centering
\includegraphics[width=.22\columnwidth]{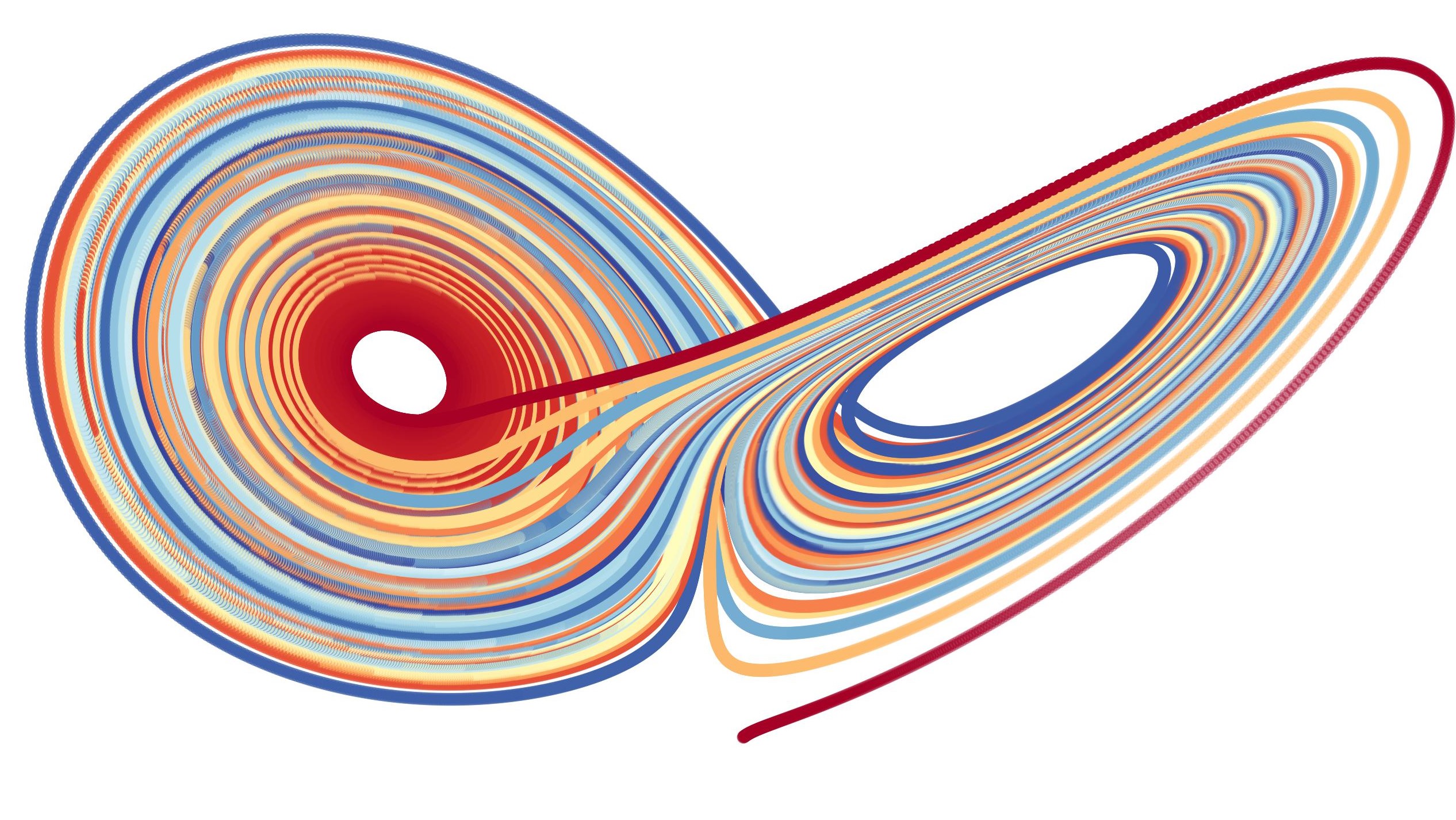}\label{fig:lorenz_sim_rsgld}
}\hspace{-0mm}
\subfigure[CycSGLD]{
\centering
\includegraphics[width=.22\columnwidth]{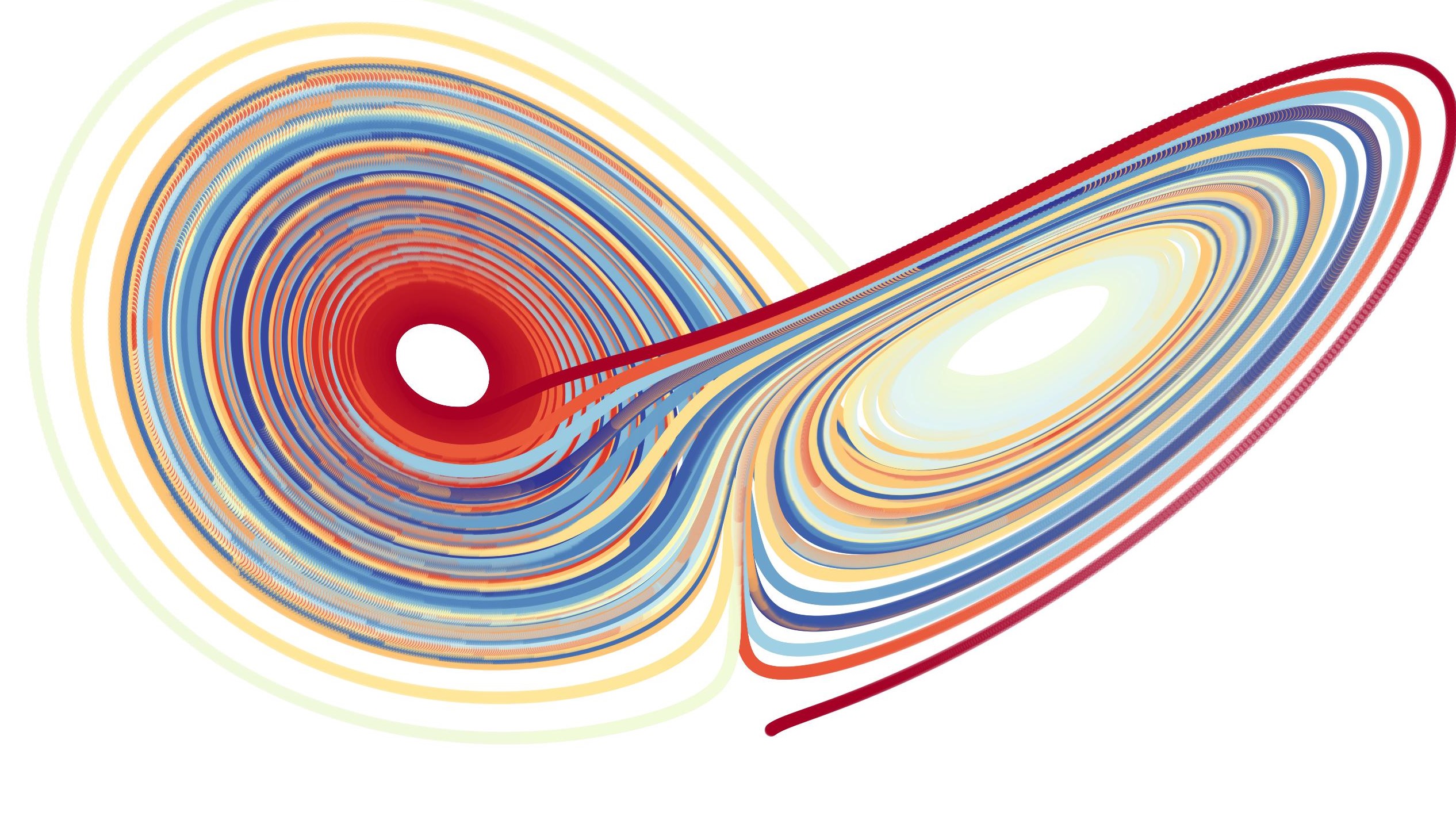}\label{fig:lorenz_sim_cycsgld}
}
\subfigure[R-cycSGLD]{
\centering
\includegraphics[width=.22\columnwidth]{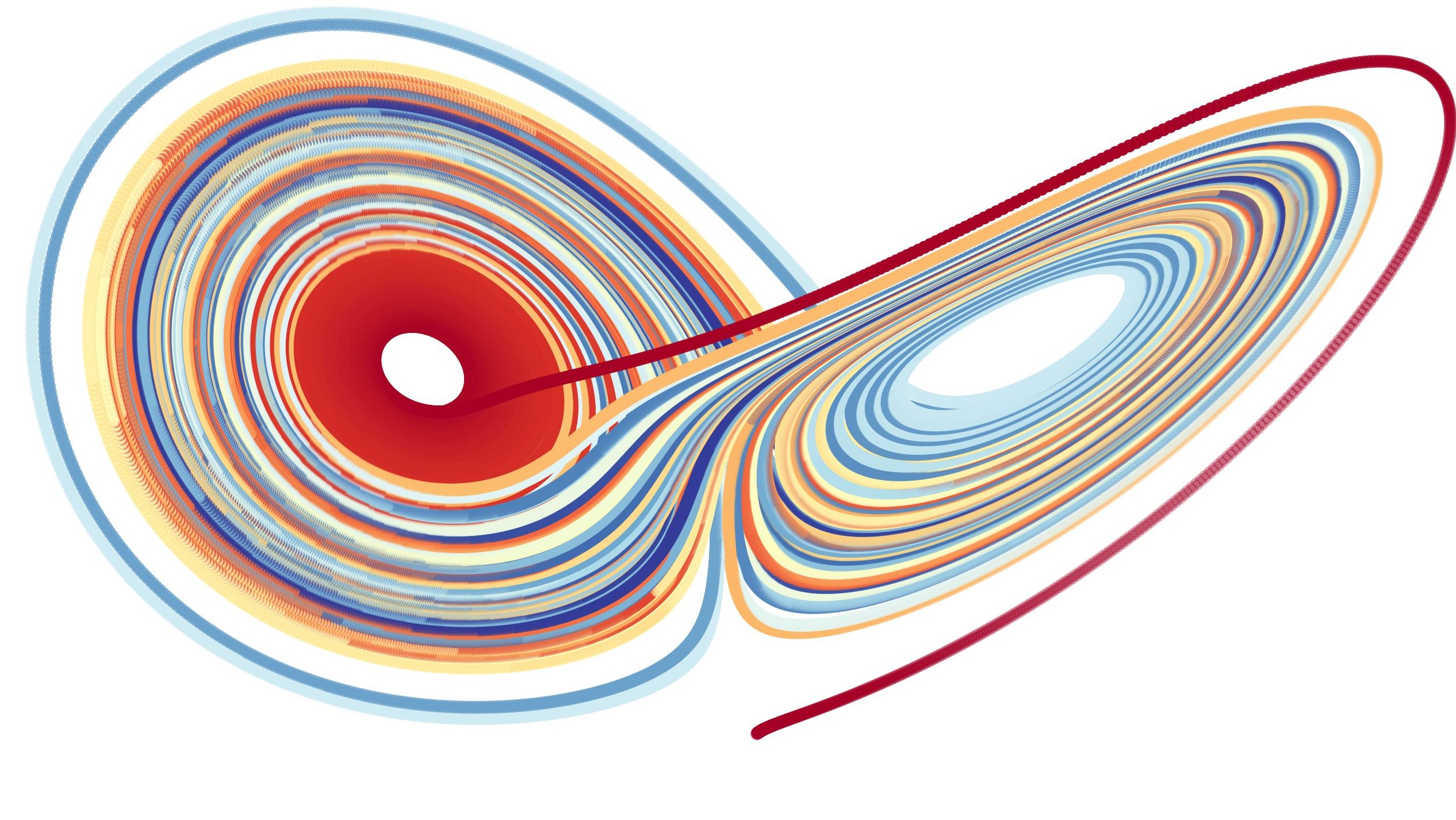}\label{fig:lorenz_sim_rcycsgld}
}
\caption{Simulation results of the identified Lorenz system.}
\label{fig:lorenz_sim_appendix}
\end{figure*}

As depicted in Figures~\ref{fig:lorenz_posterior} and~\ref{fig:lorenz_posterior_appendix}, sampling with physical constraints and reflection significantly improves the model performance, in terms of the posterior modes and the empirical posteriors. In a comparative analysis of three reflection-based sampling algorithms, R-SGLD falls behind in performance, while R-cycSGLD better concentrates the parameter posterior near the true parameters. Our proposed algorithm outperforms the others in accuracy. Figure~\ref{fig:lorenz_sim_appendix} presents the simulation results in the use of the learned parameter according to the baselines. Although all simulations capture the chaotic nature of the Lorenz system, r2SGLD (Figure~\ref{fig:lorenz_sim_r2SGLD}) most closely replicates the true system dynamics (Figure~\ref{fig:lorenz_sim_truth}).

\subsection{Lotka-Volterra Model}\label{sec:lotka_volterra_sup}

The Lotka-Volterra model (the predator-prey model) is a pair of first-order, non-linear, differential equations frequently used to describe the dynamics of biological systems in which two species interact, one as a predator and the other as prey. The equations are as follows:

\begin{equation}
\left\{
\begin{aligned}
    \dot x & = \alpha x - \beta xy,\\
    \dot y & = -\delta y + \gamma xy.
\end{aligned}
\right.
\end{equation}

Here, $x$ and $y$ represent the number of prey and predators, respectively, and $\dot x$ and $\dot y$ represent the corresponding time derivatives. The model is characterized by four key parameters: \textbf{Prey Growth Rate} $\alpha$ represents the rate of growth of the prey population in the absence of predators. It is a positive value, indicating that the prey population grows exponentially when unchallenged. \textbf{Predation Rate Coefficient} $\beta$ governs the rate at which predators destroy prey\footnote{Following traditional conventions in the Lotka-Volterra model, we use $\beta$ to differentiate it from the parameters $\bbeta$ as defined in \eqref{eq:sde_2couple}.}. It modulates the interaction between the prey and predators, which signifies the efficiency of predators in consuming prey. \textbf{Predator Mortality Rate} $\delta$ represents the rate of decline of the predator population in the absence of prey. It signifies the mortality or decay rate of predators when there is no food (prey) available. \textbf{Predator Reproduction Rate} $\gamma$ is tied to the rate at which the predator population increases due to the consumption of prey. It indicates that the growth of the predator population is dependent on the availability of prey.

The Lotka-Volterra model is a simplistic representation and makes several assumptions, such as unlimited food supply for prey and constant rates of predation, which may not always hold true in real ecosystems. However, it provides a foundational framework for understanding the dynamic interactions between predator and prey populations.

In this study, we employ the LSODA solver~\cite{hindmarsh1983odepack, linda_automatic} to generate data $x(t)$ and $y(t)$, where $t$ ranges from 0 to 100 in increments of 0.01. The framework of identifying the dynamical system aligns with the one established in Section~\ref{section:exp_lorenz} on Lorenz system identification, including the sample generation and the simulation of empirical posterior distributions. For the ease of implementation, the corresponding matrices for gradient computations are presented as follows:

\begin{equation}    
\dot{\mathbf{ \mathbf{X}}} = \left[ {\begin{array}{*{20}{c}}
{\dot x({t_1})}&{\dot y({t_1})}\\
{\dot x({t_2})}&{\dot y({t_2})}\\
 \vdots & \vdots \\
{\dot x({t_m})}&{\dot y({t_m})}\\
 \vdots & \vdots 
\end{array}} \right], \ \ 
\hat{\mathbf{\bbeta}} = \left[ {\begin{array}{*{20}{c}}
\alpha &0\\
0&{ - \delta }\\
{ - {\beta} }&\gamma 
\end{array}} \right],
\ \ 
\mathbf{\Theta} (\mathbf{X}) = \left[ {\begin{array}{*{20}{c}}
{x({t_1})}&{y({t_1})}&{x({t_1})y({t_1})}\\
{x({t_2})}&{y({t_2})}&{x({t_2})y({t_2})}\\
 \vdots & \vdots & \vdots \\
{x({t_m})}&{y({t_m})}&{x({t_m})y({t_m})}\\
 \vdots & \vdots & \vdots 
\end{array}} \right],
\end{equation}
where we establish parameter values as follows: $\alpha = 1.0$, $\beta = 0.10$, $\delta = 1.50$, and $\gamma = 0.075$. The chosen batch size is 4,096, and the sampling process involves 60,000 iterations, with the initial 20,000 serving as burn-in samples. For SGLD and R-SGLD, the settings are the same as the ones in identifying the Lorenz system. In the case of cycSGLD and R-cycSGLD, the initial learning rate is 1e-5, implemented with a ten-cycle cosine learning rate schedule. For reSGLD and r2SGLD, learning rates are 1e-5 and 2e-6 respectively for the two chains, following a similar decay pattern post 10,000 iterations at 0.9999.

\begin{figure*}[htbp]
\centering

\subfigure[SGLD]{
\centering
\includegraphics[width=.15\columnwidth]{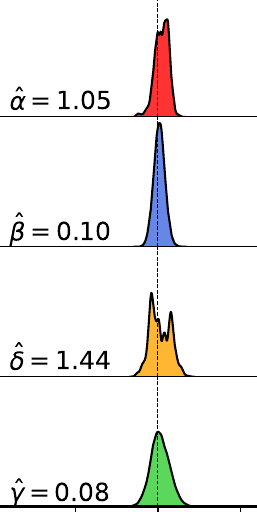}\label{fig:lotka_sgld}
}\hspace{-0mm}
\subfigure[R-SGLD]{
\centering
\includegraphics[width=.15\columnwidth]{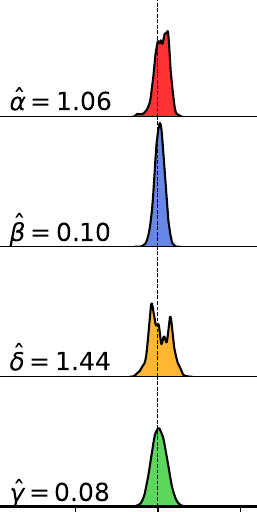}\label{fig:lotka_sgld_reflect}
}\hspace{-0mm}
\subfigure[CycSGLD]{
\centering
\includegraphics[width=.15\columnwidth]{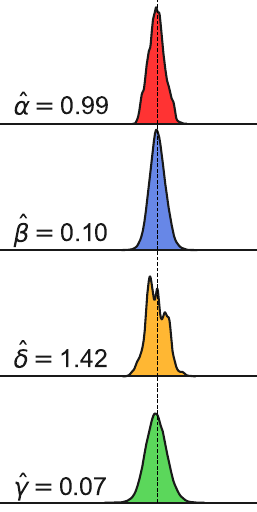}\label{fig:lotka_cycsgld}
}
\subfigure[R-cycSGLD]{
\centering
\includegraphics[width=.15\columnwidth]{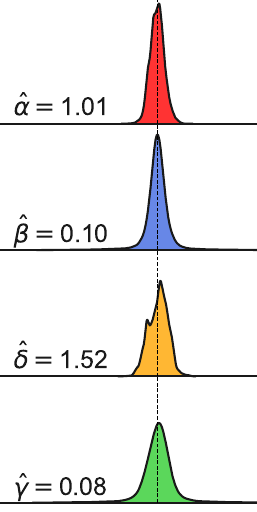}\label{fig:lotka_cycsgld_reflect}
}
\subfigure[reSGLD]{
\centering
\includegraphics[width=.15\columnwidth]{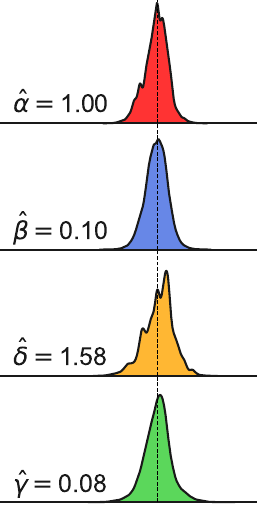}\label{fig:lotka_resgld}
}
\subfigure[r2SGLD]{
\centering
\includegraphics[width=.15\columnwidth]{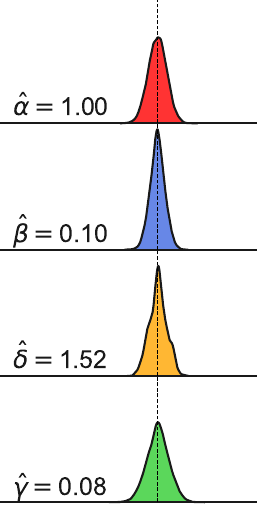}\label{fig:lotka_resgld_reflect}
}
\caption{Posterior distribution of model parameters in the Lotka-Volterra system.}
\label{fig:lv_posterior_appendix}
\end{figure*}

In Figure~\ref{fig:lv_posterior_appendix}, we observe the empirical posteriors of model parameters and their posterior modes. These results further confirm the beneficial impact of the reflection operation on parameter estimation. A comparative analysis of different algorithms for constrained sampling tasks reveals that R-SGLD (the blue lines) is observed to be the least effective, whereas R-cycSLGD (the green lines) and r2SGLD (the red lines) demonstrate comparable performance. In Figure~\ref{fig:lv_sim_appendix}, the parameters derived from these algorithms are employed to simulate the dynamics of the Lotka-Volterra system over a period from the 100th to the 200th day, as the first 100 days produced overlapping results. The simulations reveal that both R-cycSLGD and r2SGLD provide comparable outcomes, yet r2SGLD shows a marginally closer approximation to the true dynamics (the black lines).

\begin{figure*}[htbp]
\centering

\subfigure[Prey simulations.]{
\centering
\includegraphics[width=.9\columnwidth]{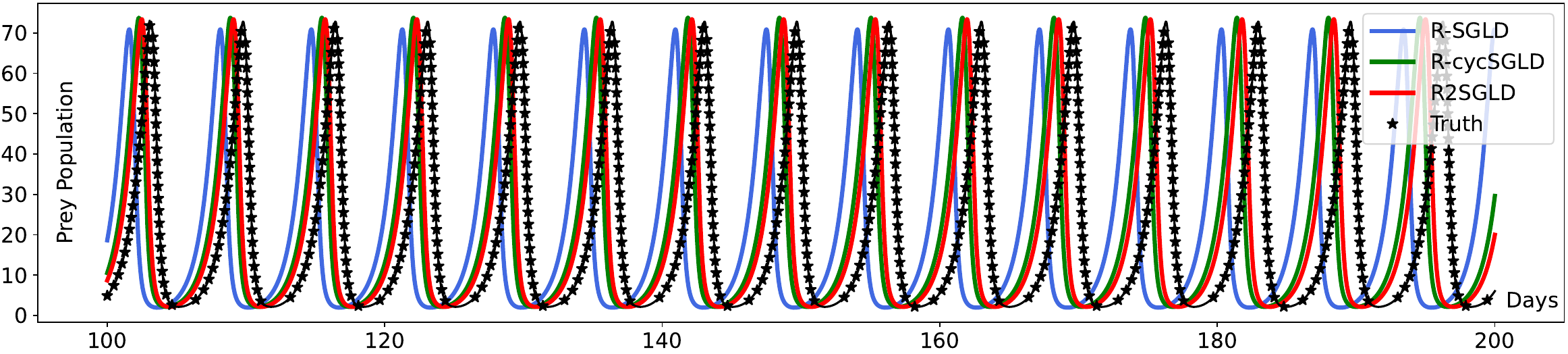}\label{fig:lv_sim1}
}
\subfigure[Predator simulations.]{
\centering
\includegraphics[width=.9\columnwidth]{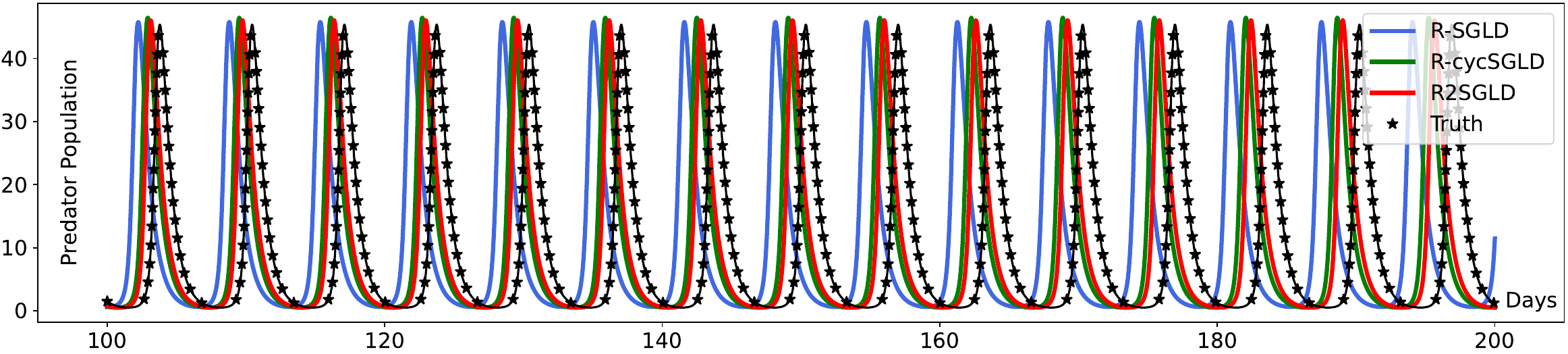}\label{fig:lv_sim2}
}
\caption{Simulation of the identified Lotka-Volterra system from 100 to 200 days.}
\label{fig:lv_sim_appendix}
\end{figure*}

\clearpage
\subsection{Constrained Sampling from Multi-mode Distributions}\label{sec:multimode_sup}

This section focuses on constrained sampling from a multi-mode distribution. 
\begin{figure*}[htbp]
\vskip 0.2in \centering
\subfigure[Truth]{
\centering
\includegraphics[width=.22\columnwidth]{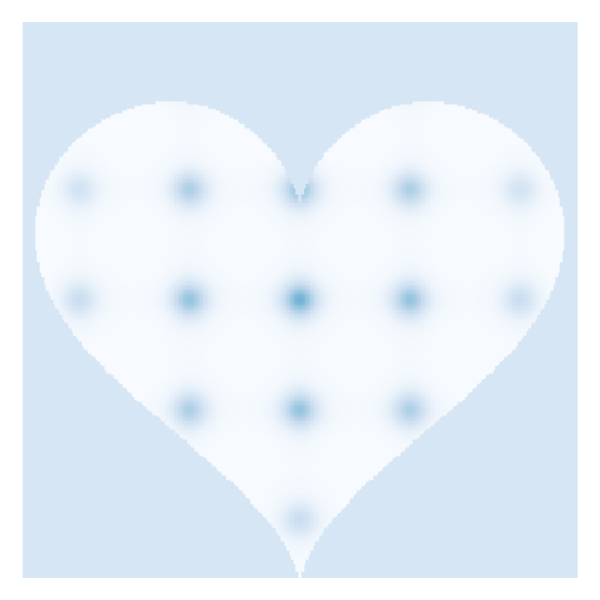}\label{fig:heart_truth}
}
\subfigure[r2SGLD]{
\centering
\includegraphics[width=.22\columnwidth]{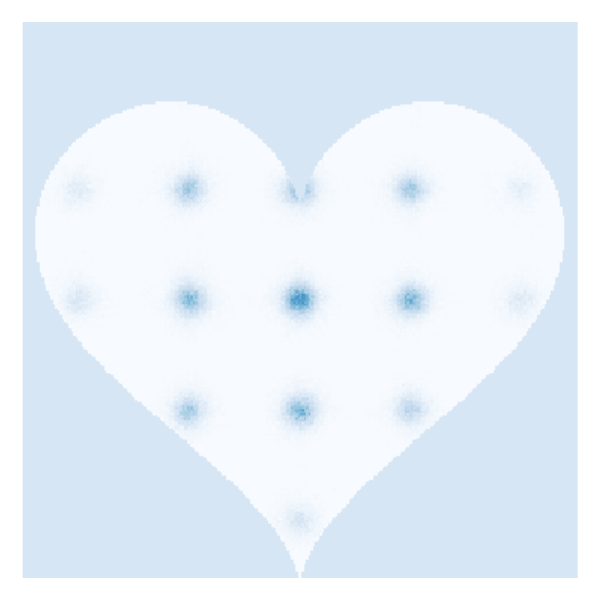}
}
\subfigure[R-cycSGLD]{
\centering
\includegraphics[width=.22\columnwidth]{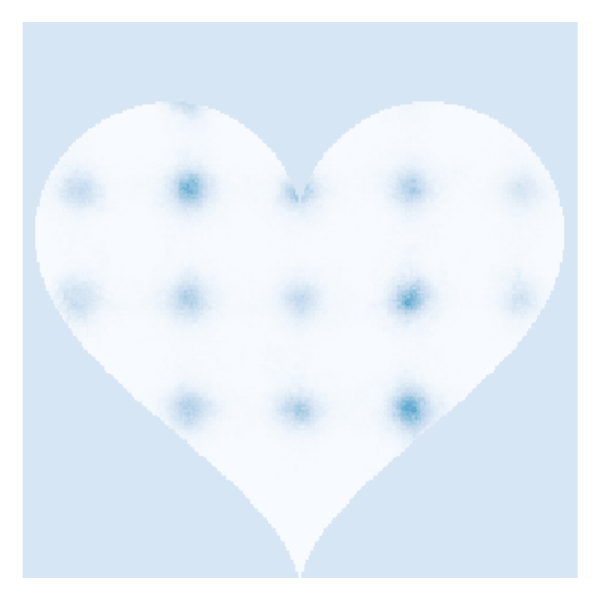}
}
\subfigure[R-SGLD]{
\centering
\includegraphics[width=.22\columnwidth]{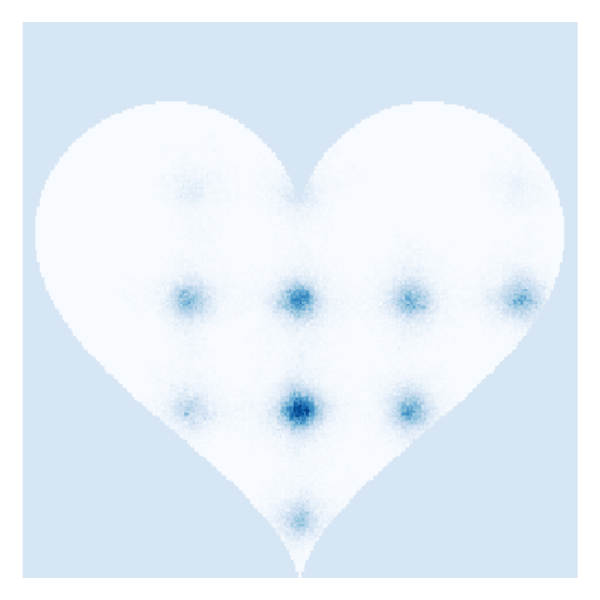}
}

\subfigure[Truth]{
\centering
\includegraphics[width=.22\columnwidth]{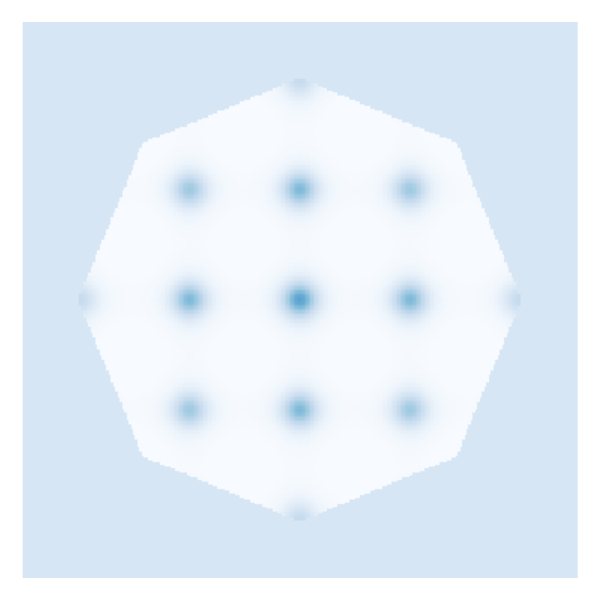}\label{fig:polygon_truth}
}
\subfigure[r2SGLD]{
\centering
\includegraphics[width=.22\columnwidth]{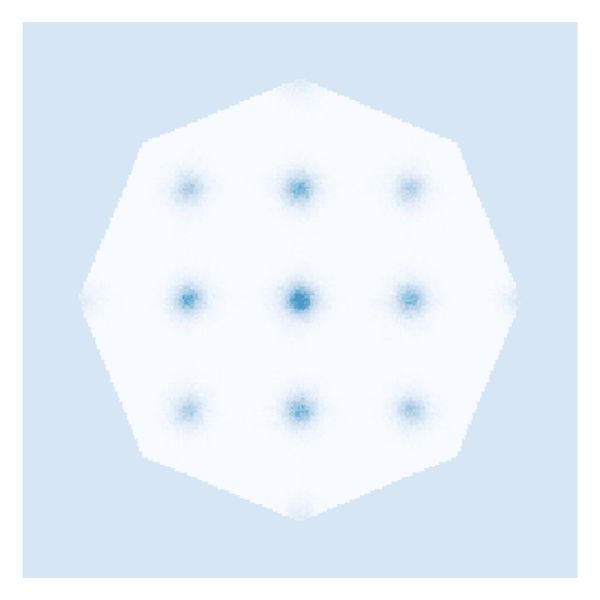}
}
\subfigure[R-cycSGLD]{
\centering
\includegraphics[width=.22\columnwidth]{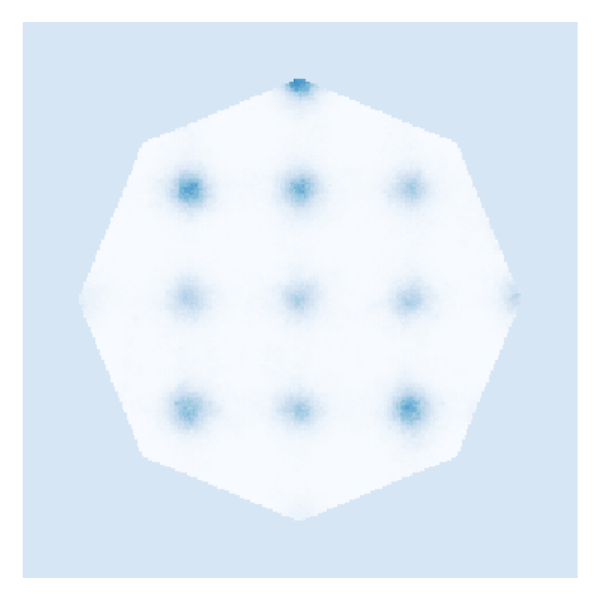}
}
\subfigure[R-SGLD]{
\centering
\includegraphics[width=.22\columnwidth]{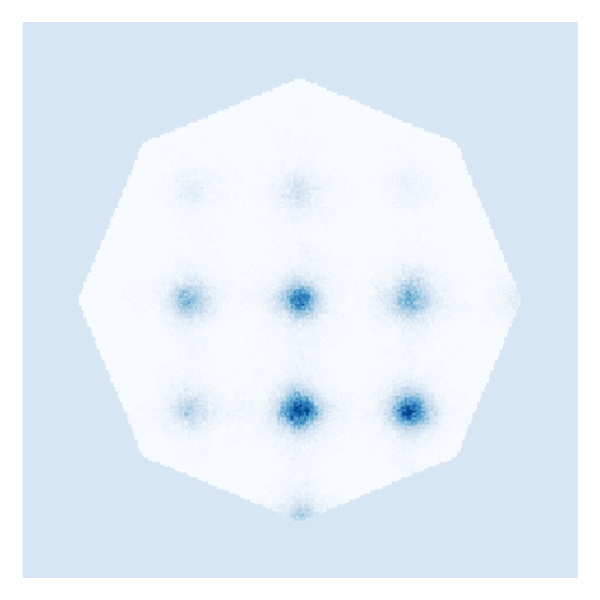}
}

\subfigure[Truth]{
\centering
\includegraphics[width=.22\columnwidth]{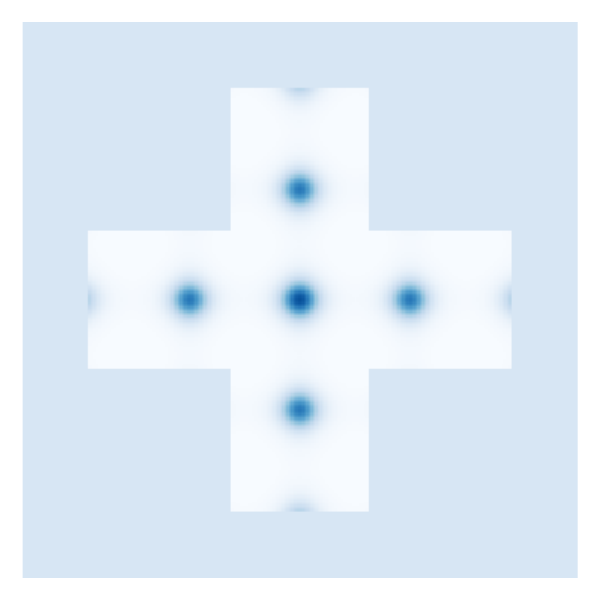}\label{fig:cross_truth}
}
\subfigure[r2SGLD]{
\centering
\includegraphics[width=.22\columnwidth]{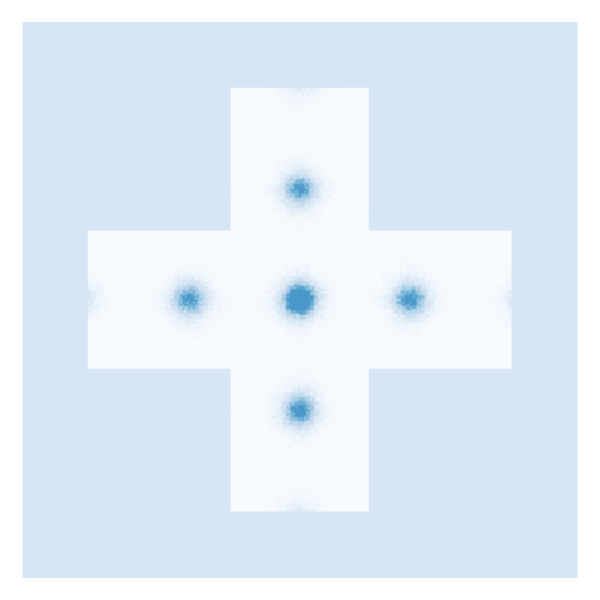}
}
\subfigure[R-cycSGLD]{
\centering
\includegraphics[width=.22\columnwidth]{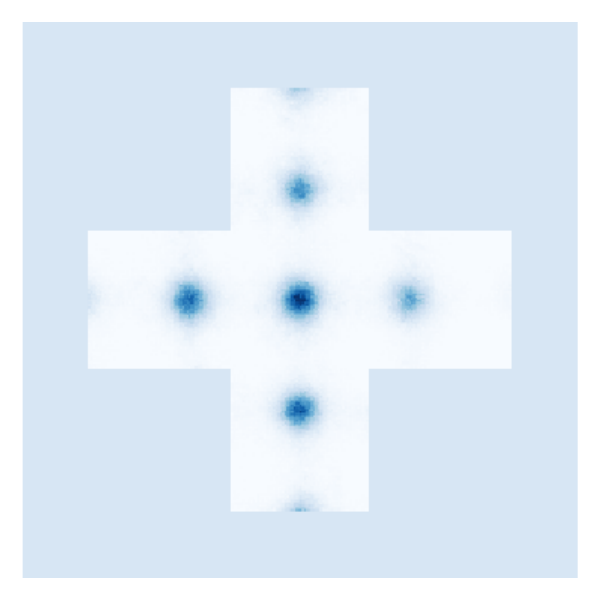}
}
\subfigure[R-SGLD]{
\centering
\includegraphics[width=.22\columnwidth]{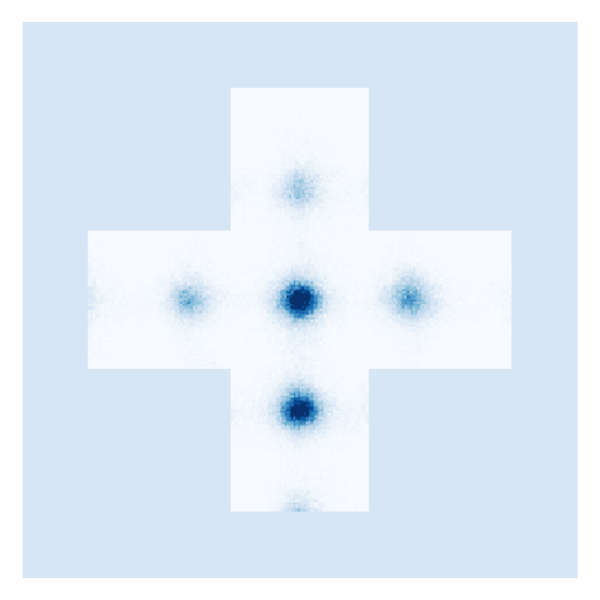}
}

\caption{Empirical behavior on multi-mode distributions with diverse boundaries.}
\label{fig:sim_multimode2}
\end{figure*}
We study a 2D mixture of 25 Gaussians within uniquely shaped domains: flower (Figure~\ref{fig:sim_multimode}), heart (Figure~\ref{fig:heart_truth}), 
polygon (Figure~\ref{fig:polygon_truth}), and cross (Figure~\ref{fig:cross_truth}). The heart domain in this simulation is defined by the following parametric equations:
\begin{equation*}
    \begin{aligned}
        x & = 16 \sin^3(2 \pi t),\\
        y & = 13 \cos(2 \pi t) - 5 \cos(4 \pi t) - 2 \cos(6 \pi t) - \cos(8 \pi t).
    \end{aligned}
\end{equation*}
Subsequently, the octagon domain is described by the equations:
\begin{equation*}
    \begin{aligned}
        r & = ct - \lfloor ct \rfloor,\\
        x & = (1-r) \cdot X_{\lfloor ct \rfloor} + r \cdot X_{\lfloor ct \rfloor + 1},\\
        y & = (1-r) \cdot Y_{\lfloor ct \rfloor} + r \cdot Y_{\lfloor ct \rfloor + 1},
    \end{aligned}
\end{equation*}
where $(X_i, Y_i)_{i=0}^C$ are the edge sequences with $(X_C, Y_C) = (X_0, Y_0)$ and $C=8$.

The probability densities are set to zero outside these shapes. For this study, we employ R-SGLD (initial rate: 5e-4), R-cycSGLD (initial rate: 1e-3 with a 5-cycle cosine schedule), and r2SGLD (two chains at 5e-4 and 1.5e-3, aiming for a 5-20\% swap rate). A total of 50,000 samples are generated, with the initial 10,000 as burn-in. The obtained empirical posteriors are shown in Figures~\ref{fig:sim_multimode} and \ref{fig:sim_multimode2}.

Upon examining the empirical posteriors, we note the limitations of R-SGLD, where the results often get trapped in local modes and fail to fully capture the target distributions. Conversely, R-cycSGLD demonstrates a more robust ability to capture most modes, albeit with a tendency to overstay in specific modes and underrepresent others, which more samples could improve. In comparison, our proposed algorithm shows significant improvements in accurately characterizing all modes, which provides the most precise empirical behavior in multi-mode distributions across different constrained domains.

\paragraph{Exploration of the Constrained Domain Diameter}

To empirically validate our theoretical analysis, we conduct additional experiments with r2SGLD, which focuses on how the diameters of the constrained domain affect mixing rates. The target distributions are selected as convex bounded domains (octagons) against a non-convex distribution of 25 Gaussian modes (refer to Figure \ref{fig:polygon_truth}). The domain diameters are set between 1.5 and 3.0 (the target distribution is $5.0\times 5.0$). We employ a dual-chain r2SGLD with learning rates of 2e-5 and 2e-4. For each diameter setting, we execute ten runs, where each generates 10,000 samples. The expected KL divergences and their 95\% confidence intervals are depicted in Figure \ref{fig:kl_diameter} for purposes of analysis. 

\begin{figure}[htbp]
\centering
\includegraphics[width=.45\columnwidth]{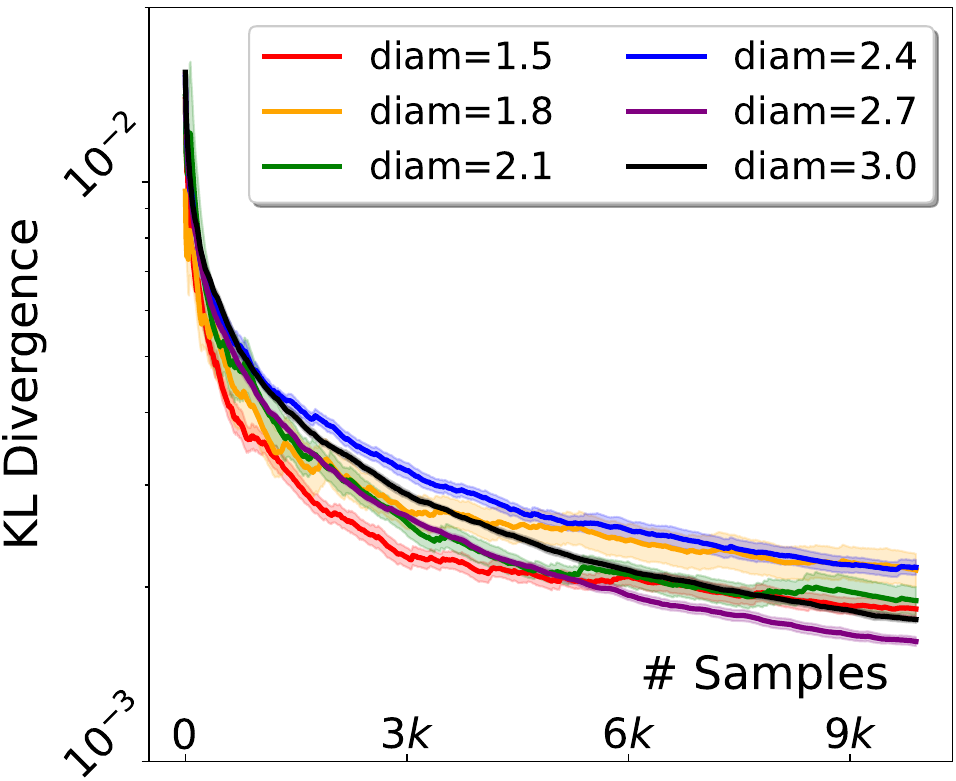}
\caption{KL divergence with respect to different domain diameter setts.}
\label{fig:kl_diameter}
\end{figure}

Illustrated by the corresponding figure, a discernible correlation emerges between the domain diameter and the convergence rate of the algorithm. Notably, the result demonstrates a trend where a reduction in domain diameter results in a more rapid decline in KL divergence. This further leads to an advanced mixing rate. This observation suggests that a smaller domain diameter facilitates the algorithm's expedited convergence to the target distribution, which also implies that constraining the parameter space informed by prior knowledge can markedly enhance the algorithm's efficiency in its convergence.

\subsection{Image Classification}\label{sec:classification_sup}

The following part is dedicated to an empirical evaluation of algorithm settings. Our objective is to the choices of hyper-parameters and demonstrate their effectiveness in the tests.

\paragraph{Learning Rate}\label{subsec:learn_rate}
This study acknowledges that while reflection operations in sampling methods can lead to more comprehensive exploration given unlimited computational resources, our experiments are conducted within the confines of 1,000 epochs. A critical component of these experiments is the determination of suitable initial learning rates to enhance model performance. 

In our study, we conduct a systematic exploration of learning rates for the single-chain R-SGHMC algorithm. We initiate this exploration by setting the learning rate at a standard value of 0.1 and incrementally increasing it to identify the point of optimal model performance. The effectiveness of various learning rates is assessed using two key metrics: BMA and NLL. These metrics, depicted in Figure~\ref{fig:exp_learn_increase}, guide our determination of suitable learning rates. Building upon these findings, we then adjust the learning rates of other baseline models to align with the identified optimal learning rate. Our results indicate that an initial learning rate of approximately 2.0 is effective for both the baseline and the proposed algorithms.

For consistency in our methodology, this initial learning rate (2.0) is applied to SGDM, R-SGDM, SGHMC, R-SGHMC, cycSGHMC, and R-cycSGHMC. For SGDM, R-SGDM, SGHMC, and R-SGHMC, this learning rate is consistently held during the initial $300$ epochs to facilitate exploration and subsequently decayed at a rate of 0.990 each epoch. In the case of cycSGHMC and R-cycSGHMC, we opt for a triple-cycle cosine learning rate schedule. For reSGHMC and r2SGHMC, which utilize four chains, the chain with the lowest learning rate begins at 1.0, scaling up to approximately 4.0 for the chain with the highest rate.

\begin{figure}[htbp]
\centering
\includegraphics[width=.6\columnwidth]{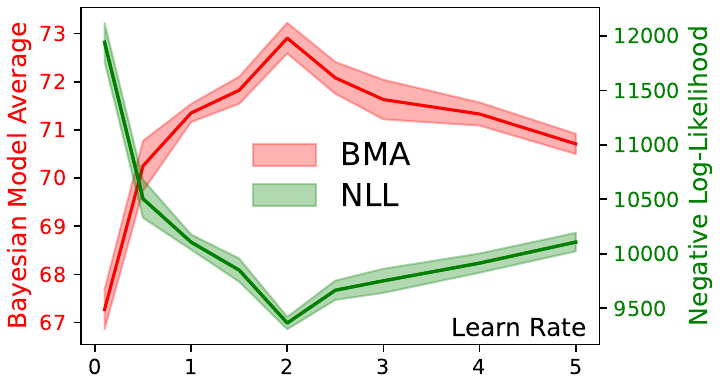}

\caption{Metric comparisons of CIFAR 100 with different initial learning rates.}
\label{fig:exp_learn_increase}
\end{figure}

\paragraph{Swap Efficiency}

In the realm of reSGLD and r2SGLD algorithms, the implementation of effective swapping schemes is crucial to enhance exploration efficiency. A fundamental challenge arises due to the minimal overlap between the distributions of the extremal temperatures, $\pi^{(1)}$ and $\pi^{(P)}$ ($P$ denotes the last chain or the number of chains), which hinders acceleration despite a large $\tau^{(P)}$. A viable solution involves introducing intermediate temperature particles $\left(\tau^{(2)}, \cdots, \tau^{(P-1)}\right)$ to facilitate ``tunnels" for more efficient chain swapping. Common strategies include all-pairs exchange (APE), adjacent (ADJ) pair swaps, and deterministic even-odd (DEO) schemes.

\begin{figure}[htbp]
\centering
\includegraphics[width=.95\columnwidth]{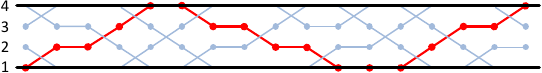}
\caption{Non-reversible chain swapping with the ODE scheme. The red line demonstrates a single chain's round-trip path.}
\label{fig:ode_scheme}
\end{figure}
The APE scheme attempts swaps between arbitrary chain pairs, offering thorough exploration \cite{Brenner07,Martin09}. However, its cubic swap time complexity ($O\left(P^3\right)$) and practical usability constraints limit its widespread adoption. The ADJ scheme, alternatively, swaps adjacent pairs in sequence, from $(1,2)$ to $(P-1, P)$ \cite{qi2018replica}. While this method is simpler and more intuitive, its sequential nature results in exchange information dependency and is less effective in multi-core or distributed environments, particularly with a larger number of chains. To address these limitations, the deterministic even-odd (DEO) scheme was introduced \cite{DEO,syed2022non,deng2023non}, which alternates swaps between even and odd pairs in successive iterations. The DEO scheme, characterized by its non-reversible and asymmetric nature, significantly increases swap rates and reduces round trip times, with its design detailed in Figure~\ref{fig:ode_scheme}.

In this research, we adopt the DEO scheme (as detailed in Algorithm~\ref{alg:r2sgld_modify}) for reSGLD and r2SGLD algorithms to optimize path efficiency in the deep learning task. The algorithm is structured to achieve a near-linear path and anticipate three to four round trips within a span of 1,000 epochs for increased swapping efficiency.

\begin{algorithm}[htbp]
   \caption{Reflected Replica Exchange Stochastic Gradient Langevin Dynamics with the DEO scheme.}
   \label{alg:r2sgld_modify}
{\textbf{Input} Window size $\mathcal W:=\left\lceil\frac{\log P+\log \log P}{-\log (1-\mathbb{S})}\right\rceil$.}\\
{\textbf{Input} Number of iterations $K$.}\\
{\textbf{Input} Number of chains $P$.}\\
{\textbf{Input} Target swap rate $\mathbb{S}$.}\\
{\textbf{Input} Step size $\gamma$.}
\begin{algorithmic}
\FOR{$k=1,2,\cdots, K$}
    \STATE{{Sampling with reflection} $\widetilde \bbeta_{k+1} = \mathcal{P}\bigg(\widetilde \bbeta_{k} - \eta\nabla {\widetilde U}(\widetilde \bbeta_k)+\sqrt{2\eta\tau_2} \bxi_k\bigg)$}
    \FOR{$p=1,2,\cdots, P-1$}
    \STATE{Update the swap indicator: $\mathcal{S}^{(p)}:=1_{\widetilde{U}\left(\boldsymbol{\beta}_{k+1}^{(p+1)}\right)+\mathcal{C}_k<\tilde{U}\left(\boldsymbol{\beta}_{k+1}^{(p)}\right)}$}
    \IF{$\mathrm{k} \bmod \mathcal{W}=0$}
        \STATE{Allow swaps $\mathcal{G}^{(p)}:=1$}
    \ENDIF
    \IF{$\mathcal{G}^{(p)}$ and $\mathcal{S}^{(p)}$}
        \STATE{Chain swap $\boldsymbol{\beta}_{k+1}^{(p)}$ and $\boldsymbol{\beta}_{k+1}^{(p+1)}$}
        \STATE{Refuse swaps $\mathcal{G}^{(p)}:=0$}
    \ENDIF
    \ENDFOR
    \STATE{Update correction $\mathcal{C}_{k+1}=\mathcal{C}_k+\gamma\left(\frac{1}{P-1} \sum_{p=1}^{P-1} \mathbf 1_{\widetilde{U}\left(\boldsymbol{\beta}_{k+1}^{(p+1)}\right)+\mathcal{C}_k-\widetilde{U}\left(\boldsymbol{\beta}_{k+1}^{(p)}\right)<0}-\mathbb{S}\right)$.}
\ENDFOR
    \vskip -1 in
\end{algorithmic}
{\textbf{Output} Parameters $\{\bbeta_k^{(1)} \}_{k=1}^K$.}
\end{algorithm}

The primary difference between Algorithm~\ref{alg:r2sgld_modify} and Algorithm~\ref{alg} is the consideration of a multiple-chain variant, which innovates chain swapping management. This is achieved through a carefully designed swap indicator $\mathcal{S}^{(p)}$ that dynamically adjusts based on the relative positions of the chains, thereby ensuring more effective and frequent exchanges. The algorithm incorporates a window-based approach, where swaps are allowed only at specific intervals, determined by the window size $\mathcal W$. This strategic regulation of swaps is further enhanced by combining with a correction term $\mathcal{C}_k$, which is updated adaptively: 
$$
\mathcal{C}_{k+1}=\mathcal{C}_k+\gamma\left(\frac{1}{P-1} \sum_{p=1}^{P-1} \mathbf 1_{\widetilde{U}\left(\boldsymbol{\beta}_{k+1}^{(p+1)}\right)+\mathcal{C}_k-\widetilde{U}\left(\boldsymbol{\beta}_{k+1}^{(p)}\right)<0}-\mathbb{S}\right).
$$
Here, $\gamma$ denotes the step size, $P=4$ represents the number of chains involved in the process, and $\mathbb{S}$ is the target swap rate among chains. It should be noted that as the iteration index $k$ trends toward infinity, it is anticipated that both the threshold and adaptive learning rates will converge toward their respective fixed points. This implementation of a uniform $\mathcal{C}$ ensures that swaps occur with a frequency that optimizes the overall convergence towards the desired target distribution. This balance between exploration and exploitation in the r2SGLD algorithm is crucial for tasks involving complex datasets like CIFAR 100, where traditional methods might struggle with efficient parameter space navigation.

\clearpage
\section{Continuous Convergence Analysis}\label{sec:appendix_continuous}

\begin{proof}[Proof of Lemma~\ref{lemma:reversibility}]
First we show the reversibility of $\pi$ under the semigroup $\{e^{t\mathcal{L}}\}_{t\ge 0}$: 
\begin{equation}
  \int_{\Omega\times\Omega}g\mathcal{L} f\,\dd \pi(x_1, x_2)=\int_{\Omega\times\Omega}f\mathcal{L} g\,\dd \pi(x_1, x_2), \qquad \text{for all $f, g\in \mathcal{D}(\mathcal{L})$.}
  \label{eqn:symmetry}
\end{equation}
Recall in \eqref{eqn:generator} that $\mathcal{L}=\mathcal{L}^{(1)}+\mathcal{L}^{(2)}+\mathcal{L}^{(s)}$. We have
  \begin{equation}
  	 \begin{aligned}
  		&\int_{\Omega}\int_\Omega g\cL^{(1)}f \,\dd \pi(x_1, x_2)\\
  		&=\frac{1}{Z}\int_\Omega \int_\Omega g(x_1, x_2) \left[ -\langle \nabla_{x_1} f(x_1, x_2), \nabla_{x_1}U(x_1) \rangle+\tau_1 \Delta_{x_1}f(x_1, x_2) \right] e^{-\frac{U(x_1)}{\tau_1}} \dd x_1 e^{-\frac{U(x_2)}{\tau_2}}\dd x_2\\
  		&=\frac{1}{Z}\int_{\Omega}\int_{\Omega}g(x_1, x_2)\nabla\cdot \left( \tau_1 \nabla_{x_1} f(x_1, x_2) e^{-\frac{U(x_1)}{\tau_1}} \right)\dd x_1 e^{-\frac{U(x_2)}{\tau_2}}\,\dd x_2\\
  		&=-\int_{\Omega}\int_{\Omega}\tau_1\nabla_{x_1}g(x_1, x_2)\cdot \nabla_{x_1}f(x_1, x_2) \dd \pi(x_1, x_2)\qquad \text{(Integration by parts with $f$ satisfying \eqref{eqn:bdy}$_1$.)}\\
  		&=\int_{\Omega}\int_{\Omega}f\cL^{(1)}g\,\dd \pi(x_1, x_2)
  	\end{aligned}
  	\label{eqn:L1exchange}
  \end{equation}
A similar calculation with $f$ satisfying \eqref{eqn:bdy}$_2$ yields that
\begin{equation}
	\begin{aligned}
		\int_\Omega\int_\Omega g\cL^{(2)}f \,\dd \pi(x_1, x_2)&=\int_\Omega\int_\Omega f\cL^{(2)} g\,\dd \pi(x_1, x_2)=-\int_\Omega\int_\Omega\tau_2 \nabla_{x_2}f(x_1, x_2)\cdot \nabla_{x_2}g(x_1, x_2)\,\dd \pi(x_1, x_2).
	\end{aligned}\label{eqn:L2exchange}
\end{equation}
By the same argument as in \citet{chen2018accelerating} (Lemma 3.2), we can also derive
\begin{equation}
	\int_\Omega\int_\Omega g\cL^{(s)}f \,\dd \pi(x_1, x_2)=\int_{\Omega}\int_{\Omega}f\cL^{(s)}g \,\dd \pi(x_1, x_2).\label{eqn:Lsexchange}
\end{equation}
Combining \eqref{eqn:L1exchange}, \eqref{eqn:L2exchange} and \eqref{eqn:Lsexchange}, we obtain \eqref{eqn:symmetry}. Now by choosing $g=1\in \mathcal{D}(\mathcal{L})$ in \eqref{eqn:symmetry}, since $\mathcal{L} 1=0$, we obtain $$\int_{\Omega\times\Omega}\mathcal{L}f\,\dd \pi=\int_{\Omega\times\Omega}f\mathcal{L}1\,\dd \pi=0.$$
Thus $\pi$ is invariant under $\{e^{t\mathcal{L}}\}_{t\ge 0}$.
\end{proof}
Next, we introduce a Lyapunov function subject to the homogeneous Neumman boundary condition, this is an extension of the Lyapunov function introduced in \cite{Cattiaux2010notes}.
\begin{lemma}
  Let $\cL^{(0)}=\cL^{(1)}+\cL^{(2)}$, then  
 there exists a Lyapunov function $V(x_1, x_2)\ge 1$ such that 
  satisfies
  \begin{equation}
    \left\{
    \begin{array}{ll}
      \cL^{(0)}V(x_1, x_2)\le -\lambda V(x_1, x_2)+b, &(x_1, x_2)\in\Omega\times \Omega, \\
      \nabla_{x_1} V(x_1, x_2)\cdot \nu(x_1)=0, & (x_1, x_2)\in \pa \Omega\times \Omega, \\
      \nabla_{x_2} V(x_1, x_2)\cdot \nu(x_2)=0,& (x_1, x_2)\in \Omega\times \pa \Omega.
    \end{array}
    \right.
    \label{eqn:Lyapunov}
  \end{equation}
  for some constants $\lambda>0$, $b\ge 0$. 
  \label{lemma:Lyapunov}
\end{lemma}
\begin{proof} Let $d(x):={\rm dist} (x, \partial \Omega)$ be the distance function between $x$ and the boundary $\partial \Omega$. 
 There exists an infinitely differentiable \emph{regularized distance function} which is equivalent to $d$ \cite{Kufner1985weightedSobolev, lieberman1985regularized}. More precisely,  there exists a positive function $\rho(x)\in C^\infty(\Omega)$ and positive constants $C_k$, $k=0, 1, 2, \cdots, $  such that for all $x\in \Omega$, 
$$\frac{1}{C_0}d(x)\le \rho(x)\le C_0 d(x),$$ and for $k\ge 1$, $$|D^k \rho(x)|\le C_k \rho(x)^{1-k}.$$

Now we defined the Lyapunov function $V$ by $\rho$:
$$V(x_1, x_2):=\exp[\rho(x_1)^2+\rho(x_2)^2] \ge 1.$$
We can compute the Hessian $ \nabla_{x_1 x_1}^2 V(x_1, x_2)$:
$$\begin{aligned}
\nabla_{x_1 x_1}^2 V(x_1, x_2)&=2 \nabla_{x_1}\rho(x_1)\otimes \nabla_{x_1}\rho(x_1) V(x_1, x_2)\\
&+2\rho(x_1) \nabla_{x_1 x_1}^2 \rho(x_1) V(x_1, x_2) + 4\rho(x_1)^2 \nabla_{x_1}\rho(x_1)\otimes \nabla_{x_1}\rho(x_1) V(x_1, x_2),
\end{aligned}$$
where $\otimes$ is the tensor product. Using the estimates for $\rho$, we obtain the following bound:
$$|\nabla_{x_1 x_1}^2
V(x_1, x_2)| \le [2C_1^2+2C_2+4C_0^2{\rm diam}(\Omega)^2]\exp[2C_0^2{\rm
diam}(\Omega)^2].$$ 
Given any $\lambda>0$,  we can take $b=\sup\{\cL^{(0)}V(x_1, x_2)+\lambda V(x_1, x_2):(x_1, x_2)\in \overline{\Omega}\times\overline{\Omega}\}\vee 0<\infty$. Then \eqref{eqn:Lyapunov}$_1$ is satisfied. 
Furthermore, we notice that since 
$$|\nabla_{x_1}V(x_1,
x_2)|\le 2\rho(x_1)|\nabla_{x_1} \rho(x_1)| V(x_1, x_2)\le 2 C_0 d(x_1)\cdot
C_1\cdot \exp[2 C_0^2 {\rm diam}(\Omega)^2]\to 0 \text{ as } x_1\to \partial
\Omega,$$
thus the homogeneous Neumann boundary condition \eqref{eqn:Lyapunov}$_2$ with respect to $x_1$ is also true for $V$. Similarly, we can get \eqref{eqn:Lyapunov}$_3$ also holds. 
\end{proof}
\begin{proof}
    [Proof of Lemma~\ref{lemma:PoincareIne}]
Next, we show the Poincar'{e} inequality using the Lyapunov function $V$ in Lemma~\ref{lemma:Lyapunov}. Suppose $f\in C^1(\Omega\times\Omega)$, and $u$ is an arbitrary constant.  Multiplying both sides of \eqref{eqn:Lyapunov}$_1$ with $(f-u)^2/\lambda V$ and integrating over $\Omega\times \Omega$ against $\pi(x_1, x_2)$ gives
\begin{equation}
	\begin{aligned}
		\int_\Omega\int_\Omega (f-u)^2 \,\dd\pi(x_1, x_2)&\le \int_\Omega\int_\Omega -\frac{\cL^{(0)}V}{\lambda V} (f-u)^2 \,\dd \pi(x_1, x_2)+b\int_\Omega\int_\Omega \frac{(f-u)^2}{\lambda V}\,\dd \pi(x_1, x_2).
	\end{aligned}\label{eqn:term}
\end{equation}
Notice that 
\begin{equation*}
    \begin{aligned}
  &\int_\Omega\int_\Omega -\frac{\cL^{(1)}V}{\lambda V}(f-u)^2 \dd \pi(x_1, x_2)\\
  &=\int_\Omega\int_\Omega -\frac{(f-u)^2}{\lambda V}\nabla_{x_1}\cdot\left(\tau_1 \nabla_{x_1} V e^{-\frac{U(x_1)}{\tau_1}} \right) \dd x_1 e^{-\frac{U(x_2)}{\tau_2}}\dd x_2\\
  &=\int_\Omega\int_\Omega\tau_1 \nabla_{x_1}\left[ \frac{(f-u)^2}{\lambda V}  \right]\cdot \nabla_{x_1} V e^{-\frac{U(x_1)}{\tau_1}}\dd x_1 e^{-\frac{U(x_2)}{\tau_2}}\dd x_2 \quad \text{(Integration by parts with $V$ satisfying \eqref{eqn:Lyapunov}$_2$.)}\\
  &=\tau_1\int_\Omega\int_\Omega \nabla_{x_1}\left[ \frac{(f-u)^2}{\lambda V} \right]\cdot \nabla_{x_1} V \dd \pi(x_1, x_2)\\
  &=\tau_1\int_\Omega\int_\Omega \frac{2(f-u)}{\lambda V}\nabla_{x_1}(f-u)\cdot \nabla_{x_1} V -(f-u)^2 \frac{|\nabla_{x_1} V|^2}{\lambda V^2} \dd\pi(x_1, x_2)\\
  &=\frac{\tau_1}{\lambda}\int_\Omega\int_\Omega\left( |\nabla_{x_1}(f-u)|^2 -\left|\nabla_{x_1}(f-u)-\frac{f-u}{V}\nabla_{x_1}V\right|^2\right)\dd \pi(x_1, x_2)\\
  &\le \frac{\tau_1}{\lambda} \int_\Omega\int_\Omega |\nabla_{x_1}(f-u)|^2 \dd \pi(x_1, x_2).
    \end{aligned}
\end{equation*}

Similar computation works for $\cL^{(2)}$, we reach that 
\begin{equation}
  \int_\Omega\int_\Omega-\frac{\cL^{(0)}V}{\lambda V}(f-u)^2 \,\dd \pi(x_1, x_2)\le \frac{1}{\lambda}\int_\Omega\int_\Omega\left( \tau_1|\nabla_{x_1}(f-u)|^2+\tau_2|\nabla_{x_2}(f-u)|^2 \right)\,\dd \pi(x_1, x_2).
  \label{eqn:Lterm}
\end{equation}
On the other hand, we choose $u^*$ such that $\int_\Omega\int_\Omega (f-u^*)\,\dd\pi(x_1, x_2)=0$, then 
\begin{equation}
	\begin{aligned}
		&b\int_\Omega\int_\Omega \frac{(f-u^*)^2}{\lambda V}\,\dd \pi(x_1, x_2)=\frac{b}{\lambda}\int_\Omega\int_\Omega (f-u^*)^2 \,\dd \pi(x_1, x_2)\le \frac{b}{\lambda}c_{\rm P}\int_\Omega\int_\Omega |\nabla f|^2 \,\dd \pi(x_1, x_2).
	\end{aligned}\label{eqn:bterm}
\end{equation}
Combining \eqref{eqn:term}, \eqref{eqn:Lterm} and \eqref{eqn:bterm} we obtain the Poincar\'{e} inequality \eqref{eqn:PoineIneq}. 
\end{proof}
Now we have all the ingredients to show the accelerated exponential convergence of $\mu_t$.

\begin{proof}
    [Proof of Theorem~\ref{thm:chiconvergence}] 
 Using the generator $\mathcal{L}$, we  have $\dd\mu_t/\dd \pi=e^{t\mathcal{L}}(\dd \mu_0/\dd \pi)$. In other words, given that $\dd\mu_0/\dd\pi$ satisfies \eqref{eqn:bdy}, $f_t:=\dd\mu_t/\dd \pi$ satisfies $\pa_t f_t=\cL f_t$ and \eqref{eqn:bdy} for all $t>0$. Then using the integration by parts we can show
\begin{equation}
\begin{aligned}
  \frac{\dd}{\dd t}\chi^2(\mu_t\|\pi)&=\frac{\dd}{\dd t}\int_{\Omega\times\Omega} \left[e^{t\mathcal{L}}\left(\frac{\dd \mu_0}{\dd \pi}\right)\right]^2\dd\pi
  =2\int_{\Omega\times \Omega}e^{t\mathcal{L}}\left( \frac{\dd \mu_0}{\dd \pi} \right)\frac{\dd}{\dd t}\left[ e^{t\mathcal{L}}\left( \frac{\dd \mu_0}{\dd \pi} \right) \right]\dd \pi\\
  &=2\int_{\Omega\times \Omega}e^{t\mathcal{L}}\left( \frac{\dd \mu_0}{\dd \pi} \right) \mathcal{L}\left[ e^{t\mathcal{L}}\left( \frac{\dd \mu_0}{\dd \pi} \right) \right]\dd\pi\\ &=-2\mathcal{E}_S\left( e^{t\mathcal{L}}\left( \frac{\dd \mu_0}{\dd \pi} \right) \right)
  =-2\cE_S\left( \frac{\dd\mu_t}{\dd \pi} \right). 
  \end{aligned}
  \label{eqn:ddtchi2}
\end{equation}
By Lemma~\ref{lemma:PoincareIne}, we have following estimate
\begin{equation}
\chi^2(\mu_t\|\pi)\le C_{\rm P} \cE\left( \frac{\dd \mu_t}{\dd \pi} \right)\le C_{\rm P}(1+\eta_S)^{-1}\cE_S \left( \frac{\dd \mu_t}{\dd \pi} \right).\label{eqn:poincarmut}
\end{equation}
Combining \eqref{eqn:ddtchi2} and \eqref{eqn:poincarmut} yields
\begin{equation*}
     \frac{\dd}{\dd t}\chi^2(\mu_t\|\pi)\le -2C_{\rm P}^{-1}(1+\eta_S)\chi^2(\mu_t\|\pi).
    \end{equation*}
According to Gronwall's inequality, we conclude that 
    \begin{equation*}
        \chi^2(\mu_t\|\pi)\le \chi^2(\mu_0\|\pi)\exp\{-2(C_{\rm P}^{-1}(1+\eta_S))\}.
    \end{equation*}
\end{proof}

\begin{proof}
[Proof of Theorem \ref{thm:W2conv}] Given that $\dd\mu_0/\dd\pi\ge0$ and satisfying \eqref{eqn:bdy}, we have that $f_t$ satisfies $\pa_t f_t=\cL f_t$, $f_t\ge 0$ and \eqref{eqn:bdy} for all $t>0$. Then again we use the integration by parts to derive
\begin{equation}
    \begin{aligned}
      \frac{\dd}{\dd t}D(\mu_t\|\pi)&=\frac{\dd}{\dd t}\int_{\Omega\times\Omega}e^{t\mathcal{L}}\left( \frac{\dd \mu_0}{\dd \pi} \right)\ln\left[ e^{t\mathcal{L}}\left( \frac{\dd \mu_0}{\dd \pi} \right) \right] \dd \pi\\
      &=\int_{\Omega\times \Omega}\left[ 1+\ln \left( e^{t\mathcal{L}}\left( \frac{\dd \mu_0}{\dd \pi} \right) \right) \right]\mathcal{L}\left[ e^{t\mathcal{L}}\left( \frac{\dd \mu_0}{\dd \pi} \right) \right]\dd \pi\\
      &=\int_{\Omega\times\Omega}\ln\left(\frac{\dd \mu_t}{\dd \pi}\right)\mathcal{L}\left(\frac{\dd \mu_t}{\dd \pi}\right)\dd \pi\le -2\mathcal{E}_S\left(\sqrt{\frac{\dd\mu_t}{\dd \pi}}\right).
    \end{aligned}\label{eqn:W2cal1}
\end{equation}
Now we apply the Log-Sobolev inequality \eqref{eqn:LogSob} to get
\begin{equation}
  D(\mu_t\|\pi)\le C_{\rm LS} \mathcal{E} \left(\sqrt{\frac{\dd\mu_t}{\dd \pi}}\right)\le C_{\rm LS}(1+\delta_S)^{-1}\mathcal{E}_S\left(\sqrt{\frac{\dd \mu_t}{\dd \pi}}\right).\label{eqn:W2cal2}
\end{equation}
Combining \eqref{eqn:W2cal1} and \eqref{eqn:W2cal2} together, we arrive at
\begin{equation*}
    \frac{\dd}{\dd t}D(\mu_t\|\pi)\le -2 C_{\rm LS}^{-1}(1+\delta_S) D(\mu_t\|\pi). 
\end{equation*}
Then we apply the Gronwall inequality to get
\begin{equation*}
    D(\mu_t\|\pi)\le D(\mu_0\|\pi)\exp\{-2(C_{\rm LS}^{-1}(1+\delta_S))\}.
\end{equation*}
By the Otto--Villani Theorem \cite{Bakry20142} (Theorem 9.6.1), we have the following quadratic transportation cost inequality
\begin{equation*}
    \mathcal{W}_2(\mu_t, \pi)\le \sqrt{2C_{\rm LS}D(\mu_t\|\pi)}\le \sqrt{2C_{\rm LS}D(\mu_0\|\pi)}\exp\{-tC_{\rm LS}^{-1}(1+\delta_S)\}
\end{equation*}
\end{proof}

\clearpage
\section{Discretization Analysis}
\label{sec:appendix_discrete}
To better formulate the discretized error, we rewrite the \textbf{Reflected Replica exchange Langevin diffusion} as below. For any fixed learning rate $\eta>0$, we define
\begin{equation}  
\small
\label{reLD_main}
\left\{  
             \begin{array}{lr}  
             d\bbeta_t=-\nabla G(\bbeta_t)dt+\Si(\alpha_t)d\bW_t+\nu_tL(dt),  \\  
              & \\
              \mathbb{P}\left(\alpha(t)=j|\alpha(t-dt)=l,  \bbeta(\lfloor t/\eta \rfloor \eta)=\bbeta\right)=rS(\bbeta) \eta \mathbf{1}_{\{t=\lfloor t/\eta \rfloor \eta\}} +o(dt),~~\text{for}~~  l\neq j,
             \end{array}  
\right.  
\end{equation} 
where we denote $\{ \bbeta_t\}_{t\ge 0 }:=\left\{\begin{pmatrix}{}
\bbeta_t^{(1)}\\
\bbeta_t^{(2)}
\end{pmatrix}\right\}_{t\ge 0 }$, $\nabla G(\bbeta):=\begin{pmatrix}{}
\nabla U(\bbeta^{(1)})\\
\nabla U(\bbeta^{(2)})
\end{pmatrix}$, and
$
\nu_tL(dt)=\begin{pmatrix}\nu^{(1)}_t &0\\
0& \nu_t^{(2)}
\end{pmatrix}\begin{pmatrix}
L^{(1)}(dt)\\
L^{(2)}(dt)
\end{pmatrix}$, where the local time $\nu_tL(dt)$ ensures the process $\bbeta_t$ stays in the domain $\Omega\times\Omega$.  We denote $\mathbf 1_{t=\lfloor t/\eta\rfloor\eta}$ as the indicator function, i.e. for every $t=i\eta$ with $i\in \mathbb N^+$, given $\bbeta(i\eta)=\bbeta$, we have  $\mathbb{P}\left(\alpha(t)=j|\alpha(t-dt)=l\right)=rS(\bbeta)\eta$, where $S(\bbeta)$ is defined as $\min\{1, S(\bbeta^{(1)}, \bbeta^{(2)})\}$ and $S(\bbeta^{(1)}, \bbeta^{(2)})$ is defined in (\ref{S_exact}). 
In this case, the Markov Chain $\alpha(t)$ is a constant on the time interval $[\lfloor t/\eta\rfloor\eta,\lfloor t/\eta\rfloor\eta+\eta)$ with some state in the finite-state space $\{0,1\}$ and the generator matrix $Q$ follows
\begin{equation*}
    Q=\begin{pmatrix}
-rS(\bbeta) \eta \delta(t-\lfloor t/\eta \rfloor \eta)&rS(\bbeta) \eta \delta(t-\lfloor t/\eta \rfloor \eta)\\
rS(\bbeta) \eta \delta(t-\lfloor t/\eta \rfloor \eta)&-rS(\bbeta) \eta \delta(t-\lfloor t/\eta \rfloor \eta)
\end{pmatrix},
\end{equation*}
where $\delta(\cdot)$ is a Dirac delta function. The diffusion matrix $\Si(\alpha_t)$
is thus defined as $(\Si(0), \Si(1) ):=\left\{\begin{pmatrix}{}
\sqrt{2\tau^{(1)}}\mathbf I_d&0\\
0&\sqrt{2\tau^{(2)}}\mathbf I_d
\end{pmatrix}, \begin{pmatrix}{}
\sqrt{2\tau^{(2)}}\mathbf I_d&0\\
0&\sqrt{2\tau^{(1)}}\mathbf I_d
\end{pmatrix}\right\}$. The process $\bbeta_t$ is the unique solution to the Skorokhod problem for the jump process defined by:
\begin{equation}  
\small
\label{reLD}
\left\{  
             \begin{array}{lr}  
             d \by_t=-\nabla G(\bbeta_t)dt+\Si(\alpha_t)d\bW_t,  \\  
              & \\
              \mathbb{P}\left(\alpha(t)=j|\alpha(t-dt)=l,  \bbeta(\lfloor t/\eta \rfloor \eta)=\bbeta\right)=rS(\bbeta) \eta \mathbf{1}_{\{t=\lfloor t/\eta \rfloor \eta\}} +o(dt),~~\text{for}~~  l\neq j.
             \end{array}  
\right.  
\end{equation} 
The existence of the unique solution for the Skorokhod problem for $\by$ follows from the original proof in \cite{tanaka1979} (see also \cite{Menaldi1985}) since the process $\by_t$ has continuous paths. The switching happens after the reflection at each time $i\eta<T$.  We should think of the process as concatenated on the time interval $[i\eta,(i+1)\eta)$ up to time horizon $T$. Similarly, we consider the following \textbf{Reflected Replica exchange stochastic gradient Langevin diffusion}, for the same learning rate $\eta>0$ as above, we have
\begin{equation}  
\small
\label{reSGLD1}
\left\{  
             \begin{array}{lr}  
             d\widetilde \bbeta_t^{\eta}=-\nabla\widetilde G(\widetilde \bbeta^{\eta}_{\lfloor t/\eta \rfloor \eta})dt+\Si(\widetilde \alpha_{\lfloor t/\eta \rfloor \eta})d\bW_t+\widetilde \nu_t \widetilde L(dt),  \\ 
              & \\
             \mathbb{P}\left(\widetilde \alpha(t)=j|\widetilde\alpha(t-dt)=l, \widetilde\bbeta(\lfloor t/\eta \rfloor \eta)=\widetilde\bbeta \right)=r\widetilde S(\widetilde \bbeta)\eta\mathbf 1_{\{t=\lfloor t/\eta\rfloor\eta\}}+o(dt),~~\text{for}~~  l\neq j,
             \end{array}  
\right.  
\end{equation} 
where $\nabla \widetilde G(\bbeta):=\begin{pmatrix}{}
\nabla \widetilde U(\bbeta^{(1)})\\
\nabla \widetilde U(\bbeta^{(2)})
\end{pmatrix}$ and $\widetilde S(\widetilde\bbeta) = \min\{1, \widetilde S(\widetilde\bbeta^{(1)},\widetilde\bbeta^{(2)})\}$ and $\widetilde S(\widetilde\bbeta^{(1)},\widetilde\bbeta^{(2)})$ is defined in Section \ref{tilde S}. Here $\int_0^t \widetilde\nu_s \widetilde L(ds)=\int_0^t\begin{pmatrix}\widetilde \nu^{(1)}_s &0\\
0&\widetilde \nu_s^{(2)}
\end{pmatrix}\begin{pmatrix}\widetilde
L^{(1)}(ds)\\\widetilde
L^{(2)}(ds)
\end{pmatrix}$denotes the bounded variation reflection process ensuring $\widetilde\bbeta_t^{\eta}$ stays inside the domain $\Omega\times\Omega$.

The reflection process $\widetilde\beta^{\eta}_t$ is then the unique solution to the Skorokhod problem for the following process,

\begin{equation}  
\small
\label{reSGLD2}
\left\{  
             \begin{array}{lr}  
             d\widetilde \by_t^{\eta}=-\nabla\widetilde G(\widetilde \bbeta^{\eta}_{\lfloor t/\eta \rfloor \eta})dt+\Si(\widetilde \alpha_{\lfloor t/\eta \rfloor \eta})d\bW_t,  \\ 
              & \\
             \mathbb{P}\left(\widetilde \alpha(t)=j|\widetilde\alpha(t-dt)=l, \widetilde\bbeta(\lfloor t/\eta \rfloor \eta)=\widetilde\bbeta \right)=r\widetilde S(\widetilde \bbeta)\eta\mathbf 1_{\{t=\lfloor t/\eta\rfloor\eta\}}+o(dt),~~\text{for}~~  l\neq j,
             \end{array}  
\right.  
\end{equation} 
which is the continuous interpolation of the replica exchange stochastic gradient Langevin dynamics. 

We then show the following proof of Theorem~\ref{theorem:discretization}. 

\begin{proof}[Proof of Theorem~\ref{theorem:discretization}]
According to the definition of the \textbf{Reflected Replica exchange Langevin diffusion}  $\{\bbeta_t\}_{t\ge 0}$ and the \textbf{Reflected Replica exchange stochastic gradient Langevin diffusion} $\{\widetilde\bbeta_t^{\eta} \}_{t\ge 0}$, we can write the difference of the two process $\bbeta_t-\widetilde\bbeta_t^{\eta}$ in the following form, 
\begin{equation}
\bbeta_t-\widetilde\bbeta_t^{\eta}=\by_t-\widetilde \by_t+\int_0^t\nu_sL(ds)-\int_0^t\widetilde\nu_s\widetilde L(ds).
\end{equation}
Since both $\by_t$ and $\widetilde\by_t$ are piece wise continuous paths, and $(\bbeta_t, \int_0^t\nu_sL(ds))$ and $(\widetilde\bbeta_t,\int_0^t\nu_sL(ds))$ solve the Skorokhod problems for $\by_t$ and $\widetilde\by_t$, respectively. Following \cite{tanaka1979} (Lemma 2.2), we have
the following estimates,
\begin{equation}\label{est 1}
\begin{split}
\|\bbeta_t-\widetilde\bbeta_t^{\eta}\|^2&\le \|\by_t-\widetilde\by_t^{\eta}\|^2+\underbrace{2\int_0^t\Big(\by_t-\widetilde\by_t^{\eta}-\by_s+\widetilde\by_s^{\eta}\Big)^{\intercal}\Big(\nu_sL(ds)- \widetilde\nu_s\widetilde L(ds) \Big)}_{\mathcal K}.
\end{split}
\end{equation}
Denote $h_{\Omega\times\Omega}$ as the support function of $\Omega\times\Omega$, which is defined by $h_{\Omega\times\Omega}(y)=\sup\{\langle x, y\rangle;x\in\Omega\times\Omega \}$, for $y\in\Omega\times\Omega$. For simplification, we assume that the origin is inside the convex domain $\Omega$. Thus the Euclidean norm $\|\cdot\|$ and the norm $\|\cdot\|_{\Omega\times\Omega}$ are equivalent, up to an constant $C$, by taking $\Omega$ as $d$-dimensional ball. Following the estimates \cite{sebastien_bubeck} (Proposition 1) and the equivalent norm in our setting, for the inner normal vectors $\nu_s$ and $\widetilde \nu_s$, we have
\begin{equation}\label{est K}
\begin{split}
\sup_{0\le t\le T}\sqrt{ \mathcal K}&\le \sup_{0\le t\le T}\Big\{ 2C ( \|\by_t-\widetilde\by_t^{\eta}\|+\sup_{0\le s\le t}\|\by_s-\widetilde\by_s^{\eta} \| )\int_0^th_{\Omega\times\Omega}( -\nu_s) L(ds)\\
&+ 2C ( \|\by_t-\widetilde\by_t^{\eta}\|+\sup_{0\le s\le t}\|\by_s-\widetilde\by_s^{\eta} \|) \int_0^th_{\Omega\times\Omega}(-\widetilde \nu_s) \widetilde L(ds)\Big\}^{1/2}\\
&\le 2C\sqrt{ \sup_{0\le t\le T} \|\by_t-\widetilde\by_t^{\eta}\| \Big(\int_0^Th_{\Omega\times\Omega}(-\nu_s)  L(ds)+\int_0^Th_{\Omega\times\Omega}(-\widetilde \nu_s) \widetilde L(ds)\Big)},
\end{split}
\end{equation}
where the last inequality follows from the fact $L_t$ ($\widetilde L_t$ resp.) is increasing. 
Taking square root, $\sup_{0\le t\le T}$ and expectation $\mathbb E[\cdot]$ on both sides of \eqref{est 1}, together with Cauchy-Schwartz inequality, applying \eqref{est K}, we get
\begin{equation}\label{est sup norm}
\begin{split}
\mathbb E\Big[\sup_{0\le t\le T} \|\bbeta_t-\widetilde\bbeta_t^{\eta}\|\Big]&\le \mathbb E\Big[\sup_{0\le t\le T} \|\by_t-\widetilde\by_t^{\eta}\|\Big]\cdots\mathcal I\\
&+\underbrace{\sqrt{\mathbb E\Big[4C\sup_{0\le t\le T} \|\by_t-\widetilde\by_t^{\eta}\|\Big]} \sqrt{\mathbb E\Big[ \Big(\int_0^Th_{\Omega\times\Omega}(-\nu_s)  L(ds)+\int_0^Th_{\Omega\times\Omega}(-\widetilde \nu_s) \widetilde L(ds)\Big)\Big]}}_{\mathcal J}.
\end{split}
\end{equation}
\noindent \textbf{Estimates of $\mathcal I$:} The estimates of the first term $\mathcal I$ is similar to \cite{deng2020} (Lemma 1). We follow the similar steps here and highlight the major differences coming from the fact that $\bbeta_t$ and $\widetilde\bbeta_t^{\eta}$ are reflected processes in the current setting. 
Plugging in the dynamics for $\by_t$ and $\widetilde \by_t$, we first have 
\begin{equation}\label{by - tilde by}
\by_t-\widetilde\by_t^{\eta}=\int_0^t \nabla G(\bbeta_s)-\nabla\widetilde G(\widetilde \bbeta^{\eta}_{\lfloor s/\eta \rfloor\eta}) ds + \int_0^t (\Si_s-\widetilde \Si^{\eta}_{\lfloor s/\eta \rfloor\eta})d\bW_s.
\end{equation}
Taking $\sup_{0\le t\le T}$, and expectation of both sides of \eqref{by - tilde by},
then applying Burkholder-Davis-Gundy inequality and  Cauchy-Schwarz inequality, we have 
\begin{equation*}\label{sup norm estimate 1}
	\begin{split}
	\mathbb E\Big[\sup_{0\le t\le T} \|\by_t-\widetilde\by_t^{\eta}\|\Big]&\le \mathbb E\left[\int_0^T\|\nabla G(\bbeta_s)-\nabla\widetilde G(\widetilde \bbeta^{\eta}_{\lfloor s/\eta \rfloor\eta})\|ds+\sup_{0\le t\le T}\left\|\int_0^t (\Si_s-\widetilde \Si^{\eta}_{\lfloor s/\eta \rfloor\eta})d\bW_s\right\|\right]\\
	&\le \underbrace{\mathbb E\left[\int_0^T\|\nabla G(\bbeta_s)-\nabla\widetilde G(\widetilde \bbeta^{\eta}_{\lfloor s/\eta \rfloor\eta})\|ds \right]}_{\mathcal I_1}+\underbrace{C\sqrt{ \mathbb E\left[ \int_0^T\|\Si_s-\widetilde \Si^{\eta}_{\lfloor s/\eta \rfloor\eta}\|^2 ds\right]}}_{\mathcal I_2}.
	\end{split}
\end{equation*}
The estimate of $\mathcal I_2$ follows the same as in \cite{deng2020} (Lemma 1, estimates (15), (16), (17)), which gives us
\begin{equation}\label{I2 est}
	\begin{split}
	(\mathcal I_2)^2&=C\mathbb E\left[\int_{0}^{T} \|\Sigma_s-\widetilde \Sigma_{{\lfloor s/\eta \rfloor}\eta}\|^2 ds\right] \\
	& \leq C\sum_{j=1}^{2d}\sum_{k=0}^{{\lfloor T/\eta \rfloor}}\int_{k \eta}^{(k+1)\eta}\mathbb E\left[\|\Sigma_s(j)-\widetilde \Sigma^{\eta}_{k \eta}(j)\|^2\right]ds \\
	& \leq C\sum_{j=1}^{2d}\sum_{k=0}^{{\lfloor T/\eta \rfloor}}\int_{k \eta}^{(k+1)\eta}\mathbb E\left[\|\Sigma_s(j)- \Sigma_{k \eta}^{\eta}(j)+\Sigma_{k \eta}^{\eta}(j)-\widetilde \Sigma_{k \eta}^{\eta}(j)\|^2\right]ds \\
	& \leq C\sum_{j=1}^{2d}\sum_{k=0}^{{\lfloor T/\eta \rfloor}}\left[\int_{k \eta}^{(k+1)\eta}\mathbb E\left[\|\Sigma_s(j)- \Sigma_{k \eta}^{\eta}(j)\|^2\right]ds+\int_{k \eta}^{(k+1)\eta}\mathbb E\left[\|\Sigma_{k \eta}^{\eta}(j)-\widetilde \Sigma_{k \eta}^{\eta}(j)\|^2\right]ds\right] \\
	&\le Cd(1+T)\tilde \delta_2(\tau_1,\tau_2)\left(\eta+\max_{k}\sqrt{\mathbb E\left[|\psi_{k}|^2\right]}\right).
	\end{split}.
\end{equation}
where $\tilde \delta_2(\tau_1,\tau_2)=4(\sqrt{\tau_2}-\sqrt{\tau_1})^2$, and $\psi_{k}$ is the noise in the swapping rate. For convenience, the constant $C$ may vary from line to line. We mainly focus on showing the estimates for the term $\mathcal I_1$. Applying the inequality $\|a+b+c\|^2\le 3(\|a\|^2+\|b\|^2+\|c\|^2 )$, we get

\begin{equation}
	\begin{split}
	\mathcal I_1	\le& \underbrace{\mathbb E\left[\int_0^T\|\nabla G(\bbeta_s)-\nabla  G(\widetilde \bbeta^{\eta}_s)\|ds\right]}_{\mathcal I_{11}}+\underbrace{\Big(T\mathbb E\left[\int_0^T\|\nabla  G(\widetilde \bbeta^{\eta}_s)-\nabla  G(\widetilde \bbeta^{\eta}_{\lfloor s/\eta \rfloor\eta})\|^2ds\right]\Big)^{1/2}}_{\mathcal I_{12}}\\
		&+\underbrace{\Big(T\mathbb E\left[\int_0^T\|\nabla  G(\widetilde \bbeta^{\eta}_{\lfloor s/\eta \rfloor\eta})-\nabla \widetilde G(\widetilde \bbeta^{\eta}_{\lfloor s/\eta \rfloor\eta})\|^2ds\right]\Big)^{1/2}}_{\mathcal I_{13}}\\
		\le& \mathcal I_{11}+\mathcal I_{12}+\mathcal I_{13}.
	\end{split}
\end{equation}
By using the smoothness assumption \eqref{assu:ULip}, we first get 
\begin{equation*}
    \mathcal I_{11} \le L\mathbb E\left[\int_0^T\|\bbeta_s-\widetilde \bbeta^{\eta}_s\|ds\right].
\end{equation*}

Following \cite{deng2020} (Lemma 1, estimates (10)), we next get
\begin{equation}\label{est cI2}
    \begin{split}
      (  \mathcal I_{12})^2&\le TL^2\mathbb E\left[\int_0^T\|\widetilde \bbeta^{\eta}_s-\widetilde \bbeta^{\eta}_{\lfloor s/\eta \rfloor\eta}\|^2ds\right]\\
        &\le TL^2\sum_{k=0}^{\lfloor T/\eta\rfloor} \mathbb E\left[\int_{k\eta}^{(k+1)\eta}\|\widetilde \bbeta^{\eta}_s-\widetilde \bbeta^{\eta}_{\lfloor s/\eta \rfloor\eta} \|^2ds \right]\\
        &\le TL^2\sum_{k=0}^{\lfloor T/\eta\rfloor} \int_{k\eta}^{(k+1)\eta}\mathbb E\left[\sup_{k\eta\le s<(k+1)\eta}\|\widetilde \bbeta^{\eta}_s-\widetilde \bbeta^{\eta}_{\lfloor s/\eta \rfloor\eta} \|^2\right]ds. 
    \end{split}{}
\end{equation}{}
For $\forall~ k\in\mathbb N$ and $s\in [k\eta,(k+1)\eta)$, since $\widetilde \bbeta^{\eta}_{\lfloor s/\eta \rfloor\eta}$ stays constant  on $[k\eta,(k+1)\eta)]$, which does not involve reflection, we thus have
\begin{equation*}
    \begin{split}
        \widetilde \bbeta^{\eta}_s-\widetilde \bbeta^{\eta}_{\lfloor s/\eta \rfloor\eta}=\widetilde \bbeta^{\eta}_s-\widetilde \bbeta^{\eta}_{k\eta}=-\nabla\widetilde G(\widetilde \bbeta^{\eta}_{k\eta})\cdot(s-k\eta)+\widetilde\Si^{\eta}_{k\eta}\int_{k\eta}^sd\bW_r+\int_{k\eta}^s\widetilde \nu_r\widetilde L(dr),
    \end{split}{}
\end{equation*}
where $\int_{k\eta}^s\widetilde \nu_r\widetilde L(dr)$ ensures the process $ \widetilde \bbeta^{\eta}_s$ stays inside the domain $\Omega\times\Omega$. Since there is no swap for the diffusion matrix $\Si(\widetilde \alpha_{\lfloor s/\eta \rfloor \eta})$ for $s\in[k\eta,(k+1)\eta)$, the existence of $\int_{k\eta}^s\widetilde \nu_r\widetilde L(dr)$ follows from the Skorohod problem for a drift-diffusion process \cite{tanaka1979}. Applying It\^o's formula to $\| \widetilde \bbeta^{\eta}_s-\widetilde \bbeta^{\eta}_{\lfloor s/\eta \rfloor\eta}\|^2$, we have  
\begin{equation*}
\begin{split}
\| \widetilde \bbeta^{\eta}_s-\widetilde \bbeta^{\eta}_{\lfloor s/\eta \rfloor\eta}\|^2&= 2\int_{k\eta}^s \langle \widetilde\bbeta_r^{\eta}-\widetilde \bbeta^{\eta}_{\lfloor r/\eta \rfloor\eta}, \Si(\widetilde \alpha_{\lfloor r/\eta \rfloor \eta})d\bW_r\rangle -2\int_{k\eta}^s \langle \widetilde\bbeta_r^{\eta}-\widetilde \bbeta^{\eta}_{\lfloor r/\eta \rfloor\eta}, \nabla\widetilde G(\widetilde \bbeta^{\eta}_{\lfloor r/\eta \rfloor \eta})\rangle dr\\
&+2\int_{k\eta}^s \langle \widetilde\bbeta_r^{\eta}-\widetilde \bbeta^{\eta}_{\lfloor r/\eta \rfloor\eta}, \widetilde\nu_r\rangle \widetilde L(dr)\\
&\le 2\int_{k\eta}^s \langle \widetilde\bbeta_r^{\eta}-\widetilde \bbeta^{\eta}_{\lfloor r/\eta \rfloor\eta}, \Si(\widetilde \alpha_{\lfloor r/\eta \rfloor \eta})d\bW_r\rangle -2\int_{k\eta}^s \langle \widetilde\bbeta_r^{\eta}-\widetilde \bbeta^{\eta}_{\lfloor r/\eta \rfloor\eta}, \nabla\widetilde G(\widetilde \bbeta^{\eta}_{\lfloor r/\eta \rfloor \eta})\rangle dr,
\end{split}
\end{equation*}
where the last inequality follows from the fact that both $ \widetilde\bbeta_r^{\eta}$ and $\widetilde \bbeta^{\eta}_{\lfloor r/\eta \rfloor\eta}$ are inside the domain, $\widetilde\nu_r$ is a inner normal vector at $\widetilde\bbeta_s^{\eta}$, and $\widetilde L$ is non-negative measure, thus $ \langle \widetilde\bbeta_r^{\eta}-\widetilde \bbeta^{\eta}_{\lfloor r/\eta \rfloor\eta}, \widetilde \nu_r\rangle\le 0$, since the domain is convex. Taking the $\sup_{k\eta\le s\le (k+1)\eta}$ and the expectation $\mathbb E$ on both sides, and applying the Burkholder-Davis-Gundy inequality, Cauchy-Schwarz inequality and the Young inequality (i.e. $2ab\le \varepsilon a^2+b^2/\varepsilon$), we have
\begin{equation*}
\begin{split}
\mathbb E\Big[\sup_{k\eta\le s\le (k+1)\eta}\| \widetilde \bbeta^{\eta}_s-\widetilde \bbeta^{\eta}_{\lfloor s/\eta \rfloor\eta}\|^2\Big]&\le 2\mathbb E\Big[\|\sup_{k\eta\le s\le (k+1)\eta}\int_{k\eta}^s \langle \widetilde\bbeta_r^{\eta}-\widetilde \bbeta^{\eta}_{\lfloor r/\eta \rfloor\eta}, \Si(\widetilde \alpha_{k \eta})d\bW_r\rangle\|\Big]\\
& +2\mathbb E\Big[  \sup_{k\eta\le s\le (k+1)\eta}\int_{k\eta}^s\langle \widetilde\bbeta_r^{\eta}-\widetilde \bbeta^{\eta}_{\lfloor r/\eta \rfloor\eta} , \nabla\widetilde G(\widetilde \bbeta^{\eta}_{k \eta}) dr\rangle \Big]\\
&\le 2C \mathbb E\Big[ \Big(\int_{k\eta}^{(k+1)\eta} \|\langle \widetilde\bbeta_r^{\eta}-\widetilde \bbeta^{\eta}_{\lfloor r/\eta \rfloor\eta}, \Si(\widetilde \alpha_{k \eta})\rangle\|^2 dr\Big)^{1/2}\Big]\\
&+\mathbb E\Big[\sup_{k\eta\le s\le (k+1)\eta}\int_{k\eta}^s \| \widetilde\bbeta_r^{\eta}-\widetilde \bbeta^{\eta}_{\lfloor r/\eta \rfloor\eta}\|^2dr +\eta^2\|  \nabla\widetilde G(\widetilde \bbeta^{\eta}_{k \eta})\|^2\Big]\\
&\le C \mathbb E\Big[ \sup_{k\eta\le s\le (k+1)\eta} \|\langle \widetilde\bbeta_s^{\eta}-\widetilde \bbeta^{\eta}_{\lfloor s/\eta \rfloor\eta}\|^2\varepsilon+\frac{1}{\varepsilon}\tau_2^2\eta\Big]\\
&+\int_{k\eta}^{(k+1)\eta}\mathbb E\Big[\sup_{k\eta\le r\le (k+1)\eta} \| \widetilde\bbeta_r^{\eta}-\widetilde \bbeta^{\eta}_{\lfloor r/\eta \rfloor\eta}\|^2\Big]dr +\mathbb E[\eta^2\|  \nabla\widetilde G(\widetilde \bbeta^{\eta}_{k \eta})\|^2],
\end{split}
\end{equation*}
where $0<\varepsilon<1$ is a small constant. Pick $\varepsilon$ such that $C\varepsilon<1$, we end up with
\begin{equation*}
\begin{split}
\mathbb E\Big[\sup_{k\eta\le s\le (k+1)\eta}\| \widetilde \bbeta^{\eta}_s-\widetilde \bbeta^{\eta}_{\lfloor s/\eta \rfloor\eta}\|^2\Big]
&\le \frac{1}{1-C\varepsilon}\Big(   \frac{1}{\varepsilon}\tau_2^2\eta +\mathbb E[\eta^2\|  \nabla\widetilde G(\widetilde \bbeta^{\eta}_{k \eta})\|^2]\Big)\\
&+\frac{1}{1-C\varepsilon}\int_{k\eta}^{(k+1)\eta}\mathbb E\Big[\sup_{k\eta\le r\le (k+1)\eta} \| \widetilde\bbeta_r^{\eta}-\widetilde \bbeta^{\eta}_{\lfloor r/\eta \rfloor\eta}\|^2\Big]dr,
\end{split}
\end{equation*}
We further have the following estimates, 
\begin{equation*}
    \begin{split}
  \eta^2 \mathbb E[ \|\nabla \widetilde G(\widetilde \bbeta^{\eta}_{k\eta})\|^2]&=        \eta^2\mathbb E[ \|(\nabla G(\widetilde \bbeta^{\eta}_{k\eta})+\bphi_k)\|^2 ]\\
  &\le 2\eta^2\mathbb E[ \|\nabla  G(\widetilde \bbeta^{\eta}_{k\eta})-\nabla  G(\bbeta^*) \|^2+\|\bphi_k\|^2 ]\\
   &\le 4C^2\eta^2\mathbb E[\|\widetilde \bbeta^{\eta}_{k\eta}\|^2+\|\bbeta^*\|^2]+2\eta^2\mathbb E[\|\bphi_k\|^2 ],
    \end{split}{}
\end{equation*}
where the first inequality follows from the separation of the noise from the stochastic gradient and the choice of stationary point $\bbeta^*$ of $G(\cdot)$ with $\nabla G(\bbeta^*)=0$, and $\bphi_k$ is the stochastic noise in the gradient at step $k$. Thus, combining the above two parts with Grownall's inequality, we get 
\begin{equation*}
\begin{split}
&\mathbb E\Big[\sup_{k\eta\le s\le (k+1)\eta} \| \widetilde \bbeta^{\eta}_s-\widetilde \bbeta^{\eta}_{\lfloor s/\eta \rfloor\eta}\|^2\Big]\\
&\le \frac{1}{1-C\varepsilon} \Big(4L^2(1+\eta)\eta^2\mathbb E[\|\widetilde \bbeta^{\eta}_{k\eta}\|^2+\|\bbeta^*\|^2]+4(1+\eta)\eta^2\mathbb E[\|\bphi_k\|^2 ]+\frac{1}{\varepsilon}\tau_2^2\eta(1+\eta)\Big).
\end{split}
\end{equation*}
Plugging the above estimates back in \eqref{est cI2}, we get 
\begin{equation*}
    \begin{split}
        (\mathcal I_{12} )^2       &\le TL^2(1+T /\eta)\Big[4L^2(1+\eta)\eta^3\sup_{k\ge 0}\mathbb E[\|\widetilde \bbeta^{\eta}_{k\eta}\|^2+\|\bbeta^*\|^2]+4\eta^3(1+\eta)\max_{k\ge 0}\mathbb E[\|\bphi_k\|^2 ]+\frac{1}{\varepsilon}\tau_2^2\eta^2(1+\eta)\Big]\\
        &\le \tilde \delta_1(d, \tau_2, T, C,\varepsilon) \eta^2 + 4TL^2(1+T)\eta^2\max_{k}\mathbb E[\|\bphi_k\|^2],
    \end{split}
\end{equation*}
where $\tilde \delta_1(d, \tau_2, T, C,\varepsilon)$ is a constant depending on $d, \tau_2,  T, C,$ and $\varepsilon$. Note that the above inequality requires a result on the bounded second moment of   $\sup_{k\ge 0}\mathbb E[ \|\widetilde \bbeta_{k\eta}^{\eta}\|^2]$, which follows from the bounded domain assumption. We are now left to estimate the term $\mathcal I_{13}$ and we have
\begin{equation}
    \begin{split}
  ( \mathcal I_{13} )^2 &\le  T\sum_{k=0}^{\lfloor T/\eta \rfloor} \mathbb E\left[\int_{k\eta}^{(k+1)\eta}\|\nabla  G(\widetilde \bbeta^{\eta}_{k\eta})-\nabla \widetilde G(\widetilde \bbeta^{\eta}_{k\eta})\|^2ds \right]\\
   &\le T (1+T/\eta)\max_{k}\mathbb E[ \|\bphi_k\|^2]\eta\\
   &\le T (1+T)\max_{k}\mathbb E[ \|\bphi_k\|^2].
    \end{split}{}
\end{equation}

Combing all the estimates of $\mathcal I_{11},\mathcal I_{12}$ and $\mathcal I_{13}$, we obtain
\begin{equation}\label{I1 est}
    \begin{split}
        \mathcal I_1\le& \underbrace{L\int_0^T\mathbb E\left[\sup_{0\le s\le T} \|\bbeta_s-\widetilde \bbeta^{\eta}_s\|\right]ds}_{\mathcal I_{11}}+\underbrace{\Big( \tilde \delta_1(d, \tau_2, T, C, \varepsilon) \eta^2+ 4TL^2(1+T)\eta^2\max_{k}\mathbb E[\|\bphi_k\|^2]\Big)^{1/2}}_{\mathcal I_{12}}\\
        &+\underbrace{\Big(T (1+T)\max_{k}\mathbb E[ \|\bphi_k\|^2]\Big)^{1/2}}_{\mathcal I_{13}}.\\
    \end{split}{}
\end{equation}{}
Combining the estimates in $\mathcal I_1$ and $\mathcal I_2$, we have 
\begin{equation}\label{est by}
    \begin{split}
       \mathbb E\Big[\sup_{0\le t\le T} \|\by_t-\widetilde\by_t^{\eta}\|\Big]\le& \underbrace{L\int_0^T\mathbb E\left[\sup_{0\le t\le T} \|\bbeta_t-\widetilde \bbeta^{\eta}_t\|\right]dt}_{\mathcal I_{11}}\\
       &+\underbrace{\Big(\tilde \delta_1(d, \tau_2, T, C, \varepsilon) \eta^2+ 4TL^2(1+T) \eta^2\max_{k}\mathbb E[\|\bphi_k\|^2]\Big)^{1/2}}_{\mathcal I_{12}}\\
        &+\underbrace{\Big(T (1+T)\max_{k}\mathbb E[ \|\bphi_k\|^2]\Big)^{1/2}}_{\mathcal I_{13}}+\underbrace{\sqrt{ Cd(1+T)\tilde \delta_2(\tau_1,\tau_2)\left(\eta+\max_{k}\sqrt{\mathbb E\left[|\psi_{k}|^2\right]}\right)}}_{\mathcal I_2}
    \end{split}{}
\end{equation}{}

\noindent \textbf{Estimates of $\mathcal J$:}
we first observe that, 
\begin{equation*}
\begin{split}
\mathcal J&= \sqrt{\mathbb E\Big[4C\sup_{0\le t\le T} \|\by_t-\widetilde\by_t^{\eta}\|\Big]} \sqrt{\mathbb E\Big[ \Big(\int_0^Th_{\Omega\times\Omega}(\nu_s)  L(ds)+\int_0^Th_{\Omega\times\Omega}(\widetilde \nu_s) \widetilde L(ds)\Big)\Big]}\\
& \le \frac{1}{2}\Big(\frac{1}{\eta^{1/4}}\mathbb E\Big[4C\sup_{0\le t\le T} \|\by_t-\widetilde\by_t^{\eta}\|\Big] +\eta^{1/4}\mathbb E\Big[ \Big(\int_0^Th_{\Omega\times\Omega}(\nu_s)  L(ds)+\int_0^Th_{\Omega\times\Omega}(\widetilde \nu_s) \widetilde L(ds)\Big)\Big] \Big).
\end{split}
\end{equation*}
The above inequality follows from Young's inequality $2ab\le a^2\frac{1}{\eta^c}+b^2 \eta^c.$ The choice of $0<c<1/2$ will affect the rate in \eqref{sup norm est}. 
The first term follows from \eqref{est by}. We are left to estimate the second term.
For $s\in[k\eta,(k+1)\eta)$, $\bbeta_s$ and $\widetilde\bbeta_s$ are reflected diffusion process without swap, by It\^o's formula, we have 
\begin{equation}
\begin{split}
 2\int_{k\eta}^s \langle \widetilde\bbeta^{\eta}_r, -\widetilde\nu_r \rangle\widetilde L(dr)=& 2\int_{k\eta}^s \langle \widetilde\bbeta_r^{\eta}, \Si(\widetilde \alpha_{\lfloor r/\eta \rfloor \eta})d\bW_r\rangle -2\int_{k\eta}^s \langle \widetilde\bbeta_r^{\eta}, \nabla\widetilde G(\widetilde \bbeta^{\eta}_{\lfloor r/\eta \rfloor \eta})dr\rangle\\
 &+\|\widetilde \bbeta_{k\eta}^{\eta}\|^2-\|\widetilde \bbeta_s^{\eta}\|^2+2d(s-k\eta).
 \end{split}
\end{equation}
Take expectation on both sides, since the domain is convex, and $\widetilde \nu_s$ is an inner normal vector, similar to \cite{sebastien_bubeck} (Lemma 9), for the support function $h_{\Omega\times\Omega}(\cdot)$, we have 
\begin{equation}
\begin{split}
 \mathbb E[2\int_{k\eta}^s h_{\Omega\times\Omega}(-\widetilde\nu_r)\widetilde L(dr)]&\le \mathbb E[\int_{k\eta}^s \|\langle \widetilde\bbeta_r^{\eta}, \nabla\widetilde G(\widetilde \bbeta^{\eta}_{\lfloor r/\eta \rfloor \eta})\rangle\|dr]+2d(s-k\eta)\\
&\le \frac{(2d+R\widetilde L)(s-k\eta)}{2}\le \frac{(2d+R\widetilde L)\eta}{2},
\end{split}
\end{equation}
where $R$ is the upper bound of the diameter of the domain, and $\widetilde L$ is the upper bound of the gradient $\nabla \widetilde G$. Summing over the estimates $\int_{k\eta}^{(k+1)\eta} $ for $k=0,\cdots, \lfloor T/\eta \rfloor $, we get the bound,
\begin{equation}
\int_0^Th_{\Omega\times\Omega}(\widetilde \nu_s) \widetilde L(ds)\le \frac{(2d+R\widetilde L)T}{2}.
\end{equation}
The same estimate also holds true for $\int_0^Th_{\Omega\times\Omega}(\nu_s) L(ds)$, which gives us 
\begin{equation}\label{est J}
\mathcal J\le \frac{2C}{\eta^{1/4}}\mathbb E\Big[\sup_{0\le t\le T} \|\by_t-\widetilde\by_t^{\eta}\|\Big] +(2d+R\widetilde L)T\eta^{1/4}.
\end{equation}

\noindent\textbf{Estimates of \eqref{est sup norm}:}
Plugging the estimates for $\mathcal I$ \eqref{est by} and $\mathcal J$ \eqref{est J} into \eqref{est sup norm}, we have 
\begin{equation*}
\begin{split}
\mathbb E\Big[\sup_{0\le t\le T} \|\bbeta_t-\widetilde\bbeta_t^{\eta}\|\Big]&\le (1+\frac{2C}{\eta^{1/4}})\mathbb E\Big[\sup_{0\le t\le T} \|\by_t-\widetilde\by_t^{\eta}\|\Big] +(2d+R\widetilde L)T\eta^{1/4}\\
&\le \frac{C}{\eta^{1/4}}\mathbb E\Big[\sup_{0\le t\le T} \|\by_t-\widetilde\by_t^{\eta}\|\Big] +(2d+R\widetilde L)T\eta^{1/4}\\
&\le \frac{C}{\eta^{1/4}}\int_0^T\mathbb E\left[\sup_{0\le t\le T} \|\bbeta_t-\widetilde \bbeta^{\eta}_t\|\right]dt +(2d+R\widetilde L)T\eta^{1/4}\\
       &+\frac{C}{\eta^{1/4}}\Big(\tilde \delta_1(d, \tau_2, T, C, \varepsilon) \eta^2+ 4TL^2(1+T) \eta^2\max_{k}\mathbb E[\|\bphi_k\|^2]\Big)^{1/2}\\
        &+\frac{C}{\eta^{1/4}}\Big(T (1+T)\max_{k}\mathbb E[ \|\bphi_k\|^2]\Big)^{1/2}+\frac{C}{\eta^{1/4}}\Big(Cd(1+T)\tilde \delta_2(\tau_1,\tau_2)\left(\eta+\max_{k}\sqrt{\mathbb E\left[|\psi_{k}|^2\right]}\right)\Big)^{1/2}.
    \end{split}
\end{equation*}
Here the constant $C$ varies from line to line. Applying Grownall's inequality, we have 
\begin{equation}\label{sup norm est}
\mathbb E\Big[\sup_{0\le t\le T} \|\bbeta_t-\widetilde\bbeta_t^{\eta}\|\Big]\le \delta_1 \eta^{1/4} + \delta_2 \sqrt{\max_{k}\mathbb E[\|\bphi_k\|^2]}+\delta_3 \sqrt{\eta^{-1/2}\max_{k}\sqrt{\mathbb E\left[|\psi_{k}|^2\right]}},
\end{equation}
where $\delta_1$ is a constant depending on $\tau_1,\tau_2,d, T, C,R,L, \widetilde L$, $\varepsilon$; $\delta_2$ depends on $T$, and $C$; $\delta_3$ depends on $d, T$, and $C$. By the definition of $\mathcal{W}_1$ distance, we get the following convergence rate in $\mathcal{W}_1$ distance,
\begin{equation}
    \mathcal W_1(\mathcal L(\bbeta_T),\mathcal L(\widetilde\bbeta_T))\le \mathbb E[\|\bbeta_T-\widetilde\bbeta_T^{\eta} \|
    ]\le E\Big[\sup_{0\le t\le T} \|\bbeta_t-\widetilde\bbeta_t^{\eta}\|\Big].
\end{equation}
\end{proof}

\end{document}